\documentclass[11pt]{article}
\usepackage{fullpage}
\usepackage{amsmath,amsthm,amsfonts,amssymb,dsfont,bm,booktabs}
\usepackage{enumerate,color,xcolor}
\usepackage[colorlinks,linkcolor=cyan,citecolor=blue]{hyperref}
\usepackage{url}
\usepackage{graphicx}
\usepackage{subcaption}
\usepackage{float}
\usepackage{placeins} %
\usepackage{tikz}
\usetikzlibrary{arrows.meta}
\numberwithin{equation}{section}
\theoremstyle{plain}

\newtheorem{theorem}{Theorem} \newtheorem{corollary}[theorem]{Corollary}  \newtheorem{lemma}[theorem]{Lemma} \newtheorem{proposition}[theorem]{Proposition} \newtheorem{assumption}[theorem]{Assumption}

\newcommand{\ep}{\varepsilon}  \newcommand{\Law}{\mathrm{Law}}  

\begin{document}
\begin{center}
{\Large\bf Reflected Anchored Langevin Algorithms}
\end{center}

\begin{center}
{Changwei Tu}\,\footnote{Hong Kong University of Science and Technology (Guangzhou), Guangzhou, Guangdong Province, People's Republic of China, ctu570@connect.hkust-gz.edu.cn}, 
{Xiaoyu Wang}\,\footnote{Hong Kong University of Science and Technology (Guangzhou), Guangzhou, Guangdong Province, People's Republic of China, xiaoyuwang@hkust-gz.edu.cn}, 
{Yingli Wang}\,\footnote{School of Mathematical Sciences, Fudan University, Shanghai, People's Republic of China, yingliwang@fudan.edu.cn}, {Xicheng Zhang}\,\footnote{School of Mathematics and Statistics, Beijing Institute of Technology, Beijing 100081, People's Republic of China, XichengZhang@gmail.com}, 
Lingjiong Zhu\,\footnote{Department of Mathematics, Florida State University, Tallahassee, Florida, United States of America, zhu@math.fsu.edu}
\end{center}

\begin{center}
\today
\end{center}

\begin{abstract}
First order Langevin algorithms for constrained sampling in machine learning, such as projected Langevin Monte Carlo which are based on discretizations of reflected Langevin dynamics, require differentiable log densities that limits their applicability. This paper introduces reflected anchored Langevin dynamics (RALD), a reflected diffusion that converges to non-differentiable targets on constrained domains. The method uses a smooth anchored reference potential and multiplies the drift and noise covariance of its reflected Langevin dynamics by the same state dependent scaling factor. Its Euler-Maruyama discretization with projection gives reflected anchored Langevin Monte Carlo (RALMC) algorithm. We prove explicit convergence bounds and iteration complexity for RALMC in the 2-Wasserstein distance to the target distribution. Numerical experiments are provided to illustrate the theoretical predictions and the empirical performance of the method.
\end{abstract}

\section{Introduction}

The problem of \textit{sampling} a target distribution
\begin{align}\label{eqn:target:unconstrained}
\pi(x)\propto\exp(-U(x)),\qquad x\in\mathbb{R}^{d},
\end{align}
for a function $U:\mathbb{R}^d \to \mathbb{R}$
arises in many applications, including Bayesian statistical inference \cite{gelman1995bayesian}, Bayesian formulations of inverse problems \cite{stuart2010inverse}, and Bayesian classification and regression tasks in machine learning \cite{andrieu2003introduction,teh2016consistency,DistMCMC19,GIWZ2024,DIGing,regime2025}.
\textit{Langevin algorithms} are popular \textit{Markov chain Monte Carlo} (MCMC) methods that allow one to sample from \eqref{eqn:target:unconstrained}, and the most classical one is based on the discretization of
\textit{overdamped Langevin dynamics} (OLD) \cite{Dalalyan,DM2016,DM2017,DK2017,Raginsky}:
\begin{equation}\label{eq:overdamped-2}
dX_{t}=-\nabla U(X_{t})dt+\sqrt{2}dW_{t},
\end{equation}
where $B_{t}$ is a standard $d$-dimensional Brownian motion. Under some mild assumptions on $U(\cdot)$, the diffusion \eqref{eq:overdamped-2} admits a unique stationary distribution \eqref{eqn:target:unconstrained} \cite{chiang1987diffusion,stroock-langevin-spectrum}.

In recent years, there has been growing interest in sampling a target distribution \eqref{eqn:target:unconstrained}, where the potential function $U(x)$ is \textit{non-differentiable} such that the classical OLD \eqref{eq:overdamped-2} is not applicable since
it requires the existence of the gradient of $U$ everywhere in $\mathbb{R}^{d}$. Several approaches modify the non-differentiable potential function to enable gradient based sampling. \cite{Zhou14} proposed a gradient based adaptive stochastic search method that smooths the original objective by integrating it against a parameterized family of exponential densities, producing a differentiable surrogate. \cite{Durmus18} introduced a proximal MCMC method using the Moreau-Yoshida envelope (MYE) to regularize the nonsmooth component. Inspired by \cite{Durmus18}, \cite{Goldman21} applied the Moreau-Yoshida envelope in both overdamped and underdamped Langevin algorithms and reported empirical success across discrete densities. Similarly, \cite{Mou19} developed the Metropolis adjusted proximal gradient Langevin dynamics (MAPGLD) with theoretical convergence guarantees. \cite{SR2020} introduced Stochastic Proximal Langevin Algorithm (SPLA), which further extends this line of work to nonsmooth convex functions that may take infinite values. For subgradient based approaches, \cite{Durmus19} proposed the Stochastic Subgradient Langevin Dynamics (SSGLD), which relaxes the differentiability requirement by using subdifferentials. \cite{Chatterji20} provided further convergence results under similar conditions. Another strategy for handling nonsmoothness is Gaussian smoothing. \cite{nesterov2017random} provide guarantees for nonsmooth, Lipschitz, and convex objectives, while \cite{Chatterji20} analyze smoothed Langevin dynamics under convexity and subdifferentiability assumptions.

Very recently, \cite{GNZZ2025} introduced \textit{anchored Langevin dynamics} (ALD) from a probabilistic perspective:
\begin{equation}\label{anchored:SDE}
dX_{t}=-\nabla U_{0}(X_{t})e^{(U-U_{0})(X_{t})}dt+\sqrt{2}e^{\frac{1}{2}(U-U_{0})(X_{t})}dW_{t},
\end{equation}
where $U,U_{0}:\mathbb{R}^{d}\rightarrow\mathbb{R}$ are continuous functions, and $W_{t}$ is a standard Brownian motion in $\mathbb{R}^d$. In \eqref{anchored:SDE}, $U_{0}$ plays the role of an \textit{anchored reference potential}. Note that in \eqref{anchored:SDE}, even though we require $U_{0}$ to be differentiable, $U$ may not be differentiable. It is proved in \cite{GNZZ2025} that under mild conditions, \eqref{anchored:SDE} admits $\pi\propto e^{-U(x)}$ as the unique invariant distribution. Hence \eqref{anchored:SDE} can be used to sample $\pi\propto e^{-U(x)}$ where $U(x)$ is not necessarily differentiable. It is known that the classical overdamped Langevin dynamics fails to sample heavy tailed distributions with exponential ergodicity. In other words, it does not converge to the target exponentially fast in time, so convergence can be slow even when convergence holds. See e.g. \cite{RT96}. As a remedy, it is shown in \cite{GNZZ2025} that the anchored Langevin dynamics \eqref{anchored:SDE} can sample a \textit{heavy tailed} Gibbs distribution $\pi\propto e^{-U(x)}$ with convergence being exponentially fast in time.

On the other hand, there has also been a growing interest in the machine learning literature to study constrained sampling problem, where the target distribution is supported on a constrained domain:
\begin{equation}\label{target:dist:constrained}
\pi\propto e^{-U(x)},\qquad x\in\mathcal{C}.
\end{equation} 
A standard approach is to discretize \textit{reflected Langevin dynamics} (RLD), namely an overdamped Langevin diffusion with reflected boundary. In the literature, \cite{bubeck2015finite, bubeck2018sampling} studied the \textit{projected Langevin Monte Carlo} (PLMC) algorithm for constrained sampling that projects the iterates back to the constraint set after applying the Langevin step. Since the seminal works \cite{bubeck2015finite, bubeck2018sampling} where the density $\pi$ is assumed to be log concave, there have been many studies proposing various Langevin algorithms for constrained sampling.
\cite{Lamperski2021} considers the projected stochastic gradient Langevin dynamics in the setting of nonconvex smooth Lipschitz $U$ on a convex body where the gradient noise is assumed to have finite variance with a uniform sub Gaussian structure. See also \cite{zheng2022constrained}.
Proximal Langevin Monte Carlo is proposed in \cite{Brosse} for constrained sampling. See also \cite{SR2020} for a study on the proximal stochastic gradient Langevin algorithm from a primal dual perspective. Proximal Langevin methods are particularly relevant for nonsmooth composite potentials when the corresponding proximal map is tractable. Our approach is different: the anchored construction keeps the constrained reflected dynamics explicit and, in the constant reference case, leads to a projected zeroth order scheme that only requires evaluations of $U$, see also \cite{2026IslamZhu} for a study in the decentralized setting.
Langevin algorithms based on mirror descent (see e.g. \cite{hsieh2018mirrored,Chewi2020,Zhang2020,TaoMirror2021,Ahn2021}) can also be used for constrained sampling. This line was initiated by \cite{hsieh2018mirrored}, inspired by classical mirror descent in optimization.
More recently, inspired by the penalty method in the optimization literature, penalized Langevin Monte Carlo algorithms are proposed and studied in \cite{GHZ2022}, where the potential function $U$ can be nonconvex in general. In addition, constrained nonconvex exploration combined with replica exchange Langevin dynamics was proposed and studied in \cite{constraint_replica}.
Recently, nonreversible Langevin algorithms with skew reflected boundary based on skew reflected nonreversible Langevin dynamics (SRNLD) were proposed and studied in \cite{DFTWZ2025}, where an antisymmetric matrix $J$ is added to the dynamics, that can accelerate the convergence of the (reversible) projected Langevin Monte Carlo.
However, how to choose the antisymmetric matrix $J$ in practice is a challenging problem. Very recently, \cite{WTWZ2025} showed acceleration for constrained sampling through the lens of large deviations and asymptotic variance for SRNLD under the additional constraint that $J(x)\nu(x)=0$, where $\nu(x)$ is the outward unit normal vector. For ball constraints, and more generally for constrained domains characterized by sublevel sets that include smoothed $\ell_{p}$ balls as special cases, \cite{WTWZ2025} provided an explicit construction of $J(x)$. Under this additional condition, skew reflection becomes regular reflection. Extensive numerical experiments using both synthetic and real data in \cite{WTWZ2025} showed improved performance for choices of $J$ satisfying this condition.

In this paper, we study reflected anchored Langevin algorithms for constrained sampling with non-differentiable targets. In particular, the constant reference construction gives a projected zeroth order scheme that requires no gradient evaluations of the target potential. In addition to handling nonsmooth constrained targets, we also identify mechanisms through which the state dependent scaling factor can improve mixing relative to PLMC, both at the level of Dirichlet forms and through local bottleneck crossing probabilities.

Consider the constrained sampling problem \eqref{target:dist:constrained}, 
where $\mathcal{C}\subsetneq\mathbb{R}^{d}$, and $U(x)$ is not necessarily differentiable. In this paper, we propose \textit{reflected anchored Langevin dynamics} (RALD):
\begin{equation}\label{eqn:RALD}
dX_{t}=-\nabla U_{0}(X_{t})e^{(U-U_{0})(X_{t})}dt+\sqrt{2}e^{\frac{1}{2}(U-U_{0})(X_{t})}dW_{t}-\nu_tL(dt),
\end{equation}
where $U_{0}:\mathbb{R}^{d}\rightarrow\mathbb{R}$ plays the role as an anchored reference potential, $W_{t}$ is a standard Brownian motion in $\mathbb{R}^d$, and $\nu_t$ is an outward unit normal vector at $X_t$ for $L(dt)$ almost every $t$. The term $-\nu_tL(dt)$ reflects the process at the boundary and ensures $X_{t}\in\mathcal{C}$ for every $t$ given that $X_{0}\in\mathcal{C}$. In particular, $K_t:=\int_{0}^{t}\nu_sL(ds)$ is an outward normal bounded variation regulator and the measure $L(dt)$ is such that $L([0,t])$ is finite, $L(dt)$ is supported on $\{t:X_{t}\in\partial\mathcal{C}\}$. In~\eqref{eqn:RALD}, the \emph{state dependent scaling factor} $a(x):=e^{U(x)-U_0(x)}$ multiplies the reference drift $-\nabla U_0(x)$ by $a(x)$ and the noise amplitude $\sqrt{2}$ by $\sqrt{a(x)}$. Thus the instantaneous noise covariance is $2a(x)I_d$. We refer to $a$ as the scaling factor below.

To implement \eqref{eqn:RALD} in practice, based on its discretization, we propose \textit{reflected anchored Langevin Monte Carlo} (RALMC):
\begin{equation}\label{eqn:RALMC}
x_{k+1}=\mathcal{P}_{\mathcal{C}}\left(x_{k}-\eta\nabla U_{0}(x_{k})e^{(U-U_{0})(x_{k})}+\sqrt{2\eta}e^{\frac{1}{2}(U-U_{0})(x_{k})}\xi_{k+1}\right),
\end{equation}
where $\mathcal{P}_{\mathcal{C}}$ is the projection onto the set $\mathcal{C}$, defined as $\mathcal P_{\mathcal{C}}(x):=\arg\min_{y\in \mathcal{C}}\|x-y\|$ for any $x\in\mathbb{R}^{d}$, $\eta>0$ is the stepsize, or learning rate, $\xi_{k}$ are i.i.d. Gaussian random vectors $\mathcal{N}(0,I_{d})$.

For ALD \eqref{anchored:SDE} to sample $\pi\propto e^{-U(x)}$ on $\mathbb{R}^{d}$, in practice, one needs to choose $U_{0}$ such that the difference $U(x)-U_{0}(x)$ is bounded so that the term $e^{(U-U_{0})(x)}$ will not cause instability. Indeed, \cite{GNZZ2025} chooses $U_{0}(x)$ to be the Gaussian smoothing of $U(x)$ to sample $\pi\propto e^{-U(x)}$ for non-differentiable $U(x)$. On the other hand, RALD \eqref{eqn:RALD} is constrained on a compact domain $\mathcal{C}$ and the term $e^{(U-U_{0})(x)}$ is automatically uniformly bounded on $\mathcal{C}$, which makes the implementation more desirable. To sample $\pi\propto e^{-U(x)}$ for non-differentiable $U(x)$ on $\mathcal{C}$, one may choose $U_{0}(x)\equiv 0$ such that \eqref{eqn:RALD} becomes: \footnote{Note that one can also choose $U_{0}(x)\equiv c_{0}$ for some constant $c_{0}$. Then \eqref{eqn:RALD} becomes: $dX_{t}=\sqrt{2}\sigma_{0}e^{\frac{1}{2}U(X_{t})}dW_{t}-\nu_tL(dt)$ with $\sigma_{0}=e^{-\frac{1}{2}c_{0}}$, and \eqref{eqn:RALMC} becomes: $x_{k+1}=\mathcal{P}_{\mathcal{C}}\left(x_{k}+\sqrt{2\eta}\sigma_{0}e^{\frac{1}{2}U(x_{k})}\xi_{k+1}\right)$. By taking the stepsize to be $\eta\sigma_{0}^{2}$, one can see that for the discrete algorithm, choosing $U_{0}(x)\equiv c_{0}$ is equivalent to $U_{0}(x)\equiv 0$.}
\begin{equation}\label{eqn:RALD:zero}
dX_{t}=\sqrt{2}e^{\frac{1}{2}U(X_{t})}dW_{t}-\nu_tL(dt),
\end{equation}
and \eqref{eqn:RALMC} becomes:
\begin{equation}\label{eqn:RALMC:zero}
x_{k+1}=\mathcal{P}_{\mathcal{C}}\left(x_{k}+\sqrt{2\eta}e^{\frac{1}{2}U(x_{k})}\xi_{k+1}\right).
\end{equation}
The algorithm \eqref{eqn:RALMC:zero} is a simple zeroth order algorithm that can sample $\pi\propto e^{-U(x)}$ on $\mathcal{C}$ without access to the gradient of $U(x)$.

The main contributions are as follows.

\begin{itemize}
\item
We introduce \emph{reflected anchored Langevin dynamics} (RALD), which modifies both the drift and diffusion terms of reflected Langevin dynamics (RLD) through an anchored reference potential $U_0(x)$ and gives a state dependent diffusion term. The discretization of RALD yields \emph{reflected anchored Langevin Monte Carlo} (RALMC), a Markov chain Monte Carlo algorithm that can sample a target distribution $\pi\propto e^{-U(x)}$ supported on a constrained domain $\mathcal{C}$, where $U$ is not necessarily differentiable.
\item
For RALD, a continuous time SDE with reflected boundary, we prove that it preserves the target distribution $\pi\propto e^{-U(x)}$, $x\in\mathcal{C}$ as an invariant distribution under regularity assumptions on the boundary (Theorem~\ref{thm:Gibbs}). Using reversibility and the associated Dirichlet form established in Proposition~\ref{prop:continuous:structure}, we prove an explicit convergence guarantee for RALD to the target distribution in $\chi^{2}$ divergence when the target distribution satisfies a Poincar\'{e} inequality (Theorem~\ref{thm:chi:square}) and, as a corollary, in $\mathcal W_2$ distance (Corollary~\ref{cor:W:2}).
\item
For RALMC, we obtain discretization error bounds in $\mathcal W_2$ distance under two sets of assumptions on the reference function $U_0$. Theorem~\ref{thm:final} treats the dissipative drift setting and gives the explicit iteration bound in Corollary~\ref{cor:complexity}. For a constant reference, Theorem~\ref{thm:U:0:const:strong} combines fixed block projected Euler-Maruyama estimates with long time stability of the continuous semigroup. Its discretization error constant is independent of the total running time. Corollary~\ref{cor:complexity:const} gives $\widetilde{\mathcal O}(\ep^{-4})$ iterations for fixed model parameters, the dependence on heat kernel and geometric quantities remains explicit and is not claimed to be dimension independent.
\item
Numerical experiments compare RALMC with PLMC on synthetic constrained sampling problems and real data Bayesian inference tasks. For the truncated Laplace distribution and the truncated Gaussian mixture model, RALMC with a constant reference potential converges faster than PLMC in $\mathcal W_2$ distance. For constrained Bayesian logistic regression with non-differentiable Laplace, SCAD, and MCP priors, experiments on the MAGIC Gamma Telescope and Breast Cancer Wisconsin datasets show that RALMC with a constant reference improves accuracy at small stepsizes, while RALMC with a Gaussian smoothed reference remains stable at larger stepsizes where the constant reference variant can become unstable.
\end{itemize}

\subsection*{Notations}

We summarize the notations below.
\begin{itemize}
\item For a differentiable function $f:\mathbb{R}^d\to\mathbb{R}$, $\nabla f:=(\partial_1 f,\cdots,\partial_d f)$.
\item For a vector $x\in\mathbb{R}^{d}$, let $\Vert x\Vert:=\sqrt{\langle x,x\rangle}$ be the Euclidean norm.
\item For a random variable $X$, define $\Vert X\Vert_{p}:=\left(\mathbb{E}[|X|^p]\right)^{1/p}$ for any $p\geq 1$.
\item
For a matrix $A=(a_{ij})_{1\leq i,j\leq d}\in\mathbb{R}^{d\times d}$, let $\Vert A\Vert_{\mathrm{F}}:=\sqrt{\text{tr}(AA^{\top})}$ be the Frobenius norm.
\item
Let $(\mathcal{M},d)$ be a metric space that is Polish. For any $\mu,\nu\in\mathcal{P}_{2}(\mathcal{M})$, the space consisting of all probability measures on $M$ with the finite second moment (based on the metric $d$), the 2-Wasserstein distance $\mathcal{W}_{2}$ (see e.g. \cite{villani2008optimal}) between $\mu$ and $\nu$ is defined as: $\mathcal{W}_{2}(\mu,\nu):=\left(\inf\mathbb{E}\left[\left(d(Y,Z)\right)^{2}\right]\right)^{1/2}$, where the infimum is taken over all joint distributions of the random variables $Y,Z$ with marginal distributions $\mu,\nu$ respectively.
\item
The total variation (TV) distance between two probability measures $\mu$ and $\nu$ on a sigma algebra $\mathcal{F}$ is defined as $\mathrm{TV}(\mu,\nu):=\Vert\mu-\nu\Vert_{\mathrm{TV}}:=\sup_{A\in \mathcal{F}} |\mu(A)-\nu(A)|$.
\item
The KL divergence between two probability measures $\mu$ and $\nu$ on a measurable space $\mathcal{X}$, where $\mu$ is absolutely continuous with respect to $\nu$, is defined as: $\mathrm{KL}(\mu\Vert\nu):=\int_{\mathcal{X}}\frac{d\mu}{d\nu}\log\left(\frac{d\mu}{d\nu}\right)d\nu.$
\item
The $\chi^{2}$ divergence between two probability measures $\mu$ and $\nu$ on a measurable space $\mathcal{X}$, where $\mu$ is absolutely continuous with respect to $\nu$, is defined as: $\chi^{2}(\mu\Vert\nu)=\int_{\mathcal{X}}\left(\frac{d\mu}{d\nu}-1\right)^{2}d\nu$.
\end{itemize}

\section{Main Results}

The following geometry for the constrained domain 
is assumed throughout the main results and their proofs, all volume integrals are over the interior.

\begin{assumption}\label{asmpt:set}
The constraint set $\mathcal{C}\subset\mathbb{R}^d$ is bounded, closed, convex, and has nonempty interior. In addition, we assume $0\in\mathcal{C}$ and $\mathcal{C}\subset B(0,R)$ for some $R>0$.
\end{assumption}
We use the outward normal cone
\begin{equation}\label{eq:outward:normal:cone}
N_{\mathcal C}^{\rm out}(x)
:=\left\{v\in\mathbb R^d:
\langle v,y-x\rangle\leq0\ \text{for every }y\in\mathcal C\right\},
\qquad x\in\mathcal C.
\end{equation}
The regulator $K$ points outward, so the normal Skorokhod reflection of a driving path $Y$ is written as
\begin{equation}\label{eq:skorokhod:sign:convention}
X_t=Y_t-K_t,
\qquad
dK_t\in N_{\mathcal C}^{\rm out}(X_t).
\end{equation}
Here $Y$ may be right continuous with left limits, $X_t\in\mathcal C$, and $K_0=0$. The regulator has finite variation on each bounded time interval, $d|K|$ is supported on $\{t:X_t\in\partial\mathcal C\}$, and $dK_t=\nu_t\,d|K|_t$ with $\nu_t\in N_{\mathcal C}^{\rm out}(X_t)\cap\mathbb S^{d-1}$ for $d|K|$ almost every $t$, see \cite{tanaka1979stochastic}. We define
\begin{equation}
\label{eqn:notation}
b(x):=-\nabla U_0(x) e^{(U-U_0)(x)},\quad \sigma(x):= e^{\frac{1}{2}(U-U_0)(x)}.
\end{equation}
Consider the reflected anchored Langevin dynamics, i.e. the continuous time RALD in~\eqref{eqn:RALD}, with the notations in~\eqref{eqn:notation}:
\begin{equation}\label{simplified eqn:RALD}
    dX_t = b(X_t)dt+\sqrt{2}\sigma(X_t)dW_t-dK_t,
    \qquad dK_t=\nu_t L(dt),
\end{equation}
where $W_t$ is a standard Brownian motion in $\mathbb R^d$, $L(dt)=d|K|_t$, and $|K|_t$ denotes the total variation of $K$ on $[0,t]$. The corresponding driving process satisfies
\begin{equation*}
     dY_t =  b(X_t) dt +\sqrt{2}\sigma(X_t) dW_t , \qquad Y_0 = X_0\in\mathcal{C}.
\end{equation*}

\subsection{Continuous time analysis}\label{sec:continuous}

We first state the boundary and potential regularity conditions, establish well posedness, and identify the Neumann form. The resulting heat kernel and energy identities give invariance and convergence.

\begin{assumption}[Continuous time regularity]\label{asmpt:continuous:regularity}
In addition to Assumption~\ref{asmpt:set}, suppose that $\partial\mathcal C$ is $C^2$ or $C^{1,1}$, that $U-U_0$ is continuous on $\overline{\mathcal C}$, and that $U_0\in C^2(\overline{\mathcal C})$.
\end{assumption}

\begin{lemma}[Well posedness by time change]\label{lem:reflected:timechange}
Under Assumption~\ref{asmpt:continuous:regularity}, the normally reflected martingale problem with Neumann boundary condition for
$$
\mathcal Lf=e^{U-U_0}\bigl(\Delta f-\langle\nabla U_0,\nabla f\rangle\bigr)
$$
is well posed for every initial point in $\mathcal C$. Write $(P_t)_{t\geq0}$ for its transition semigroup.
\end{lemma}
\begin{proof}
The proof is given in Appendix~\ref{proof:lem:reflected:timechange}.
\end{proof}
Lemma~\ref{lem:reflected:timechange} constructs RALD and establishes uniqueness in law. The next proposition identifies the analytic structure of its transition semigroup: part~(i) records the generator on Neumann test functions, while parts~(ii)-(iii) establish the associated symmetric form and heat kernel properties. The uniqueness proved in the lemma is essential for identifying the form semigroup with that of RALD.

We write $H^1(\mathcal C)$ for the Sobolev space of functions whose values and first weak derivatives are square integrable on the interior of $\mathcal C$ with respect to Lebesgue measure, with norm
$$
\|f\|_{H^1(\mathcal C)}^2:=\int_{\mathcal C}(|f|^2+\|\nabla f\|^2)\,dx.
$$
No boundary condition is imposed on this space. And we get the following proposition.

\begin{proposition}[Generator, Neumann form, and heat kernel]\label{prop:continuous:structure}
Under Assumption~\ref{asmpt:continuous:regularity}, let $(P_t)_{t\geq0}$ be the transition semigroup from Lemma~\ref{lem:reflected:timechange}, and write $Z:=\int_{\mathcal C}e^{-U(x)}dx$ and $\pi(dx):=Z^{-1}e^{-U(x)}\mathbf1_{\mathcal C}(x)dx$. Then:
\begin{enumerate}[(i)]
\item For $f\in C^2(\overline{\mathcal C})$ satisfying the Neumann condition
\begin{equation}\label{eqn:boundary}
\nabla f\cdot\nu=0\quad\text{on }\partial\mathcal C,
\end{equation}
the generator acts as
\begin{equation}\label{eqn:generator}
\mathcal Lf=e^{U-U_0}\bigl(\Delta f-\langle\nabla U_0,\nabla f\rangle\bigr).
\end{equation}
\item The form
$$
\mathcal E(f,g):=\frac1Z\int_{\mathcal C}e^{-U_0(x)}\langle\nabla f(x),\nabla g(x)\rangle dx,\qquad \mathcal D(\mathcal E)=H^1(\mathcal C),
$$
is a densely defined closed symmetric Markov form on $L^2(\pi)$ whose semigroup is $P_t$. Its nonpositive self adjoint generator agrees with $\mathcal L$ on the Neumann test functions in (i). Moreover, $P_t1=1$ and, for every $f\in L^2(\pi)$ and $t>0$,
$$
P_tf\in\mathcal D(\mathcal L),\qquad
\frac{d}{dt}\|P_tf\|_{L^2(\pi)}^2=-2\mathcal E(P_tf,P_tf).
$$
\item The semigroup $P_t$ admits a jointly continuous, strictly positive Neumann heat kernel $\kappa_t(x,y)$ relative to $\pi$ on $(0,\infty)\times\mathcal C\times\mathcal C$, including boundary points. The Lebesgue transition density $p(t,x,y)=\kappa_t(x,y)Z^{-1}e^{-U(y)}$ has the same continuity and positivity properties. In particular, for every $s>0$,
$$
0<\min_{x,y\in\mathcal C}\kappa_s(x,y)\leq\max_{x,y\in\mathcal C}\kappa_s(x,y)<\infty.
$$
\end{enumerate}
\end{proposition}
\begin{proof}
The proof is given in Appendix~\ref{sec:proof:continuous:structure}.
\end{proof}

Proposition~\ref{prop:continuous:structure} underpins the convergence analysis. Symmetry and conservativity in part~(ii) give invariance and reversibility, and the energy identity yields $\chi^2$ decay when combined with the Poincar\'e inequality. The heat kernel bounds in part~(iii) ensure uniqueness of the invariant distribution and supply the constants for the long time stability estimate in Lemma~\ref{lem:const:semigroup:stability}, which is used in the time uniform discretization bound of Theorem~\ref{thm:U:0:const:strong}. We first identify the invariant distribution in the next theorem.

\begin{theorem}\label{thm:Gibbs}
Under Assumption~\ref{asmpt:continuous:regularity}, the Gibbs distribution $\pi(dx)\propto e^{-U(x)}\mathbf{1}_{\mathcal C}(x)\,dx$ is the unique invariant distribution of \eqref{eqn:RALD}.
\end{theorem}

\begin{proof}
The proof is given in Appendix~\ref{sec:thm:Gibbs}.
\end{proof}

Let $\mu_t=\Law(X_t)$. Under Assumption~\ref{asmpt:continuous:regularity}, $\pi$ admits a finite Poincar\'e constant, as verified in Appendix~\ref{sec:thm:chi:square}. The following theorem combines any valid constant $C_P$ with the energy identity in Proposition~\ref{prop:continuous:structure}~(ii) to quantify convergence.

\begin{theorem}\label{thm:chi:square}
Under Assumption~\ref{asmpt:continuous:regularity}, assume that $\mu_0\ll\pi$ and $\chi^2(\mu_0\Vert\pi)<\infty$. If $\pi$ satisfies the Poincar\'{e} inequality
$$
    \operatorname{Var}_{\pi}(f)
    \leq
    C_P\int_{\mathcal C}\|\nabla f\|^2\,d\pi
$$
for every $f\in H^1(\mathcal C)$, \footnote{This is the standard spectral gap form of the Poincar\'e inequality. See e.g. \cite{Bakry2008,Bakry2013}.}then
\begin{equation}
	\chi^{2}(\mu_{t}\Vert\pi)\leq
	\chi^{2}(\mu_{0}\Vert\pi)e^{-2\alpha_{0} t/C_{P}},
\end{equation}
provided that $\alpha_{0}:=e^{\inf_{x\in\mathcal{C}}(U(x)-U_{0}(x))}>0$.
\end{theorem}

\begin{proof}
The proof is given in Appendix~\ref{sec:thm:chi:square}.
\end{proof}

Via total variation, the $\chi^2$ estimate on the bounded constraint set converts to the $\mathcal{W}_2$ bound
in the following corollary.

\begin{corollary}\label{cor:W:2}
Under the assumptions of Theorem~\ref{thm:chi:square}, assume that $\mathcal{C}$ is contained in a ball with radius $R>0$ and centered at zero. Then, we have
\begin{equation}
\mathcal{W}_{2}(\mu_{t},\pi)\leq 2^{3/4}R\left(\chi^{2}(\mu_{0}\Vert\pi)\right)^{1/4}\cdot e^{-\alpha_{0} t/(2C_{P})},
\end{equation}
where $\alpha_{0}$ and $C_{P}$ are defined in Theorem~\ref{thm:chi:square}.
\end{corollary}

\begin{proof}
The proof is given in Appendix~\ref{sec:cor:W:2}.
\end{proof}

In the unconstrained setting, anchored Langevin dynamics admits a random time change interpretation \cite{GNZZ2025}. With reflection, the same time change also acts on the boundary local time. The finite time estimates below therefore use a direct coupling of the reflected diffusion and the projected Euler-Maruyama scheme.

\subsection{Discrete time analysis}\label{sec:discrete}

We compare RALMC with RALD by coupling their driving Brownian motions. We first construct the projected process and its regulator, then state the coefficient assumptions and error bounds.

In order to couple the continuous time RALD ~\eqref{simplified eqn:RALD} with the discrete time algorithm ~\eqref{eqn:RALMC}, we follow the Skorokhod problem in Section 2.1 of~\cite{bubeck2015finite}. We first introduce the discrete dynamics and define the unconstrained iterative updates:
\begin{equation}\label{y:bar:process}
\bar{y}_k:=\bar{y}_{(k-1)}+\eta b\left(x_{k-1}\right)+\sqrt{2 \eta} \sigma\left(x_{k-1}\right) \xi_k, \quad \bar{y}_0=x_0,
\end{equation}
where $\xi_k \sim \mathcal{N}\left(0, I_d\right)$ are i.i.d. Gaussian random vectors and $x_{k}$ are the iterates of RALMC in \eqref{eqn:RALMC}. By the definition of RALMC in \eqref{eqn:RALMC}, we have
\begin{equation}
\label{eqn:x&y}
x_k=\mathcal{P}_C\left(x_{k-1}+\bar{y}_k-\bar{y}_{k-1}\right) .
\end{equation}

Next, we define the piecewise continuous time interpolation of the $\bar{y}_{k}$ process in \eqref{y:bar:process} such that
$$
\bar{Y}_t:=\bar{y}_{[t / \eta]}, \quad t \geq 0.
$$
Then $\bar{Y}_t = \bar{y}_k$ is constant on each interval $[k\eta,(k+1)\eta)$ and it satisfies
\begin{equation}
\label{eq:bar:Y:x}
\bar{Y}_{k \eta}-\bar{Y}_{(k-1) \eta}=\eta b\left(x_{k-1}\right)+\sqrt{2 \eta} \sigma\left(x_{k-1}\right) \xi_k, \qquad \bar{y}_0 = x_0,
\end{equation}
where $\xi_k \sim \mathcal{N}\left(0, I_d\right)$ are i.i.d. Gaussian random vectors and $x_{k}$ are the iterates of RALMC in \eqref{eqn:RALMC}. We now construct the outward regulator of the piecewise constant driver $\bar Y$. Set $\bar X_0=\bar Y_0=X_0$, $\bar K_0=0$, and, for $k\geq1$,
\begin{align}
\Delta\bar Y_k
&:=\bar Y_{k\eta}-\bar Y_{(k-1)\eta},
\qquad
z_k:=\bar X_{(k-1)\eta}+\Delta\bar Y_k,\nonumber\\
\Delta\bar K_k
&:=z_k-\mathcal P_{\mathcal C}(z_k),
\qquad
\bar X_{k\eta}:=\mathcal P_{\mathcal C}(z_k),
\qquad
\bar K_t:=\sum_{j=1}^{\lfloor t/\eta\rfloor}\Delta\bar K_j,
\qquad t\geq0.\label{eq:discrete:regulator:increment}
\end{align}
The variational characterization of the orthogonal projection gives
\begin{equation}\label{eq:projection:outward:normal}
\Delta\bar K_k\in
N_{\mathcal C}^{\rm out}(\bar X_{k\eta}).
\end{equation}
If $\Delta\bar K_k\neq0$, write $\Delta\bar L_k:=\|\Delta\bar K_k\|$ and $\bar\nu_k:=\Delta\bar K_k/\Delta\bar L_k$. Then $\bar\nu_k\in N_{\mathcal C}^{\rm out}(\bar X_{k\eta})\cap\mathbb S^{d-1}$, and
\begin{equation}\label{eq:bar:K:normal:representation}
\bar K_t
=\sum_{j=1}^{\lfloor t/\eta\rfloor}\bar\nu_j\Delta\bar L_j,
\qquad
|\bar K|_t
=\sum_{j=1}^{\lfloor t/\eta\rfloor}\Delta\bar L_j.
\end{equation}
Here a summand with $\Delta\bar K_j=0$ is set equal to zero.

The identity $\bar X=\bar Y-\bar K$ follows by induction. Indeed, if $\bar X_{(k-1)\eta}=\bar Y_{(k-1)\eta}-\bar K_{(k-1)\eta}$, then
\begin{align}
\bar Y_{k\eta}-\bar K_{k\eta}
&=\bar Y_{(k-1)\eta}+\Delta\bar Y_k
-\bar K_{(k-1)\eta}-\Delta\bar K_k\nonumber\\
&=\bar X_{(k-1)\eta}+\Delta\bar Y_k-\Delta\bar K_k
=z_k-\bigl(z_k-\mathcal P_{\mathcal C}(z_k)\bigr)
=\bar X_{k\eta}.\label{eq:discrete:skorokhod:identity}
\end{align}
Both sides are constant on $[k\eta,(k+1)\eta)$, so $\bar X_t=\bar Y_t-\bar K_t$ for every $t\geq0$. Comparing \eqref{eq:discrete:regulator:increment} with~\eqref{eqn:x&y} under the coupling $\xi_k=(W_{k\eta}-W_{(k-1)\eta})/\sqrt\eta$ gives
\begin{equation}\label{eq:grid:pathwise:identity}
x_k=\bar X_{k\eta}\quad\text{almost surely}.
\end{equation}
Replacing $x_k$ by $\bar X_{k\eta}$ therefore gives
\begin{equation}
    \bar{X}_t = X_{0} +\int_0^{\lfloor\frac{t}{\eta}\rfloor\eta}b\left(\bar{X}_{\lfloor\frac{s}{\eta}\rfloor\eta}\right)ds+\int_0^{\lfloor\frac{t}{\eta}\rfloor\eta}\sqrt{2}\sigma\left(\bar{X}_{\lfloor\frac{s}{\eta}\rfloor\eta}\right)dW_s- \bar{K}_t, \qquad t\in [k\eta,(k+1)\eta).
    \label{eqn:bar:x}
\end{equation}
In addition, $\bar{Y}_{k\eta}$ in~\eqref{eq:bar:Y:x} can be written as
\begin{equation}
\label{eqn:bar:y}
\bar{Y}_{k\eta} := \bar{Y}_{(k-1)\eta}+
\eta b\left(\bar{X}_{(k-1)\eta}\right)+
\sqrt{2\eta}\sigma\left(\bar{X}_{(k-1)\eta}\right)\xi_k, \qquad\bar{Y}_0=X_0,
\end{equation}
where $\xi_{k} \sim \mathcal{N}(0,I_d)$ are i.i.d. Gaussian random vectors.

Then we impose regularity and dissipativity assumptions directly on the effective coefficients $b$ and $\sigma$ defined in~\eqref{eqn:notation} in the following assumptions.

\begin{assumption}\label{asmpt:strongly:convex}
There exists $m>0$ such that:
\begin{equation}
\langle b(x)-b(y),\,x-y\rangle \leq -m\|x-y\|^2,
\qquad \text{for any $x,y\in{\mathcal{C}}$}.
\end{equation}
\end{assumption}

\begin{assumption}\label{asmpt:smooth}
The drift $b:\mathcal{C}\to\mathbb{R}^d$ and the scalar diffusion coefficient $\sigma:\mathcal{C}\to(0,\infty)$ are Lipschitz on $\mathcal{C}$. Viewing the diffusion matrix as $\sigma(x)I_d$, there exist constants $M>0$ and $\alpha\geq0$ such that for all $x,y\in\mathcal{C}$,
\begin{equation}
\|b(x)-b(y)\| \leq M\|x-y\|, \quad \|\sigma(x)I_d-\sigma(y)I_d\|_{\mathrm{F}} \leq \sqrt{\alpha}\|x-y\|,
\end{equation}
\end{assumption}
Under Assumptions~\ref{asmpt:set} and~\ref{asmpt:smooth}, composing $b$ and $\sigma$ with the projection $\mathcal P_{\mathcal C}$ gives globally Lipschitz extensions to $\mathbb R^d$. Thus \cite[Theorem~4.1]{tanaka1979stochastic} ensures that the reflected SDE has a pathwise unique strong solution without additional smoothness assumptions on $\partial\mathcal C$. Moreover, since $\mathcal C\subset B(0,R)$, the same Lipschitz bounds imply, for every $x\in\mathcal C$,
\begin{equation}\label{eqn:b:sigma:bdd}
\|b(x)\|\leq B_b:=\|b(0)\|+MR,\qquad
\|\sigma(x)I_d\|_{\mathrm F}\leq S_\sigma:=\|\sigma(0)I_d\|_{\mathrm F}+\sqrt{\alpha}R.
\end{equation}
These regularity and boundedness properties suffice for the regulator and bridge estimates below. The discretization estimate in Theorem~\ref{global:error:2} additionally uses Assumption~\ref{asmpt:strongly:convex} with $2m>1+4\alpha$. This condition must be checked for the effective drift $b=-e^{U-U_0}\nabla U_0$, since strong convexity of $U_0$ alone does not ensure dissipativity after multiplication by the state dependent factor. The Gaussian smoothing analysis below provides sufficient conditions. For a constant reference potential $U_0\equiv c_0$, the drift vanishes, so the positive dissipativity condition in Assumption~\ref{asmpt:strongly:convex} cannot hold. This case therefore requires a separate discretization argument, developed in Theorem~\ref{thm:U:0:const:strong} and Corollary~\ref{cor:complexity:const} without a dissipativity assumption.

In order to study the discretization error, we construct all processes using the same Brownian motion and introduce a continuously reflected auxiliary process $\hat X_t$ whose coefficients are frozen at the Euler-Maruyama grid values. This auxiliary process yields the decomposition $\|X_t-\bar X_t\|\leq\|X_t-\hat X_t\|+\|\hat X_t-\bar X_t\|$. The first term is controlled by the dissipativity of the drift and the Lipschitz continuity of the coefficients, following the mean square comparison method in~\cite{liSqrtdDimensionDependence2021}. The second term measures the error caused by replacing continuous reflection with projection at the grid points. Its control requires simultaneous estimates for the continuous regulator $\hat K$ and the projected regulator $\bar K$, using the stability estimates for reflected Euler-Maruyama schemes in~\cite{slominski2001euler}. We note that Girsanov's theorem cannot be used to compare $X_t$ and $\bar X_t$ directly because their diffusion coefficients are different, so the PLMC argument in \cite{bubeck2018sampling} cannot be used directly.

Defining the auxiliary process $\hat X_t$ by
\begin{equation}
\label{eqn:hat:x}
d\hat{X}_t = b\left(\bar{X}_{\lfloor t/\eta \rfloor\eta}\right)dt + \sqrt{2}\sigma\left(\bar{X}_{\lfloor t/\eta \rfloor\eta}\right)dW_t - d\hat{K}_t, \qquad \hat{X}_0=X_0,
\end{equation}
where $d\hat K_t=\hat\nu_t\hat L(dt)$, $\hat\nu_t\in N_{\mathcal C}^{\rm out}(\hat X_t)\cap\mathbb S^{d-1}$ for $\hat L(dt)$ almost every $t$. Thus $\hat X_t=\hat Y_t-\hat K_t\in\mathcal C$ for every $t$. Then $(\hat{X}_t, \hat{K}_t)$ solves the Skorokhod problem associated with
\begin{equation}
\label{eqn:hat:Y}
\hat{Y}_t = X_0 + \int_0^tb\left(\bar{X}_{\lfloor\frac{s}{\eta}\rfloor\eta}\right)ds+\int_0^t\sqrt{2}\sigma\left(\bar{X}_{\lfloor\frac{s}{\eta}\rfloor\eta}\right)dW_s.
\end{equation}
For $t\in[k\eta,(k+1)\eta)$, $\hat Y_t$ evolves with the coefficients frozen at $\bar X_{k\eta}$, and $\hat Y_{k\eta}=\bar Y_{k\eta}$. The process $\hat X$ serves as a continuous bridge between the projected Euler-Maruyama process $\bar X$ and the reflected diffusion $X$. We first prove the following lemmas and then we introduce our main discretization error results. When coupling $\bar{X}_t$ and $\hat{X}_t$, we use Lemma 2.3 in~\cite{bubeck2018sampling} to have the following supporting lemma.
\begin{lemma}
\label{lemma:bound:bm}
There exists a universal constant $C>0$ such that, for every $T\geq e$ and $0<\eta\leq1$,
\begin{equation}
\mathbb{E}\left[\sup_{s,t\in[0,T],|s-t|\leq \eta}\|W_t- W_s\|^2\right]
\leq Cd\eta\log(T/\eta).
\end{equation}
\end{lemma}

\begin{proof}
The proof is given in Appendix~\ref{sec:lemma:bound:bm}.
\end{proof}

For a vector valued function $K$ of bounded variation, let $|K|_t:=\operatorname{Var}_{[0,t]}(K)$. The outward regulator of the bridge satisfies $d\hat K_t=\hat\nu_t\hat L(dt)$, where $\hat L$ is supported on $\{t:\hat X_t\in\partial\mathcal C\}$ and $\|\hat\nu_t\|=1$ for $\hat L(dt)$ almost every $t$. Consequently,
\begin{equation}\label{eq:hat:K:variation}
|\hat K|_t=\int_{[0,t]}\|\hat\nu_s\|\hat L(ds)
=\hat L([0,t]).
\end{equation}
The next lemma gives second moment bounds for the total variations of the continuous regulator $\hat K$ and the projected regulator $\bar K$, which are used in the bridge estimate below.
\begin{lemma}\label{lemma:bound:K}
Under Assumptions~\ref{asmpt:smooth} and~\ref{asmpt:set}, suppose that the initial law is supported in $\mathcal C_{r_0}:=\{x\in\mathcal C:\operatorname{dist}(x,\partial\mathcal C)\geq r_0\}$ for some $r_0>0$. Then, for every $t>0$,
\begin{equation}\label{eq:hat:K:L2:bound}
\mathbb E\left[|\hat K|_t^2\right]\leq At^2+Bt,
\end{equation}
where
\begin{align}
A  :=3r_0^{-2}\left(4R^2B_b^2+S_\sigma^4\right),
\qquad
B  :=24r_0^{-2}R^2S_\sigma^2.
\end{align}
For every $T=k\eta$ with $k\in\mathbb N$ and $0<\eta\leq1$, the bound satisfies
\begin{equation}\label{eq:bar:K:L2:bound}
\mathbb E\left[|\bar K|_T^2\right]\leq \bar A T^2+\bar B T,
\end{equation}
where
\begin{align}
\bar A
:=3r_0^{-2}\left(4R^2B_b^2+2B_b^4+8S_\sigma^4\right),
\qquad
\bar B
:=3r_0^{-2}\left(8R^2S_\sigma^2+16S_\sigma^4\right),
\end{align}
with constants $B_b$ and $S_{\sigma}$ defined in~\eqref{eqn:b:sigma:bdd}.
\end{lemma}

\begin{proof}
The proof is given in Appendix~\ref{sec:lemma:bound:K}.
\end{proof}

The next lemma controls the bridge error uniformly over the entire time interval. The bound is necessary because $\bar Y$ and $\hat Y$ do not agree between grid points.

\begin{lemma}\label{lemma:alg2bri}
Under Assumptions~\ref{asmpt:smooth} and~\ref{asmpt:set}, suppose that the initial law is supported in $\mathcal C_{r_0}:=\{x\in\mathcal C:\operatorname{dist}(x,\partial\mathcal C)\geq r_0\}$ for some $r_0>0$. Let $T=k\eta\geq e$ and $0<\eta\leq1$. Set
\begin{align}
H_K
&:=\sqrt A+\sqrt{B/e}+\sqrt{\bar A}+\sqrt{\bar B/e},\label{eq:bridge:HK}\\
C_0
&:=2B_b^2+4CdS_\sigma^2
+4H_K\left(\sqrt2B_b+2\sqrt{Cd}\,S_\sigma\right),\label{eq:bridge:C0}
\end{align}
where $C$ is the universal constant in Lemma~\ref{lemma:bound:bm}, and $A,B,\bar A,\bar B$ are defined in Lemma~\ref{lemma:bound:K}. Then,
\begin{equation}\label{eq:bridge:uniform}
\mathbb E\left[\sup_{0\leq t\leq T}
\left\|\bar X_t-\hat X_t\right\|^2\right]
\leq C_0\sqrt\eta\,T\log(T/\eta).
\end{equation}
\end{lemma}

\begin{proof}
The proof is given in Appendix~\ref{sec:lemma:alg2bri}.
\end{proof}

We now prove the main discretization error result in the following theorem by controlling the remaining error term.

\begin{theorem}\label{global:error:2}
Under Assumptions~\ref{asmpt:strongly:convex},~\ref{asmpt:smooth} and \ref{asmpt:set}, moreover, we assume that the common initial law is supported in $\mathcal C_{r_0}:=\{x\in\mathcal C:\operatorname{dist}(x,\partial\mathcal C)\geq r_0\}$ for some $r_0>0$. For any $k \in \mathbb{N}$, $0<\eta\leq1$, and $T = k\eta \geq e$, we have
\begin{equation}
\mathcal{W}^2_2\left(\mathrm{Law}(\bar{X}_T),\mathrm{Law}(X_T)\right)  \leq \sqrt{\eta}T\log(T/\eta)
\left(1+\sqrt{\frac{M^2+4\alpha}{2m-1-4\alpha}}\right)^2 C_0,
\end{equation}
provided $2m>1+4\alpha$, where $C_0$ is given in \eqref{eq:bridge:C0}.
\end{theorem}

\begin{proof}
The proof is given in Appendix~\ref{sec:global:error:2}.
\end{proof}

Combining the continuous time convergence of RALD in Corollary~\ref{cor:W:2} with the discretization error between RALD and RALMC in Theorem~\ref{global:error:2} gives the following finite time convergence guarantee for RALMC in $\mathcal W_2$ distance.

\begin{theorem}\label{thm:final}
Suppose that Assumptions~\ref{asmpt:strongly:convex},~\ref{asmpt:smooth}, and~\ref{asmpt:set} hold, and that the hypotheses of Theorem~\ref{thm:chi:square} hold. Assume that the common initial law is supported in $\mathcal C_{r_0}:=\{x\in\mathcal C:\operatorname{dist}(x,\partial\mathcal C)\geq r_0\}$ for some $r_0>0$, and
$$
2m>1+4\alpha.
$$
Define
\begin{equation}\label{eq:thm:final:Cbar}
\bar C_0:=\left(1+\sqrt{\frac{M^2+4\alpha}{2m-1-4\alpha}}\right)\sqrt{C_0},
\end{equation}
where $C_0$ is given in Theorem~\ref{global:error:2}. Then, for every $k\in\mathbb N$ and $0<\eta\leq1$ such that $k\eta\geq e$,
\begin{equation}\label{eq:thm:final:explicit}
\mathcal W_2(\mathrm{Law}(x_k),\pi)
\leq 2^{3/4}R\left(\chi^2(\mu_0\Vert\pi)\right)^{1/4}e^{-\alpha_0k\eta/(2C_P)}
+\bar C_0\eta^{1/4}(k\eta)^{1/2}(\log k)^{1/2},
\end{equation}
where $\alpha_0$ and $C_P$ are defined in Theorem~\ref{thm:chi:square}. In the special case $\sigma\equiv1$, one has $U=U_0$ and RALMC reduces to PLMC for the differentiable target potential. If $\sigma\equiv1$ and $B_b,M,m,R,r_0$ are independent of $d$, then $\bar C_0=\mathcal O(d)$ and
\begin{equation}\label{eq:thm:final:dimension}
\mathcal W_2(\mathrm{Law}(x_k),\pi)
\leq 2^{3/4}R\left(\chi^2(\mu_0\Vert\pi)\right)^{1/4}e^{-\alpha_0k\eta/(2C_P)}
+\widetilde{\mathcal O}\left(d\eta^{1/4}(k\eta)^{1/2}\right).
\end{equation}
\end{theorem}

\begin{proof}
The proof is given in Appendix~\ref{sec:thm:final}.
\end{proof}

As a corollary of Theorem~\ref{thm:final}, we obtain the iteration complexity of RALMC in \eqref{eqn:RALMC}.

\begin{corollary}\label{cor:complexity}
Under the assumptions of Theorem~\ref{thm:final}, fix $\ep>0$ and set
$$
    A_0:=2^{3/4}R\left(\chi^2(\mu_0\Vert\pi)\right)^{1/4},
    \qquad
    T_\ep:=\max\left\{e,
    \frac{2C_P}{\alpha_0}\log\left(1 \vee \frac{2A_0}{\ep}\right)\right\}.
$$
Choose an integer $k$ satisfying
\begin{equation}\label{eq:complexity:valid:k}
    k\geq\max\{3,T_\ep\},
    \qquad
    \frac{k}{(\log k)^2}
    \geq
    \frac{16\bar C_0^4T_\ep^3}{\ep^4},
\end{equation}
and set $\eta:=T_\ep/k$. Then $0<\eta\leq1$, $k\eta=T_\ep\geq e$, and
$$
    \mathcal W_2(\Law(x_k),\pi)\leq\ep.
$$
In particular,
$$
    k
    =\widetilde{\mathcal O}\left(
    T_\ep+\frac{\bar C_0^4T_\ep^3}{\ep^4}
    \right).
$$
Under the additional scaling assumptions in the second part of Theorem~\ref{thm:final}, and when $A_0,C_P,\alpha_0$ are dimension independent, this becomes
\begin{align}\label{under:additional}
    \eta=\widetilde{\mathcal O}\left(\frac{\ep^4}{d^4}\right),
    \qquad
    k=\widetilde{\mathcal O}\left(\frac{d^4}{\ep^4}\right).
\end{align}
\end{corollary}

\begin{proof}
The proof is given in Appendix~\ref{sec:cor:complexity}.
\end{proof}

\subsection{The case of constant \texorpdfstring{$U_0$}{U0}}
As discussed in the introduction, the boundedness of the constraint set makes a constant reference potential a natural and practically useful choice for RALMC. When $U_0 \equiv c_0$, the drift coefficient in~\eqref{eqn:notation} satisfies $b\equiv 0$, and the target potential enters the dynamics only through the state dependent diffusion coefficient $\sigma(x)=e^{\frac12(U(x)-c_0)}$. The resulting algorithm is a zeroth order projected scheme. It only requires evaluations of $U$ and does not require access to $\nabla U$. This feature is useful for constrained sampling with non-differentiable or nonconvex potentials.

When $b\equiv0$, the drift contraction used in the preceding discretization estimate is unavailable. We compare Euler-Maruyama scheme and RALD on a fixed block, then use heat kernel smoothing and the spectral gap to control the propagation of these local errors. The block length is chosen from model parameters, independently of the total running time and desired accuracy. Nondegeneracy alone is not used as a contraction estimate. Throughout the paper, $a(x)$ denotes
\begin{equation}\label{eq:scalar:mobility}
    a(x):=e^{U(x)-U_0(x)}=(\sigma(x))^2.
\end{equation}
\begin{assumption}[Uniform ellipticity]\label{assump:elliptic}
There exist constants $0<\lambda\leq\Lambda<\infty$ such that, for every $x\in\overline{\mathcal C}$,
\begin{equation}\label{eq:uniform:elliptic}
    \lambda I_d
    \preceq
    [\sigma(x)I_d][\sigma(x)I_d]^\top
    =a(x)I_d
    \preceq
    \Lambda I_d.
\end{equation}
\end{assumption}
Since $\sigma(x)>0$, condition~\eqref{eq:uniform:elliptic} is equivalent to
\begin{equation}\label{eq:uniform:elliptic:scalar}
    \lambda\leq a(x)=(\sigma(x))^2\leq\Lambda,
    \qquad x\in\overline{\mathcal C},
\end{equation}
or, equivalently, $\sqrt\lambda\leq\sigma(x)\leq\sqrt\Lambda$. The covariance matrix per unit time of the martingale part in~\eqref{eq:sde:constU0} is therefore $2a(x)I_d$. With $U_0\equiv c_{0}$, consider the reflected anchored Langevin dynamics
\begin{equation}\label{eq:sde:constU0}
dX_t=\sqrt{2}\,\sigma(X_t)\,dW_t-dK_t,\qquad X_t\in\overline{\mathcal C},
\end{equation}
where $\sigma(x)=\exp \left(\frac12(U(x)-c_0)\right)$ and $K_t$ is the outward normal regulator at $\partial\mathcal C$. Equivalently, $dK_t=\nu_tL(dt)$ with $\nu_t\in N_{\mathcal C}^{\rm out}(X_t)\cap\mathbb S^{d-1}$ for $L(dt)$ almost every $t$. At grid times, the discrete dynamics are represented by the piecewise constant Euler-Maruyama scheme
\begin{equation}\label{eq:cont-interp}
\bar{X}_t
=
X_0+\int_0^{\left\lfloor\frac{t}{\eta}\right\rfloor\eta} \sqrt{2}\,
\sigma \left(\bar{X}_{\left\lfloor\frac{s}{\eta}\right\rfloor \eta}\right) dW_s-\bar{K}_t,
\qquad t \in[k \eta,(k+1) \eta).
\end{equation}
Recall that, when $U_{0}\equiv c_{0}$, the reflected Euler-Maruyama scheme is given by
\begin{equation}\label{eq:proj-euler}
\bar{X}_{(k+1) \eta}=\mathcal{P}_{\mathcal{C}}\left(\bar{X}_{k \eta}+\sqrt{2 \eta} \sigma\left(\bar{X}_{k \eta}\right) \xi_{k+1}\right),
\end{equation}
where $\xi_k \sim\mathcal{N}\left(0, I_d\right)$ are i.i.d. random vectors and $X_0 = \bar{X}_0$.

We first compare the projected scheme with the diffusion over a fixed interval of length $s$. The next proposition is uniform over the initial point, including boundary points, so it can be applied again after each block. %
\begin{proposition}[Fixed block projected Euler-Maruyama approximation]
\label{prop:const:strong:euler}
Consider $U_0\equiv c_0$. Suppose Assumption~\ref{asmpt:set} holds and $\sigma$ is Lipschitz on $\mathcal C$. Write
$$
\left\|\left[\sigma(x)-\sigma(y)\right]I_d\right\|_{\mathrm F}\leq\ell_\sigma\|x-y\|.
$$
For every fixed $s>0$, there is a finite constant $I_s$, given explicitly in \eqref{eq:const:Is}, such that, for every integer $n\geq1$ and $\eta=s/n$,
\begin{equation}\label{eq:const:strong:euler:log-rate}
\sup_{x\in\mathcal C}
\left(\mathbb E_x\left\|\bar X_s-X_s\right\|^2\right)^{1/2}
\leq I_s\left[\eta\left(1+\log\frac{s}{\eta}\right)\right]^{1/4}.
\end{equation}
The two processes use the same Brownian motion and initial point. The bound is uniform over all initial points in $\mathcal C$, including its boundary.
\end{proposition}
\begin{proof}
The proof is given in Appendix~\ref{sec:prop:const:strong:euler}.
\end{proof}

Proposition~\ref{prop:const:strong:euler} controls the error produced over one block, but it does not control the accumulation over many blocks. The next lemma controls how the continuous semigroup propagates each block error, the short time estimation applies when the remaining evolution time is less than $2s$, while the long time estimate gives exponential decay when it is at least $2s$.

\begin{lemma}[Long time Wasserstein stability]\label{lem:const:semigroup:stability}
Suppose $U_0\equiv c_0$, Assumptions~\ref{asmpt:set}, \ref{asmpt:continuous:regularity}, and \ref{assump:elliptic} hold, and $\sigma$ is Lipschitz with constant $\ell_\sigma$ as above. Let $C_P$ be the Poincar\'e constant in Theorem~\ref{thm:chi:square}, and set $g_0=\lambda/C_P$. Write $\kappa_t(x,y)=dP_t(x,\cdot)/d\pi(y)$ for the heat kernel relative to $\pi$. Proposition~\ref{prop:continuous:structure}~(iii) ensures that the following extrema are finite and strictly positive. For a fixed $s>0$, define
\begin{equation}\label{eq:const:heatkernel:constants}
m_s:=\min_{x,y\in\mathcal C}\kappa_s(x,y)>0,\quad
M_s:=\max_{x,y\in\mathcal C}\kappa_s(x,y)<\infty,\quad
H_s:=e^{\ell_\sigma^2s}\sqrt{\frac{M_s}{2e\,g_0s\,m_s}}.
\end{equation}
Then, for any probability measures $\mu,\nu$ on $\mathcal C$,
\begin{align}
\mathcal W_2(\mu P_t,\nu P_t)&\leq e^{\ell_\sigma^2t}\mathcal W_2(\mu,\nu), \qquad 0\leq t< 2s,\label{eq:const:short:stability}\\
\mathcal W_2(\mu P_t,\nu P_t)&\leq H_s e^{-g_0(t-2s)}\mathcal W_2(\mu,\nu), \qquad t\geq2s.\label{eq:const:long:stability}
\end{align}
\end{lemma}
\begin{proof}
The proof is given in Appendix~\ref{sec:const:semigroup:stability}.
\end{proof}

We now combine the local error in Proposition~\ref{prop:const:strong:euler} with the propagation bounds in Lemma~\ref{lem:const:semigroup:stability}. In the telescoping comparison, errors followed by evolution times $0$ and $s$ use the short time bound, and those followed by times $2s,3s,\ldots$ are bounded by a geometric series. This gives the factor $B_s$ below independently of the number of blocks. Adding the bound on continuous-time convergence to $\pi$ then gives a bound 
in the next theorem consisting of mixing and discretization errors.

\begin{theorem}[Time uniform discretization bound for constant reference]
\label{thm:U:0:const:strong}
Under the assumptions of Lemma~\ref{lem:const:semigroup:stability}, fix $s>0$ independently of the total running time and of the desired accuracy. Let
\begin{equation}\label{eq:const:Fs}
B_s:=1+e^{\ell_\sigma^2s}+\frac{H_s}{1-e^{-g_0s}},
\qquad F_s:=B_sI_s.
\end{equation}
For any initial law $\mu_0$ on $\mathcal C$, integers $m,n\geq1$, and $\eta=s/n$, $T=ms$,
\begin{equation}\label{eq:const:uniform:comparison}
\mathcal W_2\left(\Law(\bar X_T),\mu_0P_T\right)
\leq F_s\left[\eta \left(1+\log\frac{s}{\eta}\right)\right]^{1/4}.
\end{equation}
If also $\mu_0\ll\pi$ and $\chi^2(\mu_0\Vert\pi)<\infty$, denote $A_0:=2^{3/4}R(\chi^2(\mu_0\Vert\pi))^{1/4}$, then
\begin{equation}\label{eq:const-ref-unconditional-W2}
\mathcal W_2\left(\Law(\bar X_T),\pi\right)
\leq A_0e^{-g_0T/2}
+F_s\left[\eta\left(1+\log\frac{s}{\eta}\right)\right]^{1/4}.
\end{equation}
The constant $F_s$ is independent of $m$, $T$.%
\end{theorem}
\begin{proof}
The proof is given in Appendix~\ref{sec:thm:U:0:const:strong}.
\end{proof}

Theorem~\ref{thm:U:0:const:strong} separates the roles of the number of blocks $m$ and the number of steps per block $n$: increasing $m$ reduces the mixing error, while increasing $n$ reduces the discretization error. The next corollary chooses these two integers so that each term is at most $\ep/2$, yielding a total of $k=mn$ iterations.

\begin{corollary}[Iteration criterion for constant reference]
\label{cor:complexity:const}
Under the assumptions of Theorem~\ref{thm:U:0:const:strong}, including its initial $\chi^2$ condition, fix $s>0$. Given $\ep>0$, choose integers $m,n\geq1$ satisfying
$$
m\geq\frac{2}{g_0s}\log\left(1+\frac{2A_0}{\ep}\right),
\qquad
\frac{n}{(1 + \log n)}\geq\frac{16F_s^4s}{\ep^4}.
$$
Then $\eta=s/n$ and $k=mn$ give $\mathcal W_2(\Law(\bar X_{k\eta}),\pi)\leq\ep$. For fixed model parameters and fixed $s$, one can choose $k=\widetilde{\mathcal O}(\ep^{-4})$ as $\ep\rightarrow0$. This order concerns accuracy only, it does not assert polynomial dependence on dimension or all model parameters.
\end{corollary}
\begin{proof}
The proof is given in Appendix~\ref{sec:cor:complexity:const}.
\end{proof}

The preceding iteration criterion holds for any fixed $s>0$, but its constant $F_s$ depends on the heat kernel ratio $M_s/m_s$ through $H_s$ from Proposition~\ref{prop:const:strong:euler} and Lemma~\ref{lem:const:semigroup:stability}. The next corollary chooses $s$ from the spectral gap and a diagonal heat kernel value so that this ratio is at most $3$. A separate bound on $\ell_\sigma^2/g_0$ then controls the remaining exponential factors in $F_s$, without being required for the time uniform result itself.

\begin{corollary}[Choice of block length]\label{cor:const:conditioning}
Under the assumptions of Lemma~\ref{lem:const:semigroup:stability}, define
\begin{align*}
A_{\rm hk}&:=\sup_{x\in\mathcal C}\left\{\kappa_{2/g_0}(x,x)-1\right\}
=\sup_{x\in\mathcal C}\chi^2\left(P_{1/g_0}(x,\cdot)\Vert\pi\right)<\infty,\\
a_*&:=2+\log\max\left\{1,2A_{\rm hk}\right\},\qquad s_*:=a_*/g_0.
\end{align*}
Then $m_{s_*}\geq1/2$, $M_{s_*}\leq3/2$, and
$$
H_{s_*}\leq \exp\!\left(\frac{\ell_\sigma^2}{g_0}a_*\right)\sqrt{\frac{3}{2a_*}}.
$$
If, additionally, $\ell_\sigma^2/g_0\leq\ell_*$ for a specified constant $\ell_*\geq0$, then
\begin{equation}\label{eq:const:conditioning:Fs}
F_{s_*}\leq6\exp(6\ell_*)\max\left\{1,2A_{\rm hk}\right\}^{3\ell_*}\sqrt{G_{s_*}},
\end{equation}
where $G_s$, defined in \eqref{eq:const:Is}, has only dependence on $s$.
\end{corollary}
\begin{proof}
The proof is given in Appendix~\ref{sec:const:conditioning}.
\end{proof}
Only the last quantitative bound uses $\ell_\sigma^2/g_0\leq\ell_*$. The time uniform theorem and the bound $M_{s_*}/m_{s_*}\leq3$ require no such restriction. For a Lipschitz potential, a sufficient condition is $dC_P\Lambda\,\operatorname{Lip}(U)^2/(4\lambda)\leq\ell_*$. This permits nonsmooth, nonconvex potentials, but restricts coefficient variation relative to mixing. The remaining quantity $A_{\rm hk}$ is kept explicit rather than assumed dimension independent.

\subsection{\texorpdfstring{Gaussian smoothing reference potential}{Gaussian smoothing reference potential}}
We now give sufficient conditions on a Gaussian smoothed reference for the effective coefficients to satisfy the dissipative assumptions, and then control the error from approximating this reference by Monte Carlo averages. We adapt the Gaussian smoothing construction of Section~5.2 in~\cite{GNZZ2025} to the constrained setting. Let $U=f+g$ on $\mathbb{R}^{d}$ and, for $\mu>0$, define
\begin{equation}\label{def:Gaussian smoothing}
g_{\mu}(x):=\mathbb{E}[g(x+\mu\xi)],\qquad
U_{\mu}(x):=f(x)+g_{\mu}(x),\qquad
\xi\sim\mathcal{N}(0,I_d).
\end{equation}
The Gaussian convolution makes $g_{\mu}$ continuously differentiable and
\begin{equation}\label{eq:Gaussian:smoothing:gradient}
\nabla g_{\mu}(x)
=\frac{1}{\mu}\mathbb{E}\!\left[\xi g(x+\mu\xi)\right],
\end{equation}
where the identity follows from Gaussian integration by parts. In particular, $\mathbb E\xi=0$, and $g(x+\mu\xi)$ can be equivalently replaced by $g(x+\mu\xi)-g(x)$. If $U_0=U_\mu$ is used as the reference potential, then the coefficients of RALMC become
\begin{align}
&b(x)=-\nabla U_\mu(x)e^{U(x)-U_\mu(x)}
=-(\nabla f(x)+\nabla g_\mu(x))e^{g(x)-g_\mu(x)},
\\
&\sigma(x)=e^{(U(x)-U_\mu(x))/2}=e^{(g(x)-g_\mu(x))/2}.
\end{align}

\begin{assumption}\label{assumption:f:Gaussian:smoothing}
The function $f:\mathbb{R}^{d}\to\mathbb{R}$ is continuously differentiable. There exist $L_f,m_f>0$ such that, for all $x,y\in\mathcal C$,
\begin{align}
\|\nabla f(x)-\nabla f(y)\|&\leq L_f\|x-y\|,\qquad
\langle\nabla f(x)-\nabla f(y),x-y\rangle
\geq m_f\|x-y\|^2.
\end{align}
\end{assumption}

\begin{assumption}\label{assumption:g:Gaussian:smoothing}
The function $g:\mathbb{R}^{d}\to\mathbb{R}$ is Lipschitz with constant $K$ and weakly convex with parameter $\rho$: for all $x,y\in\mathbb{R}^{d}$ and $\lambda\in[0,1]$,
\begin{align}
|g(x)-g(y)|&\leq K\|x-y\|,\\
g(\lambda x+(1-\lambda)y)
&\leq \lambda g(x)+(1-\lambda)g(y)
+\frac{\rho}{2}\lambda(1-\lambda)\|x-y\|^2.
\label{assumption:g:convex:Gaussian:smoothing}
\end{align}
\end{assumption}

Under these assumptions, the next result (Lemma~\ref{lemma:propositions of Gaussian smoothing}) bounds the difference $U-U_\mu$, the gradient of $U_\mu$, and shows that smoothing preserves the weak convexity parameter of $g$. These estimates will be used in Lemma~\ref{lemma:Gaussian:sufficient:conditions} to check dissipativity and Lipschitz continuity of the effective coefficients $b$ and $\sigma$.

\begin{lemma}\label{lemma:propositions of Gaussian smoothing}
Suppose Assumptions~\ref{asmpt:set}, \ref{assumption:f:Gaussian:smoothing}, and \ref{assumption:g:Gaussian:smoothing} hold. Set $h_{\mu}:=U-U_{\mu}=g-g_{\mu}$. Then
\begin{gather}
\sup_{x\in\mathcal C}|h_{\mu}(x)|
\leq K\mu\sqrt d, \label{eq:Gaussian:hmu:sup}\\
\sup_{\substack{x,y\in\mathcal C\\x\neq y}}
\frac{|h_{\mu}(x)-h_{\mu}(y)|}{\|x-y\|}
\leq 2K, \label{eq:Gaussian:hmu:Lipschitz}\\
\sup_{x\in\mathcal C}\|\nabla U_{\mu}(x)\|
\leq \|\nabla f(0)\|+L_fR+K. \label{eq:Gaussian:Gmu:explicit}
\end{gather}
$g_{\mu}$ is weakly convex with parameter $\rho$ and has Lipschitz gradient with constant $K\sqrt d/\mu$. If $m_f>\rho$, then $U_{\mu}$ is strongly convex with constant $m_f-\rho$ on $\mathcal C$ and has
\begin{equation}
\operatorname{Lip}(\nabla U_\mu)\leq L_f+\frac{K\sqrt d}{\mu}.
\end{equation}
\end{lemma}

\begin{proof}
The proof is given in Appendix~\ref{sec:proof:Gaussian:preliminary:bounds}.
\end{proof}

Lemma~\ref{lemma:propositions of Gaussian smoothing} controls $U_\mu$, but the effective drift is $-e^{h_\mu}\nabla U_\mu$, not $-\nabla U_\mu$. We therefore need to control the change in the drift caused by $e^{h_\mu}$ and the variation of the diffusion coefficient $e^{h_\mu/2}$. The next lemma, motivated by Proposition~22 in~\cite{GNZZ2025}, gives a condition under which these changes preserve the dissipativity required by Theorem~\ref{global:error:2}. For fixed $\mu>0$, define their Lipschitz bounds by
\begin{align}
\varepsilon_{b,\mu}
&:=
\sup_{\substack{x,y\in\mathcal C\\x\neq y}}
\frac{
\left\|
\nabla U_\mu(x)(e^{h_\mu(x)}-1)
-\nabla U_\mu(y)(e^{h_\mu(y)}-1)
\right\|}
{\|x-y\|},
\\
\varepsilon_{\sigma,\mu}
&:=
\sup_{\substack{x,y\in\mathcal C\\x\neq y}}
\frac{
|e^{h_\mu(x)/2}-e^{h_\mu(y)/2}|
}
{\|x-y\|}.
\end{align}

\begin{lemma}[Gaussian smoothing coefficient condition]
\label{lemma:Gaussian:sufficient:conditions}
Suppose Assumptions~\ref{asmpt:set}, \ref{assumption:f:Gaussian:smoothing}, and \ref{assumption:g:Gaussian:smoothing} hold with $m_f>\rho$. If
\begin{equation}\label{eq:Gaussian:perturbation:small}
    \varepsilon_{b,\mu}
    +2d\varepsilon_{\sigma,\mu}^2
    <m_f-\rho-\frac12,
\end{equation}
then
\begin{equation}\label{eq:Gaussian:exact:coefficients}
    b(x):=-\nabla U_\mu(x)e^{h_\mu(x)},
    \qquad
    \sigma(x):=e^{h_\mu(x)/2}
\end{equation}
satisfy Assumptions~\ref{asmpt:strongly:convex} and \ref{asmpt:smooth} with
\begin{equation}\label{eq:Gaussian:m:M:alpha}
    m_\mu
    :=m_f-\rho-\varepsilon_{b,\mu},\qquad
    M_\mu
    :=L_f+\frac{K\sqrt d}{\mu}+\varepsilon_{b,\mu},\qquad
    \alpha_\mu
    :=d\varepsilon_{\sigma,\mu}^2.
\end{equation}
Moreover,
$$
    2m_\mu>1+4\alpha_\mu,
$$
and hence the coefficient assumptions of Theorem~\ref{global:error:2} hold.
\end{lemma}

\begin{proof}
The proof is given in Appendix~\ref{sec:proof:Gaussian:sufficient:conditions}.
\end{proof}

The condition in Lemma~\ref{lemma:Gaussian:sufficient:conditions} can be verified using explicit upper bounds, see Appendix~\ref{sec:Gaussian:explicit:criterion}. We denote
$$
    H_\mu:=K\mu\sqrt d,
    \qquad
    L_\mu:=L_f+\frac{K\sqrt d}{\mu},
    \qquad
    G:=\|\nabla f(0)\|+L_fR+K.
$$
The perturbation terms satisfy
$$
    \varepsilon_{\sigma,\mu}
    \leq Ke^{H_\mu/2},
    \qquad
    \varepsilon_{b,\mu}
    \leq L_\mu\left(e^{H_\mu}-1\right)+2KGe^{H_\mu}.
$$
Consequently, the computable condition
\begin{equation}\label{eq:Gaussian:perturbation:explicit}
    L_\mu\left(e^{H_\mu}-1\right)
    +2KGe^{H_\mu}
    +2dK^2e^{H_\mu}
    <m_f-\rho-\frac12
\end{equation}
is sufficient for \eqref{eq:Gaussian:perturbation:small}.

We note that the Lipschitz and weak convexity assumptions on $g$ do not by themselves imply that $\varepsilon_{b,\mu}$ or $\varepsilon_{\sigma,\mu}$ tends to zero. For example, in one dimension with $g(x)=|x|$ and $0$ in the interior of $\mathcal C$,
$$
    g_\mu'(x)=2\Phi(x/\mu)-1,
    \qquad
    h_\mu(0)=-\mu\sqrt{2/\pi},
    \qquad
    \lim _{t \rightarrow 0^+} \frac{h_\mu(t)-h_\mu(0)}{t}=1.
$$
Here $\Phi$ denotes the standard normal distribution function. Hence,
$$
    \varepsilon_{\sigma,\mu}
    \geq\frac12e^{h_\mu(0)/2}
    =\frac12e^{-\mu\sqrt{2/\pi}/2},
$$
which does not tend to zero. Therefore, we will use the finite coefficient margin in \eqref{eq:Gaussian:perturbation:small} in Theorem~\ref{theorem:Gaussian:MC:convergence}. For penalties such as $\ell_1$, SCAD \eqref{eqn:SCAD}, or MCP \eqref{eqn:MCP}, this margin must be verified for the chosen $f,g,\mu$, and domain. In Bayesian inference, this construction applies when $f$ is a smooth negative log likelihood and $g$ is a nonsmooth negative log prior. The reference potential $U_\mu=f+g_{\mu}$ preserves the smooth component and smooths only the non-differentiable component.

We next replace the two expectations in~\eqref{def:Gaussian smoothing} and~\eqref{eq:Gaussian:smoothing:gradient} by independent Monte Carlo averages. Fix an integer $N\geq1$. At each iteration $k$, let $\{\xi_{i,k}\}_{i=1}^{N}$ and $\{\hat\xi_{i,k}\}_{i=1}^{N}$ be two independent batches of i.i.d. $\mathcal N(0,I_d)$ vectors, independent across $k$ and independent of $x_0$. Define, for $x\in\mathcal C$,
\begin{align}
\tilde U_{\mu,k}(x)
&:=f(x)+\frac1N\sum_{i=1}^{N}g(x+\mu\xi_{i,k}),
\\
\widetilde{\nabla} U_{\mu,k}(x)
&:=\nabla f(x)+\frac{1}{\mu N}\sum_{i=1}^{N}
\hat\xi_{i,k}\left(g\left(x+\mu\hat\xi_{i,k}\right)-g(x)\right).
\end{align}
Let $\{\xi_k\}_{k\geq1}$ be i.i.d. $\mathcal N(0,I_d)$ vectors, independent of all Monte Carlo batches and of $x_0$. Starting from $\tilde x_0=x_0$, we synchronously couple the RALMC algorithms with ideal and approximated Gaussian smoothing reference function by
\begin{align}
x_{k+1}
&=\mathcal P_{\mathcal C}\!\left(
x_k+\eta b(x_k)+\sqrt{2\eta}\,\sigma(x_k)\xi_{k+1}\right),
\\
\tilde x_{k+1}
&=\mathcal P_{\mathcal C}\!\left(
\tilde x_k+\eta\tilde b_k(\tilde x_k)
+\sqrt{2\eta}\,\tilde\sigma_k(\tilde x_k)\xi_{k+1}\right),
\label{eq:Gaussian:MC:RALMC}
\end{align}
where \begin{equation}
\tilde b_k(x):=-\widetilde{\nabla} U_{\mu,k}(x)
e^{U(x)-\tilde U_{\mu,k}(x)},\qquad
\tilde\sigma_k(x):=e^{(U(x)-\tilde U_{\mu,k}(x))/2}.
\end{equation}
Denote $\nu_k:=\Law(x_k)$ and $\tilde\nu_k:=\Law(\tilde x_k)$. The following lemma bounds the Monte Carlo approximation error uniformly in $k$, 
whose proof uses synchronous coupling and the nonexpansiveness of $\mathcal P_{\mathcal C}$.

\begin{lemma}\label{lemma:Gaussian:MC:error}
Suppose Assumptions~\ref{asmpt:set}, \ref{assumption:f:Gaussian:smoothing}, and \ref{assumption:g:Gaussian:smoothing} hold, condition~\eqref{eq:Gaussian:perturbation:small} is satisfied, and $m_\mu,M_\mu,\alpha_\mu$ are defined by \eqref{eq:Gaussian:m:M:alpha}. Lemma~\ref{lemma:Gaussian:sufficient:conditions} then gives $2m_\mu>1+4\alpha_\mu$. Define
\begin{align}
&A_{\mu,N}
:=\frac{K^2d(d+2)}{N}e^{2K\mu\sqrt d+2K^2\mu^2/N}
+\left(\|\nabla f(0)\|+L_fR+K\right)^2e^{2K\mu\sqrt d}
\left(e^{2K^2\mu^2/N}-1\right),
\label{eq:Gaussian:A}\\
&B_{\mu,N}
:=e^{K\mu\sqrt d}
\left(e^{K^2\mu^2/(2N)}-1\right).
\label{eq:Gaussian:B}
\end{align}
Then, for every $k\in\mathbb N$ and
\begin{equation}\label{eq:Gaussian:eta:conservative}
0<\eta\leq
\min\left\{1,
\frac{2m_\mu-1-4\alpha_\mu}
{2\bigl(\max\{0,M_\mu^2-2m_\mu\}+M_\mu^2\bigr)}
\right\},
\end{equation}
we have
\begin{align}
\mathcal W_2^2(\nu_k,\tilde\nu_k)
&\leq
\frac{2}{2m_\mu-1-4\alpha_\mu}
\left((1+\eta)A_{\mu,N}+4dB_{\mu,N}\right).
\label{eq:Gaussian:MC:W2}
\end{align}
\end{lemma}
\begin{proof}
The proof is given in Appendix~\ref{sec:proof:Gaussian:MC:error}.
\end{proof}
Lemma~\ref{lemma:Gaussian:MC:error} bounds the distance between the Monte Carlo iterates $\tilde\nu_k$ and the exact coefficient iterates $\nu_k$, uniformly in $k$. To obtain an error bound relative to the target $\pi$, we combine it with Theorem~\ref{thm:final} through the triangle inequality
\begin{align}\label{ineq:two:terms}
\mathcal W_2(\tilde\nu_k,\pi)\leq\mathcal W_2(\tilde\nu_k,\nu_k)+\mathcal W_2(\nu_k,\pi).
\end{align}
The first term in \eqref{ineq:two:terms} is controlled by Lemma~\ref{lemma:Gaussian:MC:error}, and the second term in \eqref{ineq:two:terms} contains the mixing and discretization errors. This gives the three terms in \eqref{eq:Gaussian:final:W2} in the following Theorem~\ref{theorem:Gaussian:MC:convergence}. The coefficient assumptions above do not imply $U_\mu\in C^2(\overline{\mathcal C})$. The continuous time part of the next theorem therefore retains Assumption~\ref{asmpt:continuous:regularity} as a separate condition. For a smoothing level $\mu$ satisfying \eqref{eq:Gaussian:perturbation:small}, let $C_{0,\mu}$ denote the constant $C_0$ in Theorem~\ref{global:error:2} evaluated at the coefficients in \eqref{eq:Gaussian:exact:coefficients}, and set
\begin{align}
\alpha_{0,\mu}
:=\exp\!\left(\inf_{x\in\mathcal C}h_\mu(x)\right)
\geq e^{-K\mu\sqrt d}, \quad
\bar C_{0,\mu}
:=\left(1+\sqrt{\frac{M_\mu^2+4\alpha_\mu}
{2m_\mu-1-4\alpha_\mu}}\right)\sqrt{C_{0,\mu}}.
\label{eq:Gaussian:Cbar}
\end{align}
The lower bound for $\alpha_{0,\mu}$ follows from Lemma~\ref{lemma:propositions of Gaussian smoothing}.

\begin{theorem}\label{theorem:Gaussian:MC:convergence}
Suppose Assumptions~\ref{asmpt:set}, \ref{assumption:f:Gaussian:smoothing}, and \ref{assumption:g:Gaussian:smoothing} hold. Fix $\mu\in(0,1]$ satisfying \eqref{eq:Gaussian:perturbation:small}, and suppose Assumption~\ref{asmpt:continuous:regularity} holds with $U_0=U_\mu$. Suppose also that $\pi$ satisfies the Poincar\'{e} inequality used in Theorem~\ref{thm:chi:square}, that $\nu_0=\mu_0\ll\pi$ with $\chi^2(\nu_0\Vert\pi)<\infty$, and that the common initial law is supported in $\mathcal C_{r_0}:=\{x\in\mathcal C:\operatorname{dist}(x,\partial\mathcal C)\geq r_0\}$ for some $r_0>0$. Then, for every $N\geq1$, $k\in\mathbb N$ with $k\eta\geq e$, and every $\eta$ satisfying \eqref{eq:Gaussian:eta:conservative}, the law of RALMC iterates in~\eqref{eq:Gaussian:MC:RALMC} satisfies
\begin{align}
\mathcal W_2(\tilde\nu_k,\pi)
&\leq 2^{3/4}R\left(\chi^2(\nu_0\Vert\pi)\right)^{1/4}
e^{-\alpha_{0,\mu}k\eta/(2C_P)}
+\eta^{1/4}(k\eta)^{1/2}(\log k)^{1/2}\bar C_{0,\mu}\nonumber\\
&\qquad\qquad+\left(\frac{2}{2m_\mu-1-4\alpha_\mu}
\left((1+\eta)A_{\mu,N}+4dB_{\mu,N}\right)\right)^{1/2}.
\label{eq:Gaussian:final:W2}
\end{align}
Here $C_P$ is defined in Theorem~\ref{thm:chi:square}, $A_{\mu,N}$ and $B_{\mu,N}$ are defined in \eqref{eq:Gaussian:A}-\eqref{eq:Gaussian:B}. Consequently, for every fixed $\mu$ satisfying these assumptions and every $\ep>0$, there exist $T=k\eta$, $\eta>0$, and $N\geq1$ such that $\mathcal W_2(\tilde\nu_k,\pi)\leq\ep$.
\end{theorem}

\begin{proof}
The proof is given in Appendix~\ref{sec:proof:Gaussian:MC:convergence}.
\end{proof}

The three terms in \eqref{eq:Gaussian:final:W2} represent mixing, discretization, and Monte Carlo error. Fix any admissible $\mu>0$. Choose $T$ to control mixing, then $k$ with $\eta=T/k$ to control discretization, and finally $N$ to control Monte Carlo error. Each term can be made at most $\ep/3$.

\subsection{Comparison with PLMC}
\label{sec:plmc:comparison}
Throughout this section, suppose that $\mathcal C$ satisfies Assumption~\ref{asmpt:set}. We compare RALD and RALMC with projected Langevin dynamics after fixing a common time scale. Define the unnormalized scaling factor, its stationary mean, and the normalized scaling factor by
$$
    a_0(x):=e^{U(x)-U_0(x)},
    \qquad
    \bar a:=\int_{\mathcal C}a_0(x)d\pi(x),
    \qquad
    a_{\rm n}(x):=\frac{a_0(x)}{\bar a}.
$$
The normalization is introduced to remove an arbitrary global change of time before comparing RALD with PLD. In particular, Equation~\eqref{eqn:generator} gives
$$
    \mathcal Lf(x)
    =a_0(x)\left[\Delta f(x)-\langle\nabla U_0(x),\nabla f(x)\rangle\right].
$$
If $a_0$ is multiplied by a constant $c>0$, then the generator, semigroup, and spectral gap satisfy
$$
    \mathcal L^{(c)}=c\mathcal L,
    \qquad
    P_t^{(c)}=P_{ct},
    \qquad
    \lambda^{(c)}=c\lambda.
$$
This multiplication changes only the clock. For example, when $U$ is differentiable, take $U_0=U-\log c$. Then $a_0(x)=c$, $\nabla U_0=\nabla U$, and
$$
    \mathcal Lf
    =c\left[\Delta f-\langle\nabla U,\nabla f\rangle\right]
    =c\mathcal L_{\rm PLD}f.
$$
Thus the unnormalized process is exactly PLD run $c$ times faster. Since $c$ is arbitrary, an unnormalized comparison could produce an arbitrary apparent acceleration without changing the dynamics itself. To remove this ambiguity, define the normalized generator directly by
$$
    \mathcal L_{\rm n}f(x)
    :=\frac{1}{\bar a}\mathcal Lf(x)
    =a_{\rm n}(x)\left[
        \Delta f(x)-\langle\nabla U_0(x),\nabla f(x)\rangle
      \right].
$$
Then
$$
    \int_{\mathcal C}a_{\rm n}d\pi=1,
    \qquad
    \int_{\mathcal C}\operatorname{tr}\left[2a_{\rm n}(x)I_d\right]d\pi(x)
    =2d,
$$
so normalized RALD and PLD have the same stationary average of the instantaneous noise covariance. More generally, the shifts $U\mapsto U+c_1$ and $U_0\mapsto U_0+c_0$ give
$$
    a_0\mapsto e^{c_1-c_0}a_0,
    \qquad
    \bar a\mapsto e^{c_1-c_0}\bar a,
    \qquad
    a_{\rm n}\mapsto a_{\rm n}.
$$
Hence the normalized comparison does not depend on the arbitrary additive constants in either potential. At the same time,
$$
    \frac{a_{\rm n}(x)}{a_{\rm n}(y)}
    =\frac{a_0(x)}{a_0(y)},
$$
so normalization preserves the ratio of scaling factors between different regions. For the Euler-Maruyama scheme~\eqref{eqn:RALMC}, let $G\sim\mathcal N(0,I_d)$ denote the fresh Gaussian vector $\xi_{k+1}$ in a single update, independent of the current iterate. Using the same $G$ on both sides, the normalized RALMC proposal from $x$ with stepsize $\eta$ satisfies
\begin{equation}\label{eq:normalized:RALMC:proposal}
\begin{aligned}
    x-\eta a_{\rm n}(x)\nabla U_0(x)
    +\sqrt{2\eta a_{\rm n}(x)}\,G
    &=x-\frac{\eta}{\bar a}a_0(x)\nabla U_0(x)
    +\sqrt{2\frac{\eta}{\bar a}a_0(x)}\,G.
\end{aligned}
\end{equation}
Equation~\eqref{eq:normalized:RALMC:proposal} identifies this proposal with the unnormalized RALMC proposal in~\eqref{eqn:RALMC} with stepsize $\eta/\bar a$. Consequently, every advantage stated below on the normalized time scale comes from the spatial variation of the scaling factor, rather than from increasing variance by the same constant.

\paragraph{A normalized contraction stepsize advantage.}
We give a parameterized comparison for a quadratic potential with a convex perturbation of large curvature. The reference remains the Gaussian smoothing of that potential. Fix $\kappa,K,R,\mu>0$ and, for $\delta>0$, set
\begin{equation}\label{eq:GS:sharp:family}
\mathcal C:=[-R,R],\quad
U_\delta(x):=\frac{\kappa}{2}x^2+g_\delta(x),\quad
g_\delta(x):=K\sqrt{x^2+\delta^2},\quad
\pi_\delta(dx)\propto e^{-U_\delta(x)}\mathbf1_{\mathcal C}(x)dx.
\end{equation}
We denote
\begin{align}
& U_{\delta,\mu}(x):=\frac{\kappa}{2}x^2+\mathbb E[g_\delta(x+\mu G)],
\quad G\sim\mathcal N(0,1),\nonumber\\
& h_\delta :=U_\delta-U_{\delta,\mu},\quad
a_\delta:=e^{h_\delta},\quad
\bar a_\delta:=\int_{\mathcal{C}} a_\delta\,d\pi_\delta,\quad
a_\delta^{\rm n}:=\frac{a_\delta}{\bar a_\delta},\quad b_\delta^{\rm n}:=-a_\delta^{\rm n}U_{\delta,\mu}',\quad
\sigma_\delta^{\rm n}:=\sqrt{a_\delta^{\rm n}}.
\label{eq:GS:sharp:coefficients}
\end{align}
This normalization gives $\int_{\mathcal{C}} a_\delta^{\rm n}d\pi_\delta=1$, so the underlying diffusion and PLD have the same stationary mean noise covariance. Using the Euclidean projection $\mathcal P_{\mathcal C}$ onto $\mathcal C=[-R,R]$, define the ideal Gaussian reference RALMC and PLMC kernels by
\begin{align}
Q^{\rm R}_{\delta,\eta}(x,\cdot)
&:=\Law\!\left(\mathcal P_{\mathcal C}\left(x+\eta b_\delta^{\rm n}(x)
+\sqrt{2\eta}\sigma_\delta^{\rm n}(x)G\right)\right),\nonumber\\
Q^{\rm P}_{\delta,\eta}(x,\cdot)
&:=\Law\!\left(\mathcal P_{\mathcal C}\left(x-\eta U_\delta'(x)+\sqrt{2\eta}G\right)\right).
\label{eq:GS:sharp:kernels}
\end{align}

\begin{lemma}[Uniform Gaussian reference coefficient bounds]
\label{lem:GS:sharp:coefficients}
Suppose that
\begin{equation}\label{eq:GS:sharp:admissibility}
\kappa>2K(\kappa R+K)+e^{2K\mu}K^2.
\end{equation}
Then, for every $\delta>0$ and $x,y\in\mathcal C$,
\begin{align}
(b_\delta^{\rm n}(x)-b_\delta^{\rm n}(y))(x-y)
&\leq-e^{-K\mu}\bigl[\kappa-2K(\kappa R+K)\bigr]|x-y|^2,\nonumber\\
|b_\delta^{\rm n}(x)-b_\delta^{\rm n}(y)|&\leq e^{K\mu}\bigl[\kappa+K/\mu+2K(\kappa R+K)\bigr]|x-y|,\nonumber\\
|\sigma_\delta^{\rm n}(x)-\sigma_\delta^{\rm n}(y)|^2
&\leq e^{K\mu}K^2|x-y|^2.
\label{eq:GS:sharp:bounds}
\end{align}
All three bounds are independent of $\delta$. In contrast, $U_\delta''(0)=\kappa+K/\delta$. For any fixed $\kappa,R,\mu>0$, condition \eqref{eq:GS:sharp:admissibility} holds for all sufficiently small $K>0$, without restricting $\delta$.
\end{lemma}
\begin{proof}
The proof is given in Appendix~\ref{sec:proof:GS:sharp:coefficients}.
\end{proof}
Lemma~\ref{lem:GS:sharp:coefficients} shows that Gaussian smoothing keeps the effective coefficient bounds uniform as $\delta\to0$, even though the target curvature diverges. Proposition~\ref{prop:GS:sharp:contraction} converts these bounds into a common contraction stepsize interval by coupling two RALMC updates with the same Gaussian noise and using nonexpansiveness of projection. Under~\eqref{eq:GS:sharp:admissibility}, define the net dissipativity margin and drift Lipschitz bound by
\begin{equation}\label{eq:GS:sharp:constants}
\begin{aligned}
\gamma_*&:=e^{-K\mu}\bigl[\kappa-2K(\kappa R+K)\bigr]-e^{K\mu}K^2>0,\\
M_*&:=e^{K\mu}\bigl[\kappa+K/\mu+2K(\kappa R+K)\bigr].
\end{aligned}
\end{equation}
For a stepsize satisfying
\begin{equation}\label{eq:GS:sharp:eta}
0<\eta<\eta_c:=\frac{2\gamma_*}{M_*^2},
\end{equation}
set $q_\eta=[1-2\gamma_*\eta+M_*^2\eta^2]^{1/2}$. The following proposition establishes contraction for arbitrary initial laws.

\begin{proposition}[Global contraction with a uniform stepsize interval]
\label{prop:GS:sharp:contraction}
Assume~\eqref{eq:GS:sharp:admissibility} and~\eqref{eq:GS:sharp:eta}, with $q_\eta$ defined above. Let $\nu$ and $\nu'$ be any two initial probability distributions on $\mathcal C$. For a transition kernel $Q$, write $(\nu Q)(A):=\int_{\mathcal C}Q(x,A)\,\nu(dx)$ for every Borel set $A\subseteq\mathcal C$, this is the distribution after one update from an initial state with law $\nu$. Then, for every $\delta>0$,
\begin{equation}\label{eq:GS:sharp:contraction}
\mathcal W_2\left(\nu Q^{\rm R}_{\delta,\eta},\nu'Q^{\rm R}_{\delta,\eta}\right)
\leq q_\eta\mathcal W_2(\nu,\nu'),\qquad 0<q_\eta<1.
\end{equation}
The kernel has a unique invariant probability measure $\pi^{\rm R}_{\delta,\eta}$. Starting from $\nu_0:=\nu$, let $\nu_k:=\nu_0(Q^{\rm R}_{\delta,\eta})^k$ denote the law after $k$ RALMC updates, where the power denotes $k$ successive applications of the transition kernel. Taking $\nu'=\pi^{\rm R}_{\delta,\eta}$ in~\eqref{eq:GS:sharp:contraction} and iterating gives, for every integer $k\geq0$,
$$
\mathcal W_2\left(\nu_k,\pi^{\rm R}_{\delta,\eta}\right)
\leq q_\eta^k\mathcal W_2(\nu_0,\pi^{\rm R}_{\delta,\eta}).
$$
Both the admissible stepsize interval and the contraction factor are independent of $\delta$.
\end{proposition}

\begin{proof}
The proof is given in Appendix~\ref{sec:proof:GS:sharp:contraction}.
\end{proof}
Proposition~\ref{prop:GS:sharp:contraction} provides a contraction interval for RALMC that remains fixed as $\delta\to0$. To establish a stepsize advantage over PLMC, we next examine the projected PLMC transition law near the point $x=0$, where $U_\delta''(0)=\kappa+K/\delta$. The following theorem uses the sensitivity of its mean to bound the Wasserstein distance from below. For any fixed stepsize in the RALMC contraction interval, this bound exceeds the distance between the initial points when $\delta$ is sufficiently small, establishing local expansion of PLMC at the same normalized stepsize.

To specialize Proposition~\ref{prop:GS:sharp:contraction} to deterministic starting points, take $\nu=\delta_x$ and $\nu'=\delta_y$ in~\eqref{eq:GS:sharp:contraction}, where $\delta_x$ denotes the point mass at $x$. Since $\mathcal W_2(\delta_x,\delta_y)=|x-y|$ and $\delta_x Q^{\rm R}_{\delta,\eta}=Q^{\rm R}_{\delta,\eta}(x,\cdot)$, we obtain
\begin{equation}\label{eq:GS:sharp:RALMC:pointwise}
\frac{\mathcal W_2\bigl(Q^{\rm R}_{\delta,\eta}(x,\cdot),
Q^{\rm R}_{\delta,\eta}(y,\cdot)\bigr)}{|x-y|}
\leq q_\eta<1,
\qquad x,y\in\mathcal C,\quad x\neq y.
\end{equation}
Taking $y=0$ in~\eqref{eq:GS:sharp:RALMC:pointwise} gives the RALMC bound to be compared with PLMC in Theorem~\ref{thm:GS:sharp:separation}.

\begin{theorem}[A normalized contraction stepsize separation from PLMC]
\label{thm:GS:sharp:separation}
Assume \eqref{eq:GS:sharp:admissibility}, and let $\eta_c$ be defined by~\eqref{eq:GS:sharp:eta}. Given $\delta>0$, choose $\eta$ such that
\begin{equation}\label{eq:GS:sharp:separation:condition}
0<\eta<\eta_c,\qquad
\eta(\kappa+K/\delta)>1+p_\eta^{-1},
\qquad
p_\eta:=\mathbb P\left(\left|\sqrt{2\eta}G\right|<R\right).
\end{equation}
With this same $\eta$ in both kernels in~\eqref{eq:GS:sharp:kernels}, normalized RALMC satisfies the ratio bound~\eqref{eq:GS:sharp:RALMC:pointwise}, in particular with $y=0$, whereas the actual projected PLMC kernel satisfies
\begin{equation}\label{eq:GS:sharp:PL:expansion}
\liminf_{\substack{x\to0\\x\neq0}}
\frac{\mathcal W_2(Q^{\rm P}_{\delta,\eta}(x,\cdot),
Q^{\rm P}_{\delta,\eta}(0,\cdot))}{|x|}
\geq [\eta(\kappa+K/\delta)-1]p_\eta>1.
\end{equation}
The comparison range is nonempty for sufficiently small $\delta$: every fixed $\eta\in(0,\eta_c)$ satisfies \eqref{eq:GS:sharp:separation:condition} whenever
\begin{equation}\label{eq:GS:sharp:delta:range}
0<\delta<\frac{\eta K}{1+p_\eta^{-1}}.
\end{equation}
Moreover, fix any finite $\eta_{\max}>0$. If $0<\eta\leq\eta_{\max}$ and PLMC is globally nonexpansive in $\mathcal W_2$, then necessarily
\begin{equation}\label{eq:GS:sharp:PL:necessary}
\eta\leq\frac{1+p_\eta^{-1}}{\kappa+K/\delta}
\leq\frac{1+p_{\eta_{\max}}^{-1}}{\kappa+K/\delta}.
\end{equation}
Thus PLMC nonexpansiveness in any fixed bounded stepsize range requires $\eta=O(\delta)$ as $\delta\rightarrow0$, whereas RALMC admits a $\delta$-independent contraction interval under the same stationary noise normalization.
\end{theorem}
\begin{proof}
The proof is given in Appendix~\ref{sec:proof:GS:sharp:separation}.
\end{proof}
Theorem~\ref{thm:GS:sharp:separation} compares the two kernels in~\eqref{eq:GS:sharp:kernels} at the same numerical stepsize under the stationary noise normalization in~\eqref{eq:GS:sharp:coefficients}. Fix $\eta\in(0,\eta_c)$, with $\eta_c$ from~\eqref{eq:GS:sharp:eta}. For the same starting points $x$ and $0$, the RALMC ratio in~\eqref{eq:GS:sharp:RALMC:pointwise} with $y=0$ is at most $q_\eta<1$ for every $x\neq0$. For $\delta$ satisfying~\eqref{eq:GS:sharp:delta:range}, the corresponding PLMC ratio in~\eqref{eq:GS:sharp:PL:expansion} has lower limit strictly greater than one as $x\to0$. Both ratios divide the distance between the one step laws by the same initial distance $|x|$. Thus RALMC reduces this distance globally, while PLMC increases it for sufficiently small nonzero $x$, even after projection. The factor $p_\eta$ in~\eqref{eq:GS:sharp:separation:condition} accounts for projection: it is the probability that the proposal from zero lies inside $(-R,R)$.

Equation~\eqref{eq:GS:sharp:PL:necessary} gives the corresponding restriction on PLMC, within any fixed bounded stepsize range, global nonexpansiveness requires $\eta=O(\delta)$. This is a necessary condition, not a sufficient one. RALMC instead retains the fixed contraction interval~\eqref{eq:GS:sharp:eta} as $\delta\to0$. The advantage established here concerns one step Wasserstein contraction, local expansion does not imply failure of PLMC to converge over multiple steps. Moreover, contraction toward $\pi^{\rm R}_{\delta,\eta}$ in Proposition~\ref{prop:GS:sharp:contraction} does not by itself control the error relative to the target $\pi_\delta$. Discretization bias and, when the Gaussian expectations in~\eqref{eq:GS:sharp:coefficients} are approximated, Monte Carlo error require separate estimates.

\paragraph{Spectral gaps and local crossing probabilities.}
We now return to a general constraint set $\mathcal C\subset\mathbb R^d$ satisfying Assumption~\ref{asmpt:set} and a general target potential $U$. The one dimensional family and coefficient condition used in Theorem~\ref{thm:GS:sharp:separation} are not assumed below.

A bottleneck is a region of low conductance separating two parts of the state space. It may be caused by a narrow part of the constraint set or by a low density region inside a regular constraint set. For example, a two mode Gaussian mixture on a ball has a low density separator between the modes. The constant reference case makes the conductance explicit. If $U_0\equiv c_0$ and $\pi(x):=Z^{-1}e^{-U(x)}$ denotes the target density, then
\begin{align}
\label{eq:constant:mobility:density:cancellation}
    \bar a
    &=\int_{\mathcal C}e^{U(x)-c_0}\frac{e^{-U(x)}}{Z}dx
    =\frac{e^{-c_0}|\mathcal C|}{Z},\quad
    a_{\rm n}(x)
    =\frac{Ze^{U(x)}}{|\mathcal C|},\quad a_{\rm n}(x)\pi(x)
    =\frac{1}{|\mathcal C|}.
\end{align}
Consequently, the Dirichlet forms are
\begin{equation}\label{eq:constant:dirichlet:cancellation}
\begin{aligned}
    \mathcal E_{\rm RALD}^{\rm n}(f)
    &=\frac{1}{|\mathcal C|}
      \int_{\mathcal C}\|\nabla f(x)\|^2\,dx,\qquad
    \mathcal E_{\rm PLD}(f)
    =\frac{1}{Z}
      \int_{\mathcal C}e^{-U(x)}\|\nabla f(x)\|^2\,dx.
\end{aligned}
\end{equation}
Thus a low target density suppresses the PLD Dirichlet energy near a density bottleneck, however, $a_{\rm n}$ cancels this density factor for normalized constant reference RALD. Proposition~\ref{prop:dirichlet:comparison} gives the corresponding spectral gap comparison. Proposition~\ref{prop:bottleneck:minorization} then gives a separate finite step result, when $a_{\rm n}>1$ on a bottleneck layer, the probability ratio bound implies a larger one step crossing probability for sufficiently small stepsizes.

\begin{proposition}[Dirichlet form comparison]
\label{prop:dirichlet:comparison}
Suppose that $U$ is continuous on $\overline{\mathcal C}$ and that $a_{\rm n}$ is measurable with the positive finite bounds stated below. On the common form domain $H^1(\mathcal C)\subset L^2(\pi)$, define the Neumann Dirichlet forms of PLD and normalized RALD by
$$
    \mathcal E_{\rm PLD}(f):=\int_{\mathcal C}\|\nabla f\|^2d\pi,
    \qquad
    \mathcal E_{\rm RALD}^{\rm n}(f)
    :=\int_{\mathcal C}a_{\rm n}\|\nabla f\|^2d\pi.
$$
The bounded positive target density and scaling factor make both form norms equivalent to the usual $H^1(\mathcal C)$ norm. These are therefore closed forms with a common domain, their spectral gaps are defined by
$$
    \lambda_{\rm PLD}
    :=
    \inf_{f\in H^1(\mathcal C):\operatorname{Var}_\pi(f)>0}
    \frac{\mathcal E_{\rm PLD}(f)}{\operatorname{Var}_\pi(f)},
    \qquad
    \lambda_{\rm RALD}^{\rm n}
    :=
    \inf_{f\in H^1(\mathcal C):\operatorname{Var}_\pi(f)>0}
    \frac{\mathcal E_{\rm RALD}^{\rm n}(f)}{\operatorname{Var}_\pi(f)}.
$$
If $ a_{{\rm n},\min}:=\inf_{x\in\mathcal C}a_{\rm n}(x)>0,
    \,
    a_{{\rm n},\max}:=\sup_{x\in\mathcal C}a_{\rm n}(x)<\infty, $ then
\begin{equation}\label{eq:gap:comparison}
    a_{{\rm n},\min}\lambda_{\rm PLD}
    \leq
    \lambda_{\rm RALD}^{\rm n}
    \leq
    a_{{\rm n},\max}\lambda_{\rm PLD}.
\end{equation}
\end{proposition}

\begin{proof}
The proof is given in Appendix~\ref{proof:prop:dirichlet:comparison}.
\end{proof}
The normalization fixes a common time scale for RALD and PLD. Since $\int_{\mathcal C}a_{\rm n}\,d\pi=1$, one has $a_{{\rm n},\min}\leq1\leq a_{{\rm n},\max}$. Consequently, the comparison~\eqref{eq:gap:comparison} alone does not determine whether the normalized RALD spectral gap is larger than the PLD spectral gap. It provides baseline bounds, a strict improvement requires information about how the scaling factor is distributed relative to the gradients of the relevant test functions. For every $f\in H^1(\mathcal C)$ with $\operatorname{Var}_\pi(f)>0$,
\begin{align}\label{diff:ratio}
    \frac{\mathcal E_{\rm RALD}^{\rm n}(f)}{\operatorname{Var}_\pi(f)}
    -\frac{\mathcal E_{\rm PLD}(f)}{\operatorname{Var}_\pi(f)}
    =
    \frac{\displaystyle\int_{\mathcal C}(a_{\rm n}-1)\|\nabla f\|^2d\pi}
    {\operatorname{Var}_\pi(f)}.
\end{align}
For a two well target, a slow mode changes mainly across the low density bottleneck, so $a_{\rm n}>1$ there can make the integral in \eqref{diff:ratio} positive. This observation does not by itself prove $\lambda_{\rm RALD}^{\rm n}>\lambda_{\rm PLD}$ because the two spectral gaps take infima over all admissible $f$.

The preceding proposition concerns continuous time contraction rates. We now compare one Euler-Maruyama step near a density bottleneck. Let $B\subset\mathcal C$ be a low conductance layer and let $D_+\subset\operatorname{int}(\mathcal C)$ be a measurable region of positive Lebesgue measure on the opposite side of $B$. For the PLMC comparison, suppose that $U$ is differentiable on $B$. For $x\in B$ and $G\sim\mathcal N(0,I_d)$, define the Gaussian proposals before projection by
$$
    Z_{\rm R}(x)
    :=x-\eta a_{\rm n}(x)\nabla U_0(x)
       +\sqrt{2\eta a_{\rm n}(x)}\,G,
    \qquad
    Z_{\rm P}(x)
    :=x-\eta\nabla U(x)+\sqrt{2\eta}\,G.
$$
Here $Z_{\rm R}$ is the normalized RALMC proposal and $Z_{\rm P}$ is the PLMC proposal. Denote the corresponding projected one step kernels by $\mathbb P_\eta^{\rm RALMC}$ and $\mathbb P_\eta^{\rm PLMC}$. Since $D_+\subset\operatorname{int}(\mathcal C)$, projection does not change the event of landing in $D_+$:
$$
\begin{aligned}
    \mathbb P_\eta^{\rm RALMC}(x,D_+)
    =\mathbb P(Z_{\rm R}(x)\in D_+),\quad
    \mathbb P_\eta^{\rm PLMC}(x,D_+)
    =\mathbb P(Z_{\rm P}(x)\in D_+).
\end{aligned}
$$
In the constant reference case, Equation~\eqref{eq:constant:mobility:density:cancellation} gives $a_{\rm n}(x)=1/(|\mathcal C|\pi(x))$. Thus a point whose target density is below the uniform density has $a_{\rm n}(x)>1$. Conditionally on the current state $x$, the two proposal covariances therefore satisfy
$$
\begin{aligned}
    \operatorname{Cov}(Z_{\rm R}(x)\mid x)
    =2\eta a_{\rm n}(x)I_d,
    \qquad
    \operatorname{Cov}(Z_{\rm P}(x)\mid x)
    =2\eta I_d,
\\
\operatorname{Cov}(Z_{\rm R}(x)\mid x)
-\operatorname{Cov}(Z_{\rm P}(x)\mid x)
=2\eta(a_{\rm n}(x)-1)I_d\succ0.
\end{aligned}
$$
Thus RALMC has a wider local proposal at a low density bottleneck point. Figure~\ref{fig:bottleneck:geometry} shows the geometry of the comparison. The starting point $x$ is fixed in the bottleneck layer $B$, and both algorithms are compared on the same landing set $D_+$. The distance
$$
    r_x:=\operatorname{dist}(x,D_+)
$$
is the shortest displacement required to cross into that set. At the $1/\eta$ scale, the two Gaussian crossing costs are respectively $r_x^2/(4a_{\rm n}(x))$ and $r_x^2/4$, the drift affects only lower order terms. This is why $a_{\rm n}(x)>1$ produces the exponential advantage in the result below.

\begin{figure}[H]
\centering
\begin{tikzpicture}[
    x=1cm,
    y=1cm,
    line cap=round,
    line join=round,
    >={Latex[length=2.2mm]},
    every node/.style={font=\small},
    label box/.style={fill=white,inner sep=1.5pt},
    guide/.style={black!55,densely dotted}
]
\path[fill=black!2,draw=black,very thick] (0,0) ellipse (4.8 and 1.55);
\begin{scope}
    \clip (0,0) ellipse (4.8 and 1.55);
    \fill[black!14] (-2.75,0) ellipse (1.35 and 0.95);
    \fill[black!14] (2.75,0) ellipse (1.35 and 0.95);
    \fill[black!6] (-0.72,-1.6) rectangle (0.72,1.6);
\end{scope}
\draw[densely dashed] (-0.72,-1.36) to (-0.72,1.36);
\draw[densely dashed] (0.72,-1.36) to (0.72,1.36);

\coordinate (x0) at (-0.40,0);
\coordinate (ystar) at (1.40,0);
\coordinate (z0) at (2.65,0);
\fill (x0) circle (2.1pt);
\node[below left=3pt] at (x0) {$x$};
\fill (ystar) circle (1.5pt);
\node[below right=3pt] at (ystar) {$y_*$};
\fill (z0) circle (1.6pt);
\draw[->,thick] (x0) to[bend left=13]
    node[pos=0.47,above=8pt,label box]
    {$Z(x)\in D_+$} (z0);

\draw[guide] (x0) to (-0.40,-1.82);
\draw[guide] (ystar) to (1.40,-1.82);
\draw[<->] (-0.40,-1.82) to (1.40,-1.82);
\node[below=3pt,label box] at (0.50,-1.82)
    {$r_x=\operatorname{dist}(x,D_+)$};

\node at (-2.75,0.72) {$D_-$};
\node at (3.72,0.78) {$D_+$};
\node at (-4.00,-0.62) {$\mathcal C$};
\node[label box] at (0,1.10) {$B$};
\end{tikzpicture}
\caption{Conditional one step crossing at a density bottleneck. The layer $B$ separates the two high density regions $D_-$ and $D_+$. Starting from a fixed $x\in B$, both algorithms are evaluated on the same event of landing in $D_+$. The point $y_*$ represents a closest point of $\overline{D_+}$ to $x$ (it need not belong to $D_+$ when $D_+$ is open), so the distance is $r_x=\operatorname{dist}(x,D_+)$. Because $D_+\subset\operatorname{int}(\mathcal C)$, projection does not change this landing event.}
\label{fig:bottleneck:geometry}
\end{figure}
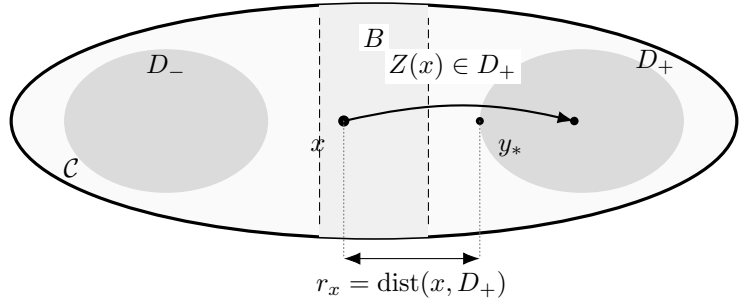

Proposition~\ref{prop:bottleneck:minorization} quantifies how the larger proposal variance competes with the change in drift at a fixed stepsize. It first compares the crossing probabilities from a given $x\in B$, then gives a bound valid throughout $B$. Both comparisons use the same starting point, landing set $D_+$, and normalized stepsize $\eta$.

\begin{proposition}[Direct one step crossing comparison]
\label{prop:bottleneck:minorization}
Fix $x\in B$ with $a_{\rm n}(x)>1$ and $r_x:=\operatorname{dist}(x,D_+)>0$. Then, for every $\eta>0$,
\begin{align}
\label{eq:bottleneck:pointwise:ratio}
& \frac{\mathbb P_\eta^{\rm RALMC}(x,D_+)}{\mathbb P_\eta^{\rm PLMC}(x,D_+)}
\nonumber
\\
& \, \geq a_{\rm n}(x)^{-d/2}\exp\Bigg\{
\frac{r_x^2}{4\eta}\left(1-\frac1{a_{\rm n}(x)}\right)-\frac{\operatorname{diam}(\mathcal C)}2\|\nabla U(x)-\nabla U_0(x)\|
-\frac\eta4 a_{\rm n}(x)\|\nabla U_0(x)\|^2\Bigg\}.
\end{align}
For a uniform comparison on $B$, define
$$
\rho:=\inf_{x\in B}a_{\rm n}(x),\qquad
A:=\sup_{x\in B}a_{\rm n}(x),\qquad
r:=\operatorname{dist}(B,D_+).
$$
Suppose $1<\rho\leq A<\infty$, $r>0$, gradients are bounded on $B$. Then, for every $x\in B$ and $\eta>0$,
\begin{align}
\label{eq:bottleneck:finite:ratio}
& \frac{\mathbb P_\eta^{\rm RALMC}(x,D_+)}{\mathbb P_\eta^{\rm PLMC}(x,D_+)}
\nonumber
\\
& \, \geq A^{-d/2}\exp\Bigg\{
\frac{r^2}{4\eta}\left(1-\frac1\rho\right)-\frac{\operatorname{diam}(\mathcal C)}2
\sup_{z\in B}\|\nabla U(z)-\nabla U_0(z)\|
-\frac\eta4\sup_{z\in B}a_{\rm n}(z)\|\nabla U_0(z)\|^2\Bigg\}.
\end{align}
\end{proposition}

\begin{proof}
The proof is given in Appendix~\ref{sec:prop:bottleneck:minorization}.
\end{proof}

In~\eqref{eq:bottleneck:pointwise:ratio}, the positive term $r_x^2(1-1/a_{\rm n}(x))/(4\eta)$ measures the gain from the larger proposal variance, while $a_{\rm n}(x)^{-d/2}$ and the drift terms reduce the lower bound. Thus $a_{\rm n}(x)>1$ guarantees a crossing advantage for sufficiently small $\eta$, but this bound need not certify an advantage at every stepsize. Under the uniform assumptions, \eqref{eq:bottleneck:finite:ratio} guarantees a larger RALMC crossing probability for every $x\in B$ whenever
$$
\frac{r^2}{4\eta}\left(1-\frac1\rho\right)
>\frac d2\log A
+\frac{\operatorname{diam}(\mathcal C)}2\sup_{z\in B}\|\nabla U(z)-\nabla U_0(z)\|
+\frac\eta4\sup_{z\in B}a_{\rm n}(z)\|\nabla U_0(z)\|^2.
$$
The proposition therefore supplies a sufficient condition at a finite stepsize. Corollary~\ref{cor:bottleneck:minorization} addresses the remaining question: is its leading $1/\eta$ exponent sharp? For an open landing set, the answer is yes. The corollary determines the exact logarithmic rates of both probabilities and their ratio, unlike the proposition, it does not require $a_{\rm n}(x)>1$ to state those rates.

\begin{corollary}[Exact exponential crossing advantage]
\label{cor:bottleneck:minorization}
Suppose $D_+$ is a nonempty open subset of $\operatorname{int}(\mathcal C)$. Fix $x\in B$ and assume $r_x=\operatorname{dist}(x,D_+)>0$. Then
\begin{equation}\label{eq:bottleneck:exact:rates}
    \lim_{\eta\rightarrow0}
    \eta\log\mathbb P_\eta^{\rm RALMC}(x,D_+)
    =-\frac{r_x^2}{4a_{\rm n}(x)},\quad
    \lim_{\eta\rightarrow0}
    \eta\log\mathbb P_\eta^{\rm PLMC}(x,D_+)
    =-\frac{r_x^2}{4}.
\end{equation}
Consequently,
\begin{equation}\label{eq:bottleneck:probability:ratio}
    \lim_{\eta\rightarrow0}
    \eta\log
    \frac{\mathbb P_\eta^{\rm RALMC}(x,D_+)}
         {\mathbb P_\eta^{\rm PLMC}(x,D_+)}
    =\frac{r_x^2}{4}
      \left(1-\frac{1}{a_{\rm n}(x)}\right).
\end{equation}
In particular, $a_{\rm n}(x)>1$ is both necessary and sufficient for a strictly positive exponential rate in this same event comparison. Under the uniform assumptions of Proposition~\ref{prop:bottleneck:minorization}, the right hand side of~\eqref{eq:bottleneck:probability:ratio} is at least
$$
    \frac{r^2}{4}\left(1-\frac1\rho\right)>0,
    \qquad x\in B.
$$
\end{corollary}

\begin{proof}
The proof is given in Appendix~\ref{sec:cor:bottleneck:minorization}.
\end{proof}

Equation~\eqref{eq:bottleneck:exact:rates} shows that both crossing probabilities tend to zero as $\eta\to0$, since the landing set remains a positive distance from $x$. When $a_{\rm n}(x)>1$, the RALMC probability decays more slowly, and~\eqref{eq:bottleneck:probability:ratio} gives the precise exponential growth rate of their ratio. When $a_{\rm n}(x)<1$, the ordering is reversed at this scale, when $a_{\rm n}(x)=1$, the exponential rate of the ratio is zero, so lower order terms determine the comparison. Thus the corollary establishes sharpness at the logarithmic scale, without identifying probability prefactors.

Together, Proposition~\ref{prop:bottleneck:minorization} and Corollary~\ref{cor:bottleneck:minorization} quantify a local mechanism: after the chain reaches a point of $B$ with $a_{\rm n}(x)>1$, normalized RALMC is more likely to land in the opposite region in one sufficiently small step. A large probability ratio does not imply a large absolute crossing probability. Nor do these conditional estimates account for the time needed to reach $B$ or for subsequent returns, comparisons of global mixing times or mean transition times require additional analysis.

\section{Numerical Experiments}

The preceding sections show that RALMC \eqref{eqn:RALMC} applies to constrained sampling with possibly non-differentiable densities and can improve over PLMC \cite{bubeck2018sampling} in the regimes described above. This section compares RALMC with PLMC numerically. Recall that PLMC consists of a Langevin step followed by projection:
\begin{align}\label{alg:PLMC}
    x_{k+1} = \mathcal{P}_{\mathcal{C}}\left(x_k-\eta\nabla U(x_k)+\sqrt{2\eta}\xi_{k+1}\right),
\end{align}
where $\nabla U(x_k)$ may fail to be continuous at some points in $\mathcal{C}$. In the examples below these points form a small exceptional set, and PLMC is used as a numerical benchmark.

We compare the convergence behavior of PLMC and RALMC for several choices of the reference potential $U_0$.

\subsection{Truncated Laplace distribution sampling}
We first sample from the joint distribution of three independent standard Laplace variables truncated by a ball in $\mathbb{R}^3$:
\begin{align}\label{exp:laplace}
    \pi(x)\propto e^{-(\vert x_1\vert+\vert x_2\vert+\vert x_3\vert)},\quad x=(x_1, x_2, x_3)\in\mathcal{C}:=\left\{z\in \mathbb{R}^3: \Vert z\Vert^2\leq 4\right\}.
\end{align}

Reference samples are generated by rejection sampling. We sample each coordinate independently from the one dimensional standard Laplace distribution and discard points outside the constraint set until the required sample size is reached. For RALMC, we choose the zero reference potential $U_0\equiv0$. Since
$$
    U(x)=\|x\|_1
    =|x_1|+|x_2|+|x_3|,
$$
the diffusion coefficient and scaling factor are
$$
    \sigma(x)=e^{\|x\|_1/2},
    \qquad
    a(x)=\sigma(x)^2=e^{\|x\|_1}.
$$
The function $\sigma$ is Lipschitz on $\mathcal C$ and is bounded above and away from zero there. Hence Assumption~\ref{assump:elliptic} and the coefficient regularity required in Theorem~\ref{thm:U:0:const:strong} hold, so Corollary~\ref{cor:complexity:const} applies.

For both RALMC and PLMC, we generated $n = 600$ samples from the same initial point $x_0\in\mathcal{C}$, where $x_0$ was drawn uniformly from $\mathcal{C}$. Each run used $1000$ epochs with stepsize $\eta = 2\times 10^{-3}$, and the experiment was repeated over $100$ runs to obtain one standard deviation intervals. Figure~\ref{fig:visualized_laplace} illustrates the first two coordinates of the target distribution in the left panel, and kernel density estimates from the RALMC and PLMC samples in the middle and right panels. Both algorithms match the target distribution visually.

\begin{figure}[!htbp]
\centering
 \begin{subfigure}{0.3\textwidth}
\centering \includegraphics[width=\linewidth]{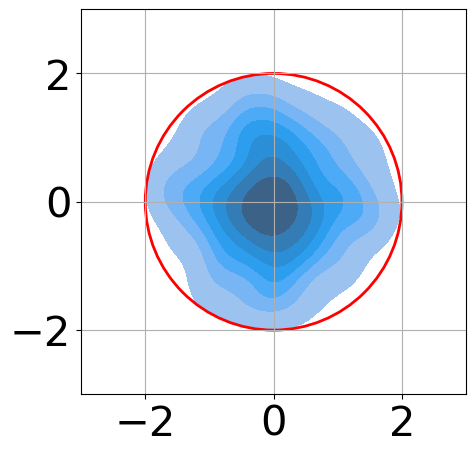} \caption{Target distribution}
 \end{subfigure}%
\hspace{0.5cm}
 \begin{subfigure}{0.3\textwidth}
\centering \includegraphics[width=\linewidth]{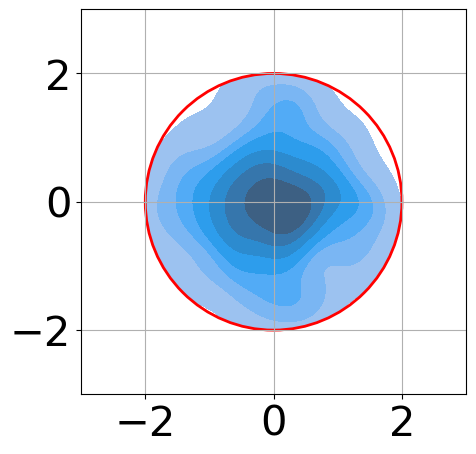} \caption{RALMC}
 \end{subfigure}%
\hspace{0.5cm}
 \begin{subfigure}{0.3\textwidth}
\centering \includegraphics[width=\linewidth]{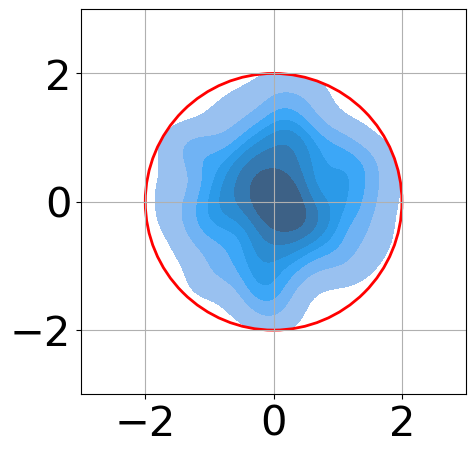} \caption{PLMC}
 \end{subfigure}
\caption{Visualized final sampled density plots for the first 2 dimensions in ball constraint}
 \label{fig:visualized_laplace}
\end{figure}
\FloatBarrier

Figure~\ref{fig:2wasserball_laplace} reports convergence in $\mathcal W_2$ distance. The solid line is the mean over $100$ independent runs, and the shaded region is the one standard deviation interval. The green curve corresponds to RALMC and the blue curve to PLMC. The initial point is uniform on the ball, so it is typically away from the high density region around the origin. For the unprojected RALMC noise increment at such a point,
$$
\begin{aligned}
    \mathbb E\left[
      \left\|\sqrt{2\eta a(x)}\,\xi_{k+1}\right\|^2
      \,\middle|\, x_k=x
    \right]
    &=2d\eta a(x)
      =2d\eta e^{\|x\|_1}.
\end{aligned}
$$
Thus particles in the low density outer part of the ball receive larger updates, while $a(x)\simeq1$ near the high density center. This adaptive redistribution is consistent with the steep initial decrease of the RALMC $\mathcal W_2$ curve. In this experiment the relevant bottleneck mechanism is movement through low density outer regions, rather than passage through a geometrically narrow constraint.

\begin{figure}[!htbp]
\centering \includegraphics[width=0.46\textwidth]{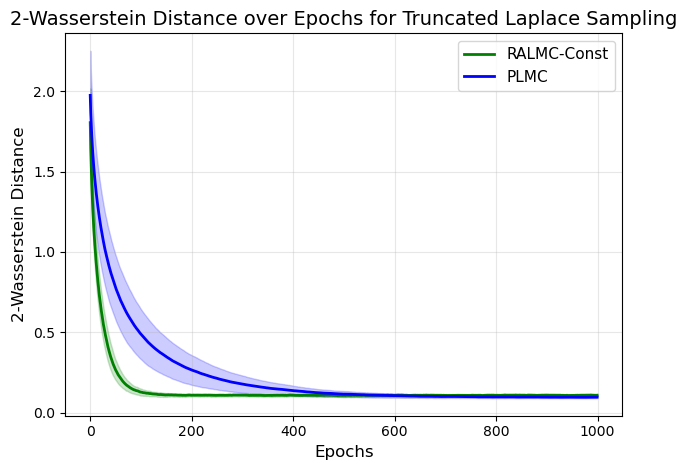} \caption{Convergence of RALMC and PLMC in $\mathcal W_2$ distance for a truncated Laplace distribution within a ball constraint}
 \label{fig:2wasserball_laplace}
\end{figure}
\FloatBarrier

\subsection{Truncated Gaussian Mixture Model}
\label{sec:numerical:GMM}

To further empirically illustrate the crossing advantage of RALMC over PLMC in Corollary~\ref{cor:bottleneck:minorization}, we use a truncated two component Gaussian mixture model as a synthetic multimodal target distribution. We follow a similar setup in \cite{DFTWZ2025} to sample from a truncated two component Gaussian mixture model on $\mathcal{C}:= \{x \in \mathbb{R}^3 : \|x\| \leq 1\}$. Consider two Gaussian distributions $\mathcal{N}(\mu_{1},\Sigma_{1})$ and $\mathcal{N}(\mu_{2},\Sigma_{2})$ with

$$
\mu_1=[\mu,0,0]^{\top},
\qquad
\mu_2=[-\mu,0,0]^{\top},
\qquad
\Sigma_1=\sigma_1^2 I_3,\qquad
\Sigma_2=\sigma_2^2 I_3,
$$
where $\mu,\sigma_1,\sigma_2>0$ and $I_d$ denotes the $d\times d$ identity matrix. The two components are symmetrically located along the first coordinate and have isotropic covariance matrices. This gives a controlled multimodal target. The unnormalized mixture density is

\begin{equation}
\label{eq:gaussian:mix}
\tilde p(x)
=
w_1\varphi(x;\mu_1,\Sigma_1)
+
w_2\varphi(x;\mu_2,\Sigma_2),
\end{equation}
where $w_1,w_2>0$ are mixture weights that satisfy $w_1+w_2=1$ and

$$
\varphi(x,\mu,\Sigma)
:=
\frac{1}{(2\pi)^{\frac{3}{2}}|\Sigma|^{1/2}}
\exp\left(
-\frac12 (x-\mu)^\top \Sigma^{-1}(x-\mu)
\right)
$$
denotes the Gaussian density on $\mathbb{R}^3$. To study sampling under constrained support, we further truncate the mixture to a constraint set $\mathcal{C}$. The truncated Gaussian mixture distribution is defined by

$$
p(x)
=\frac{1}{Z_\mathcal{C}}
\Bigl(
w_1\varphi(x,\mu_1,\Sigma_1)
+
w_2\varphi(x,\mu_2,\Sigma_2)
\Bigl)
\mathbf{1}_{\mathcal{C}}(x),
$$
where
$$
Z_\mathcal{C}
=
\int_{\mathcal{C}}
\Bigl(
w_1\varphi(x,\mu_1,\Sigma_1)
+
w_2\varphi(x,\mu_2,\Sigma_2)
\Bigr)dx
$$
is the normalization constant. The potential function defined by

$$
U(x):=-\log p(x),
\qquad x\in \mathcal{C},
$$
has two separated low energy regions associated with the two Gaussian components, located around $\mu_1$ and $\mu_2$, with a low density barrier between them. This target is therefore suitable for evaluating whether a sampling algorithm can efficiently move between separated modes rather than remaining trapped in a single local region.
\begin{figure}[!htbp]
\centering \includegraphics[width=0.56\textwidth]{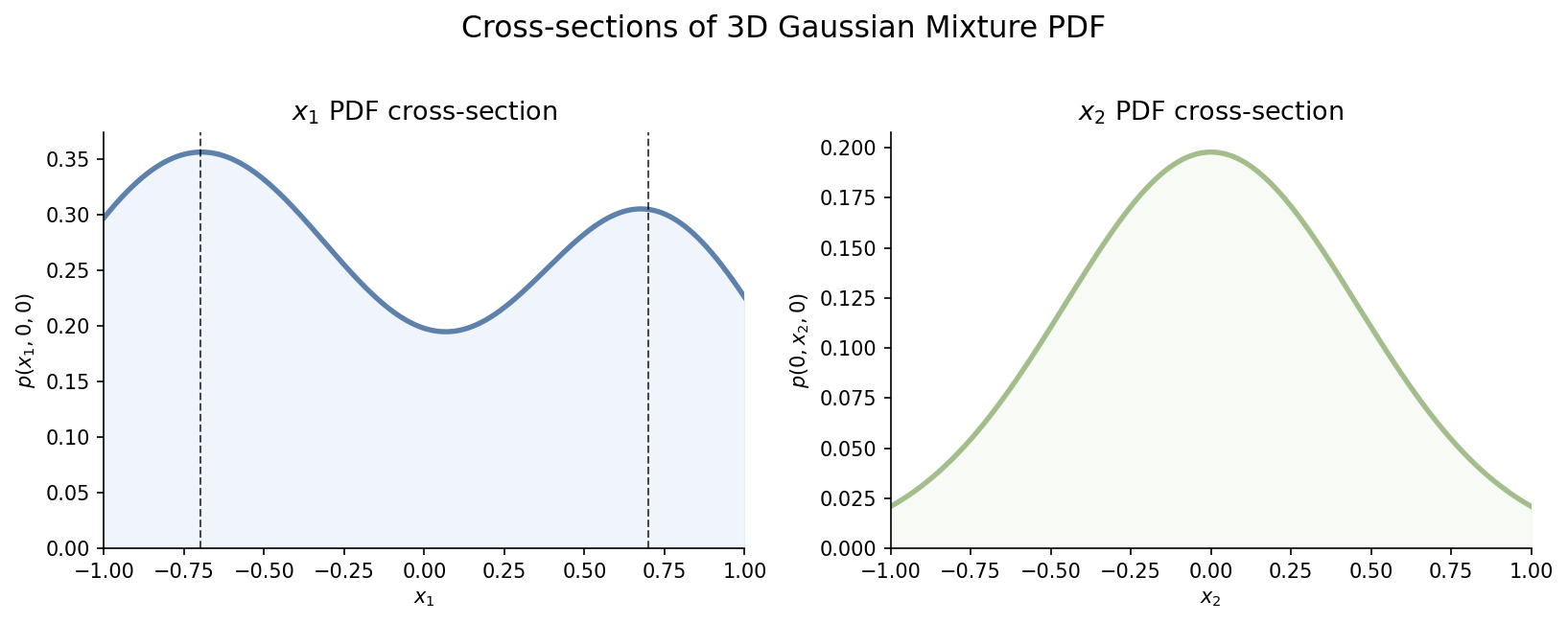} \caption{Cross sections of the target truncated Gaussian mixture model along the first two dimensions}
 \label{fig:GSM_crosssection}
\end{figure}
\FloatBarrier

In this experiment, we set $\mu = 0.7$, $\sigma_1 = 0.4$, $\sigma_2 = 0.5$, and $w_1=0.3$. Figure~\ref{fig:GSM_crosssection} shows the double well shape of the truncated Gaussian mixture through cross sections along the first two coordinates. We compare PLMC with RALMC using the constant reference function $U_0\equiv0$. Inside $\mathcal C$,
$$
    a(x)=e^{U(x)}=\frac{1}{p(x)}.
$$
Let $x_{\rm mode}$ lie near a mode and let $x_{\rm val}$ lie in the valley between the modes. Then
\begin{equation}\label{eq:GMM:mobility:ratio}
    \frac{a(x_{\rm val})}{a(x_{\rm mode})}
    =\frac{p(x_{\rm mode})}{p(x_{\rm val})}>1,
    \qquad
    \frac{\sigma(x_{\rm val})}{\sigma(x_{\rm mode})}
    =\sqrt{\frac{p(x_{\rm mode})}{p(x_{\rm val})}}>1.
\end{equation}
The first cross section in Figure~\ref{fig:GSM_crosssection} has modal heights of approximately $0.30$-$0.35$ and a central valley height of approximately $0.20$. The corresponding scaling factor ratio is therefore about $1.5$-$1.8$, and the proposal standard deviation in the valley is about $1.2$-$1.3$ times its value near a mode. This larger scaling factor in the valley reduces the time spent in the low density separator once it is reached and promotes movement between the two modal regions. The identity \eqref{eq:constant:mobility:density:cancellation} gives the associated continuous time conductance interpretation.

To match the local comparison above, we define the low probability valley $B=\{(x_1,x_2,x_3)\in\mathcal{C}:-0.25\leq x_1\leq0.25\}$ and the region around the left mode $D^+:=\{(x_1,x_2,x_3)\in\mathcal{C}:x_1<-0.5\}$. Specifically, RALMC with a constant reference is
$$
    x_{k+1}^{\rm RALMC}
    =\mathcal{P}_{\mathcal C}\left(
        x_k^{\rm RALMC}
        +\sqrt{2\eta a_{\rm n}\left(x_k^{\rm RALMC}\right)}\,\xi_{k+1}
      \right),
    \qquad \xi_{k+1}\sim\mathcal N(0,I_d),
$$
where
$$
    a_{\rm n}\left(x_k^{\rm RALMC}\right)
    =\frac{3e^{U(x_k^{\rm RALMC})}}{4\pi},
$$
because $\mathcal C$ is the three dimensional unit ball and $p=e^{-U}$ is normalized on $\mathcal C$. 

For each method, this numerical experiment starts with $600$ samples from the uniform distribution in $B$ and iterates for a maximal $K=1000$ epochs in 100 independent runs. We quantify the crossing advantage of RALMC with a constant reference function through the first entering iteration $\tau_{D^+}$ and the transition numbers $N_{D^+}$:
\begin{align}
    \bar\tau_{D^+}&:=\frac{1}{n}\sum_{i=1}^n\tau_{D^+}^{i},\quad\text{where}\quad \tau_{D^+}^{i}:=\min\left\{k\in\mathbb{N}:x^{(i)}_k \in D^+\right\}.\\ N_{D^+}&:=\sum_{i=1}^{n}
\mathbf{1}_{\left\{\tau_{D^+}^{(i)}\leq K\right\}},
\end{align}
where $x^{(i)}_k$ can be interpret as the $i$-th sample in the $k$-th iteration of RALMC or the PLMC.

\begin{figure}[!htbp]
\centering
 \begin{subfigure}{0.3\textwidth}
\centering \includegraphics[width=\linewidth]{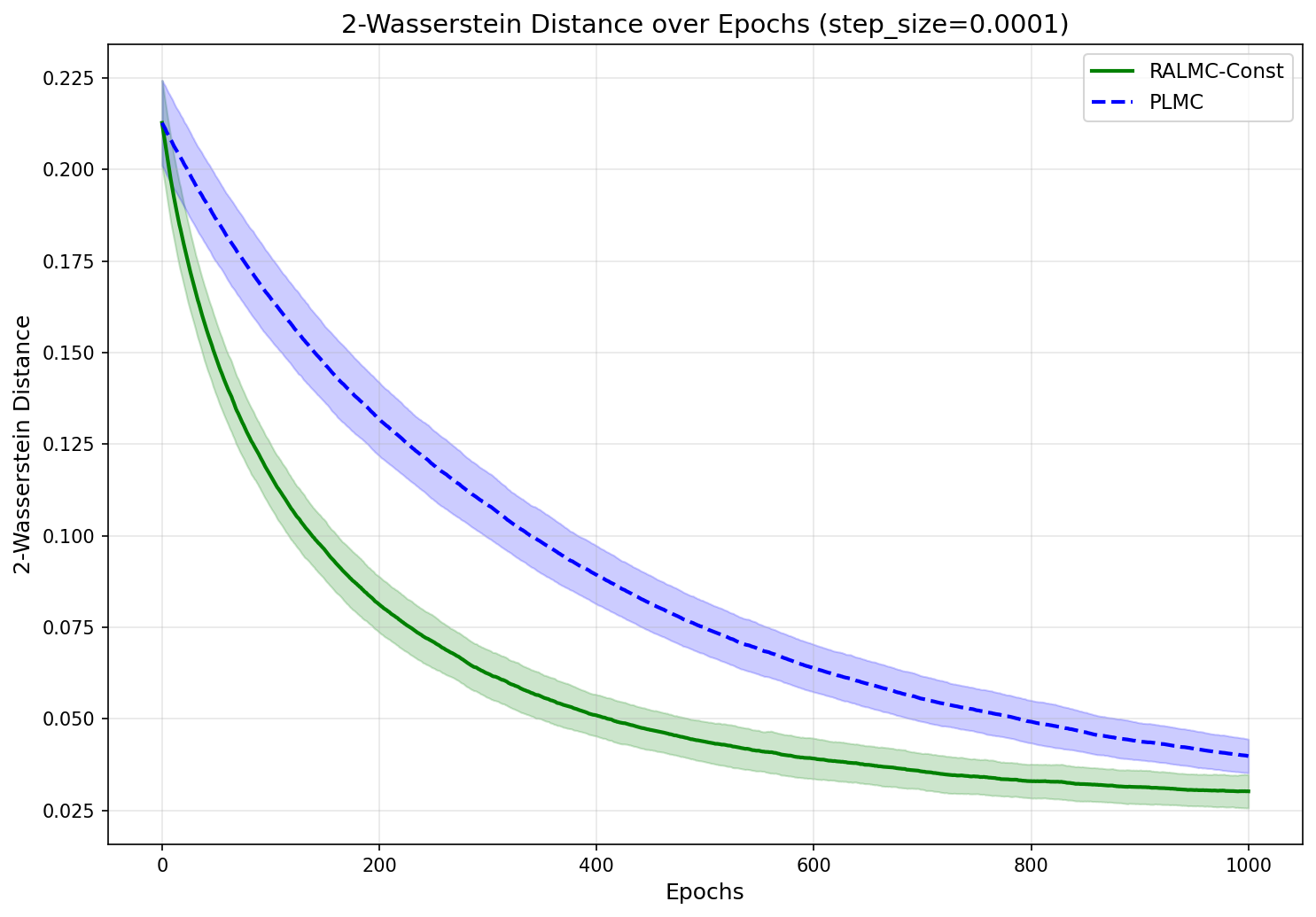} \caption{$\eta=10^{-4}$}
 \end{subfigure}%
\hspace{0.5cm}
 \begin{subfigure}{0.3\textwidth}
\centering \includegraphics[width=\linewidth]{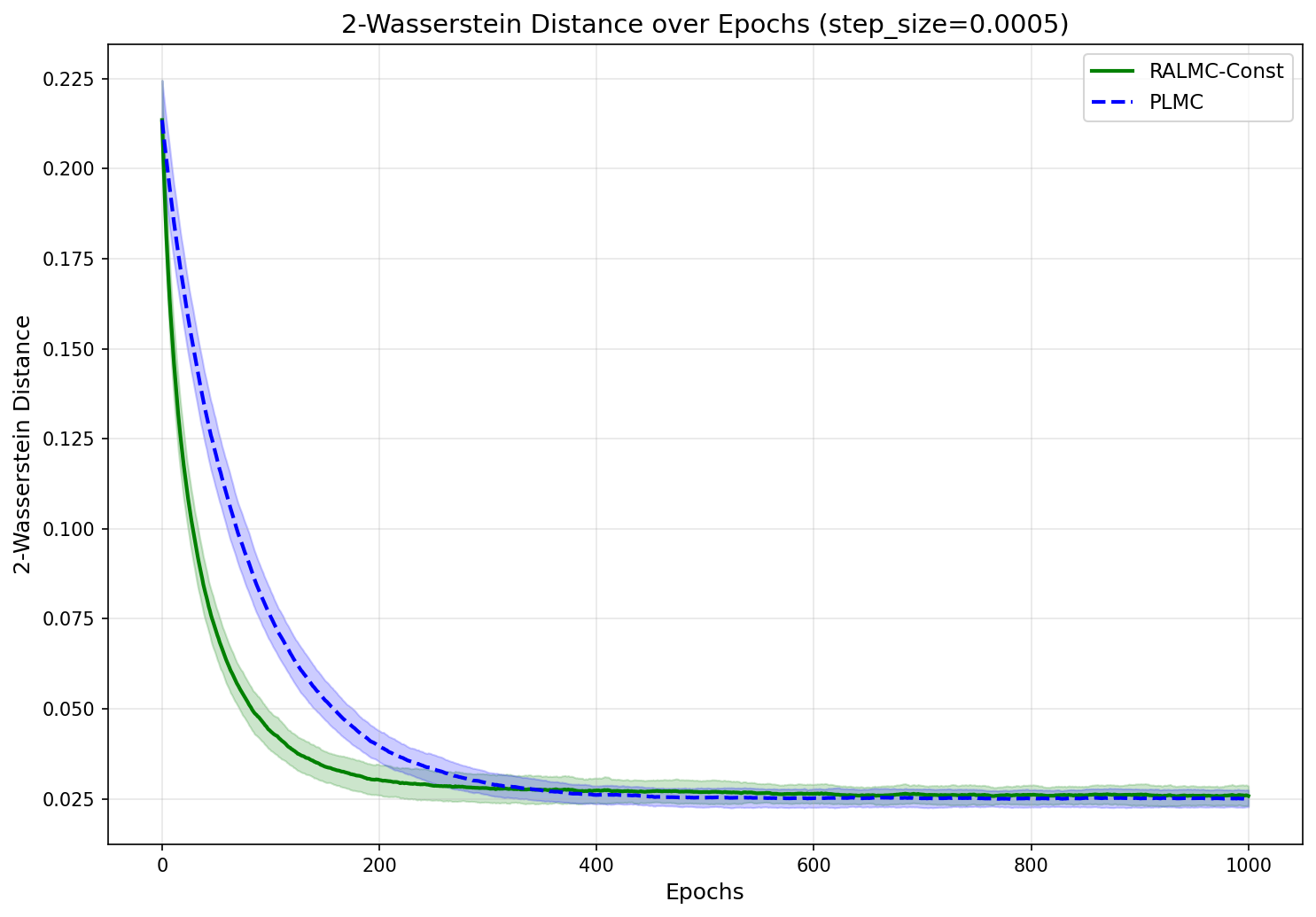} \caption{$\eta=5\times10^{-4}$}
 \end{subfigure}%
\hspace{0.5cm}
 \begin{subfigure}{0.3\textwidth}
\centering \includegraphics[width=\linewidth]{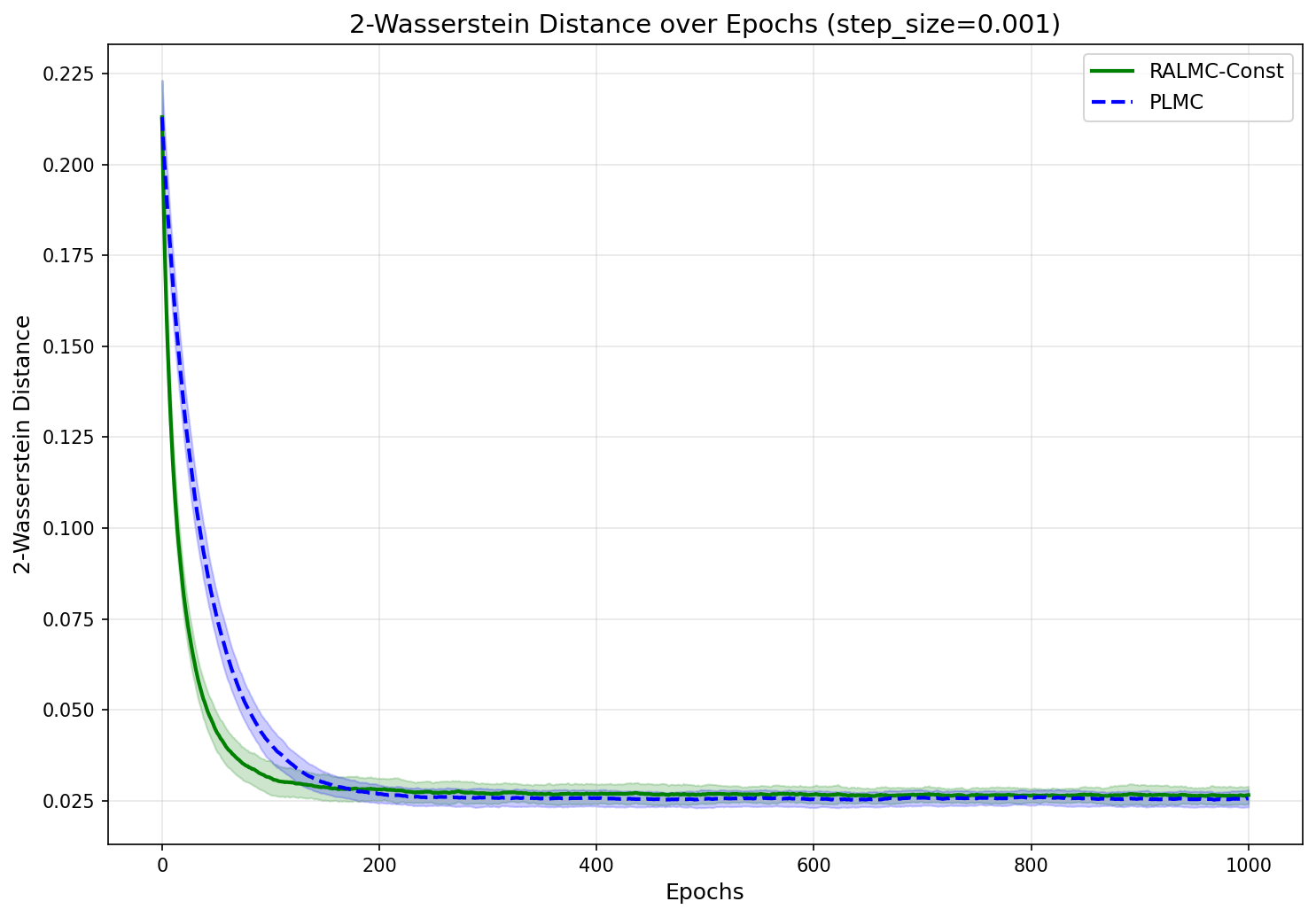} \caption{$\eta=10^{-3}$}
 \end{subfigure}
\caption{Performance of RALMC with a constant potential and PLMC, measured by the $\mathcal W_2$ distance to the target Gaussian mixture model, for different stepsizes $\eta$ over $1000$ epochs.}
 \label{fig:GSM_dist}
\end{figure}
\FloatBarrier

Figure~\ref{fig:GSM_dist} reports the 2-Wasserstein distance between the algorithms and the target distribution across three stepsizes $\eta\in\{10^{-4},5\times10^{-4},10^{-3}\}$. We can observe that RALMC with a constant reference converges faster than PLMC, which is consistent with the scaling factor ratio in \eqref{eq:GMM:mobility:ratio}. An additive shift of $U$ rescales every RALMC proposal by the same constant, so a comparison at the same nominal stepsize also contains a global time rescaling. In contrast, the ratios in \eqref{eq:GMM:mobility:ratio} are invariant under such a shift and isolate the state dependent acceleration across the density valley.

\begin{table}[htbp]
\centering
\begin{tabular}{cccc}
\toprule
Stepsize & Method & \# of Entered $D^+$ & Enter epoch \\
\midrule
$1\times 10^{-4}$ & RALMC-Const & $293.00 \pm 12.36$ & $712.15 \pm 14.13$ \\
 & PLMC & $204.79 \pm 11.06$ & $834.47 \pm 10.29$ \\
\addlinespace
$5\times 10^{-4}$ & RALMC-Const & $519.67 \pm 8.70$ & $361.23 \pm 14.62$ \\
 & PLMC & $427.95 \pm 11.95$ & $503.66 \pm 15.44$ \\
\addlinespace
$1\times 10^{-3}$ & RALMC-Const & $581.70 \pm 4.21$ & $219.32 \pm 10.62$ \\
 & PLMC & $510.79 \pm 8.77$ & $362.68 \pm 15.49$ \\
\bottomrule
\end{tabular}
\caption{First enter performance from $B=\{x\in\mathcal{C}:-0.25\leq x_1\leq0.25\}$ to $D^+=\{x\in\mathcal{C}:x_1<-0.5\}$. Results are mean $\pm$ standard deviation over independent runs. Paths that do not enter $D^+$ are censored at the terminal epoch.}\label{tab:first-enter-statistics}
\end{table}

Table~\ref{tab:first-enter-statistics} reports the enter performance across three stepsizes $\eta\in\{10^{-4},5\times10^{-4},10^{-3}\}$ for each method. Among all the three choices, we can observe that $\bar\tau_{D^+}^{\rm RALMC}<\bar\tau_{D^+}^{\rm PLMC}$ and $N_{D^+}^{\rm RALMC}>N_{D^+}^{\rm PLMC}.$ It demonstrates well that the RALMC with a constant reference can travel across the bottleneck at a higher probability and in a shorter time.

\subsection{Constrained Bayesian logistic regression on real data}

We next evaluate the algorithms on binary classification problems through constrained Bayesian logistic regression. We choose the constraint domain $\mathcal{C}:=\{z\in \mathbb{R}^d: \Vert z\Vert^2\leq 4\}$, so the constrained model corresponds to Bayesian ridge logistic regression.

Suppose we have access to a dataset $Z=\{z_j\}_{j=1}^{n}$ where $z_j=\left(X_j, y_j\right), X_j \in \mathbb{R}^d$ are the features and $y_j \in\{0,1\}$ are the labels with the assumption that $X_j$ are independent and the probability distribution of $y_j$ given $X_j$ and the regression coefficients $\beta \in \mathbb{R}^d$ are given by
\begin{equation}
\mathbb P\left(y_j=1 \mid X_j, \beta\right)=\frac{1}{1+e^{-\beta^{\top} X_j}}.
\end{equation}
Then by the Baye's rule, the goal of the constrained Bayesian logistic regression becomes to sample from the posterior distribution $\pi(\beta) \propto e^{-U(\beta)}\mathbf{1}_\mathcal{C}$ with:
\begin{align}
\label{eq:potential of BLR}
U(\beta)&:=-\left(\sum_{j=1}^{n} y_j\log p\left(y_j \mid X_j, \beta\right)+(1-y_j)\log (1-p\left(y_j \mid X_j, \beta)\right)\right)-\log p(\beta)\nonumber\\
&=\left(\sum_{j=1}^{n} y_j\log \left(1+e^{-\beta^{\top} X_j}\right)+(1-y_j)\log \left(1+e^{\beta^{\top} X_j}\right)\right)-\log p(\beta),
\end{align}
where $\log p(\beta)$ is denoted as the prior density function, omitting the normalization. In our experiments, we choose three different non-differentiable $p(\beta)$ as follows:
\begin{enumerate}
 \item Laplace prior: $\log p(\beta)=-\lambda\sum_{i=1}^d\vert \beta_i\vert$.
 \item Smoothly clipped absolute deviation (SCAD) prior: $\log p(\beta)=-\sum_{i=1}^d g_{\lambda,a}(\beta_i)$, where $a>1$ and
\begin{align}\label{eqn:SCAD}
g_{\lambda,a}(x) =
\begin{cases}
\lambda|x| & \text{if } |x| \leq \lambda, \\
\displaystyle \frac{2a\lambda|x|-x^2-\lambda^2}{2(a-1)} & \text{if } \lambda < |x| \leq a\lambda, \\
\displaystyle \frac{\lambda^2(a+1)}{2} & \text{otherwise.}
\end{cases}
\end{align}
 \item Minimax concave penalty (MCP) prior: $\log p(\beta)=-\sum_{i=1}^d h_{\lambda,a}(\beta_i)$, where $a>1$ and
\begin{align}\label{eqn:MCP}
h_{\lambda,a}(x) =
\begin{cases}
\lambda|x|-\frac{x^2}{2a} & \text{if } |x| \leq a\lambda, \\
\displaystyle \frac{a\lambda^2}{2} & \text{otherwise.}
\end{cases}
\end{align}
\end{enumerate}
Here, the SCAD and MCP priors are adapted from the corresponding regularizers in the literature \cite{Zhang10,Fan01}.

From~\eqref{eq:potential of BLR}, the potential decomposes into the differentiable log likelihood term
$$
f(\beta) :=
\sum_{j=1}^{n}
\log \left(1+e^{-\beta^{\top} X_j}\right),
$$
and the non-differentiable log prior term $\log p(\beta)$. We therefore choose the reference only for the non-differentiable part and keep the differentiable part unchanged. The experiments use the constant reference $U_0(\beta)\equiv0$ and the Gaussian smoothing reference
$$
\tilde{U}_\mu(\beta)
=
\frac{1}{N_0}\sum_{i=1}^{N_0}\log p(\beta+\mu \xi_i),
$$
where the $\xi_i$ are i.i.d. $\mathcal{N}(0,I_d)$, $N_0=300$, and $\mu = 0.5$.

We conducted experiments using RALMC with Gaussian smoothing (RALMC-GS), RALMC with constant reference potential (RALMC-Const), and PLMC for the constrained Bayesian logistic regression with three different prior distributions on the MAGIC Gamma Telescope dataset \footnote{The MAGIC Gamma Telescope dataset is publicly available from: \url{https://archive.ics.uci.edu/ml/datasets/magic+gamma+telescope} }and the Breast Cancer Wisconsin (Diagnostic) dataset \footnote{The Breast Cancer Wisconsin (Diagnostic) dataset is publicly available from: \url{https://archive.ics.uci.edu/dataset/17/breast+cancer+wisconsin+diagnostic}.}. The MAGIC Gamma Telescope dataset is a binary classification benchmark used to distinguish high energy gamma rays from background hadron showers. It comprises $n=19,020$ complete simulated instances, consisting of $12,332$ gamma and $6,688$ hadron samples. Each instance is defined by $d=10$ continuous numeric features. The Breast Cancer Wisconsin (Diagnostic) dataset is a standard benchmark for binary classification, used to distinguish between malignant and benign breast masses. It contains $n=569$ complete samples ($357$ benign, $212$ malignant), each defined by $d=30$ continuous features. For each dataset, we split $80\%$ of the whole dataset as the training set and the rest as the test set. To eliminate the influence that features are measured in entirely different units, we apply standard normalization to every feature for each dataset before training. Accuracy is the standard metric for evaluating binary classification performance. While a formal theoretical analysis of accuracy convergence is beyond the scope of this work, it still serves as a critical empirical indicator of algorithmic efficacy. To this end, we evaluate and compare the test set accuracy convergence for the three algorithms in the following experiments.

\begin{figure}[!htbp]
\centering
 \begin{subfigure}{0.28\textwidth}
\centering \includegraphics[width=\linewidth]{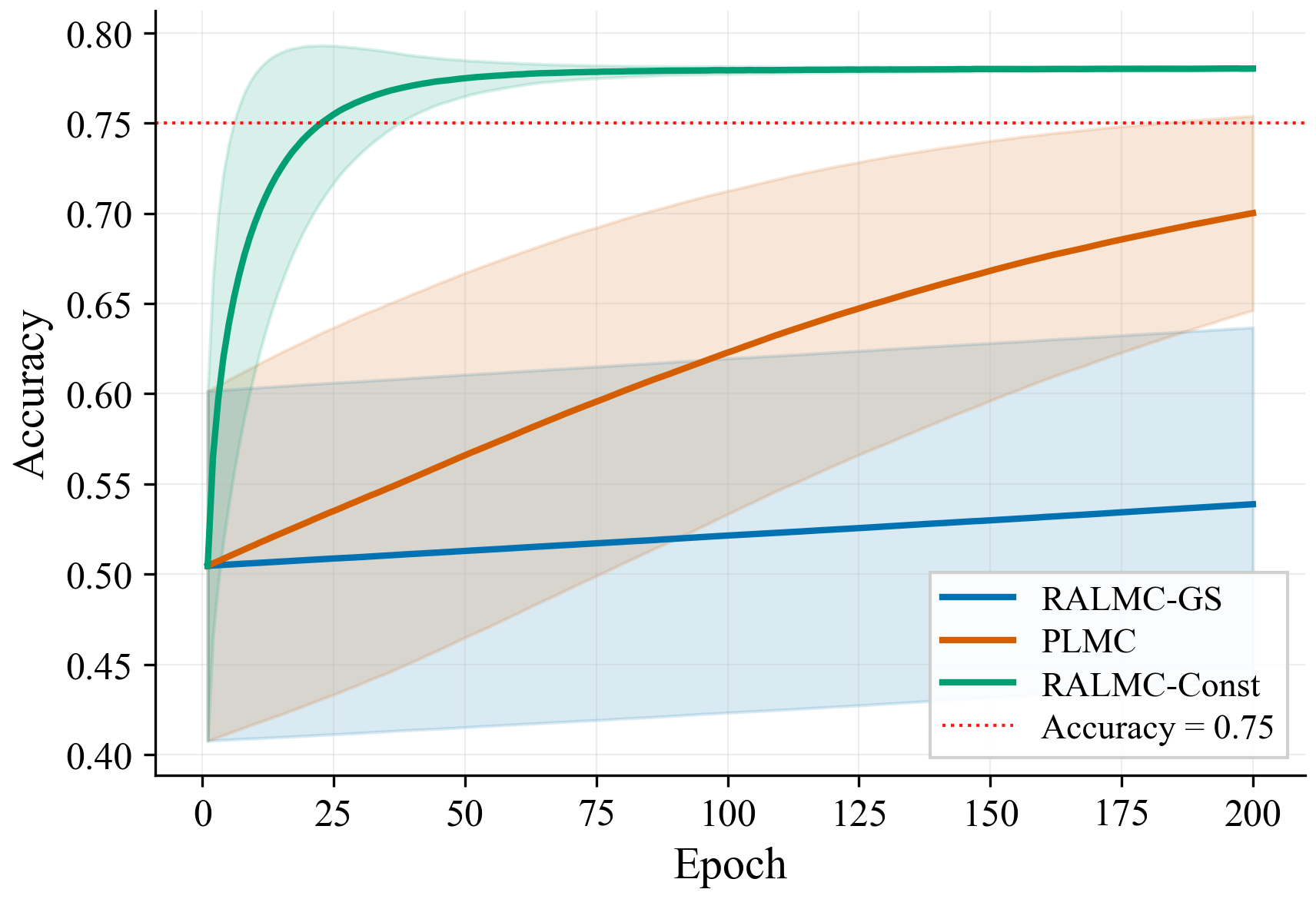} \caption{Laplace prior}
 \end{subfigure}%
\hspace{0.5cm}
 \begin{subfigure}{0.28\textwidth}
\centering \includegraphics[width=\linewidth]{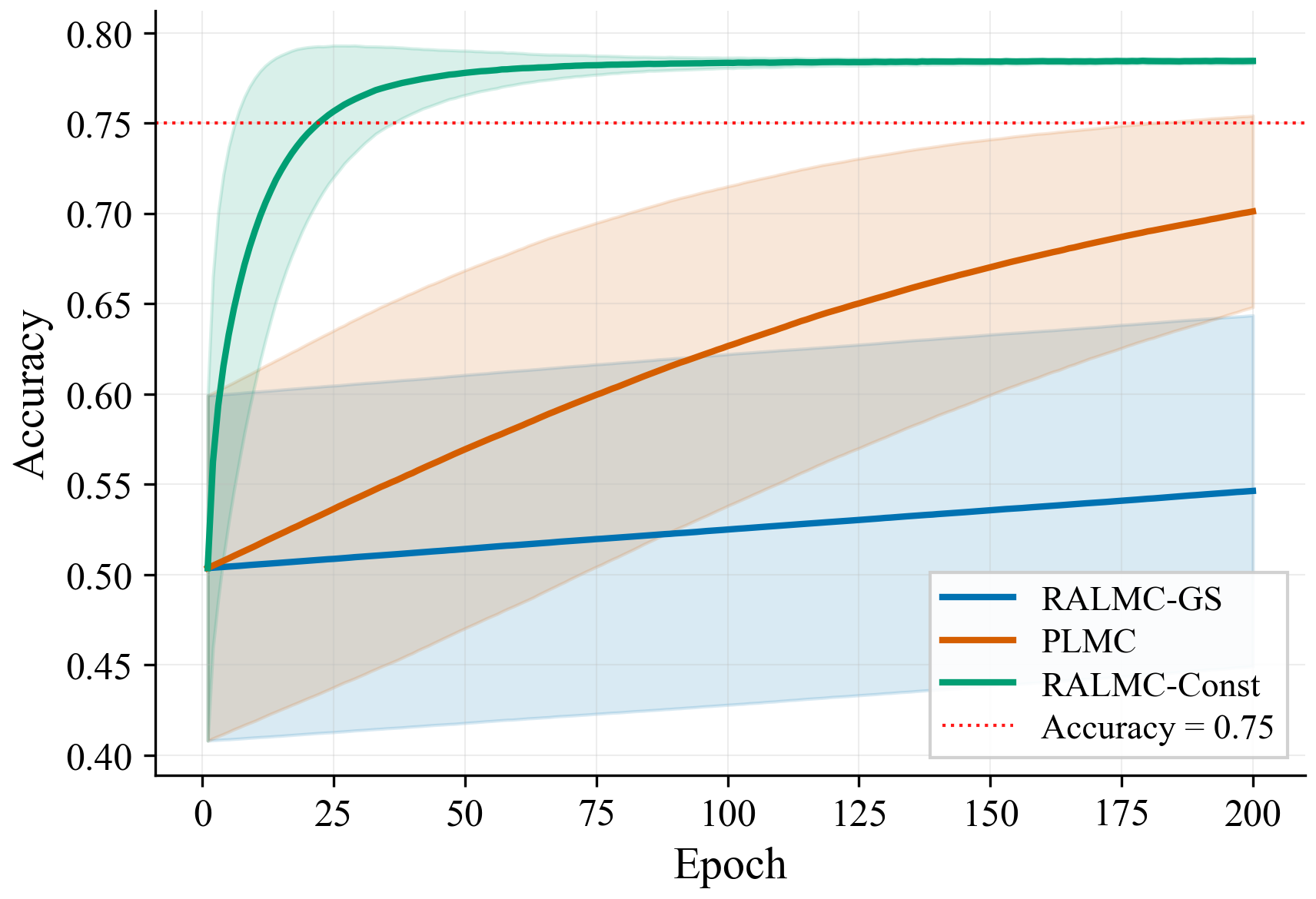} \caption{SCAD prior}
 \end{subfigure}%
\hspace{0.5cm}
 \begin{subfigure}{0.28\textwidth}
\centering \includegraphics[width=\linewidth]{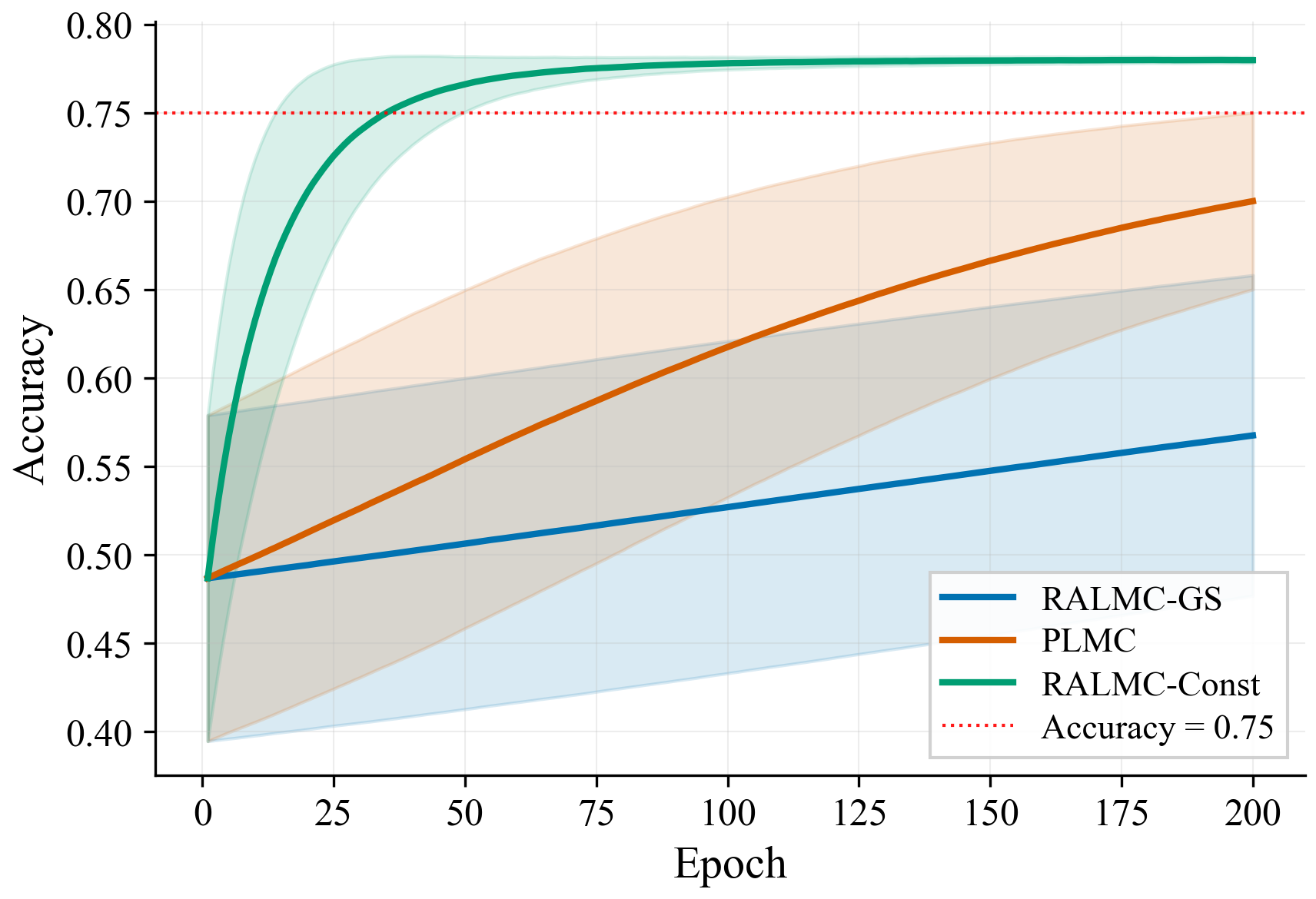} \caption{MCP prior}
 \end{subfigure}
\captionsetup{skip=4pt} \caption{Performance on accuracy of constrained Bayesian logistic regression under different prior distributions through RALMC with Gaussian smoothing, RALMC with constant potential and PLMC for the test set of the MAGIC Gamma Telescope dataset using the stepsize $\eta = 1\times10^{-6}$ in $200$ epochs}
 \label{fig:small_BLR_GS_tele_accuracy}
\end{figure}
\FloatBarrier For the MAGIC Gamma Telescope dataset, we set $\lambda = 1$ and $a = 1.5$, and evaluate the three algorithms at a larger stepsize $\eta = 5\times10^{-4}$ and a smaller stepsize $\eta = 1\times10^{-6}$. As shown in Figure~\ref{fig:small_BLR_GS_tele_accuracy}, RALMC-Const gives the fastest accuracy improvement for all three non-differentiable priors at the smaller stepsize. Figure~\ref{fig:large_BLR_GS_tele_accuracy} shows that RALMC-GS gives the best performance at the larger stepsize. This behavior is consistent with the scaling in RALMC-Const, where positive potential values amplify the effective update size and may cause instability unless the nominal stepsize is small. In contrast, RALMC-GS remains stable for larger stepsizes because the approximation error $U-\hat{U}_\mu$ is controlled through the smoothing parameter $\mu$. The same pattern appears on the Breast Cancer Wisconsin (Diagnostic) dataset in Figures~\ref{fig:small_BLR_GS_bc_accuracy} and~\ref{fig:large_BLR_GS_bc_accuracy}. For this dataset, we set $\lambda = 0.5$, use $\eta = 5\times10^{-4}$ as the smaller stepsize and $\eta = 2\times10^{-2}$ as the larger stepsize, and keep the remaining parameters unchanged.

\begin{figure}[!htbp]
\centering
 \begin{subfigure}{0.3\textwidth}
\centering \includegraphics[width=\linewidth]{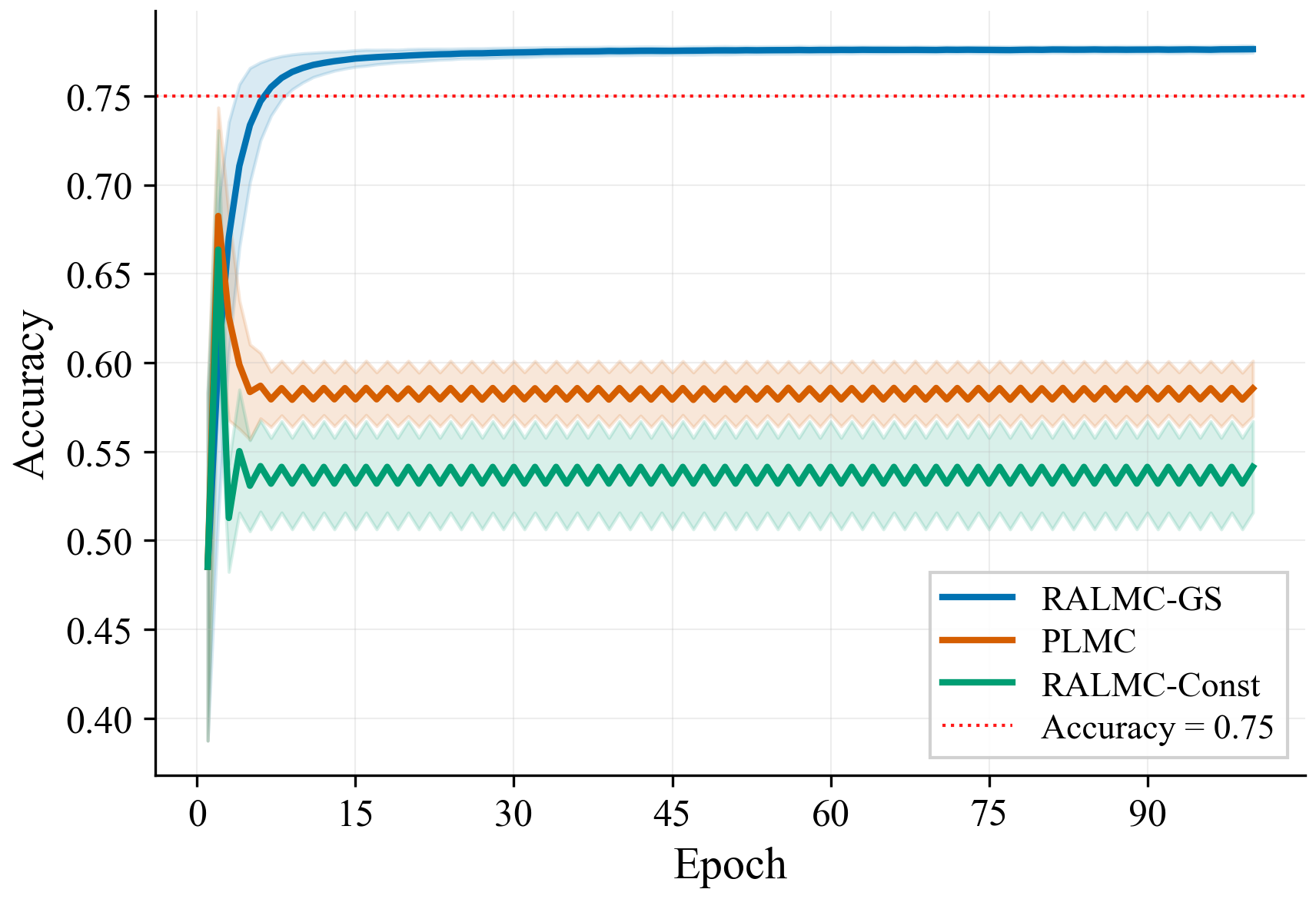} \caption{Laplace prior}
 \end{subfigure}%
\hspace{0.5cm}
 \begin{subfigure}{0.3\textwidth}
\centering \includegraphics[width=\linewidth]{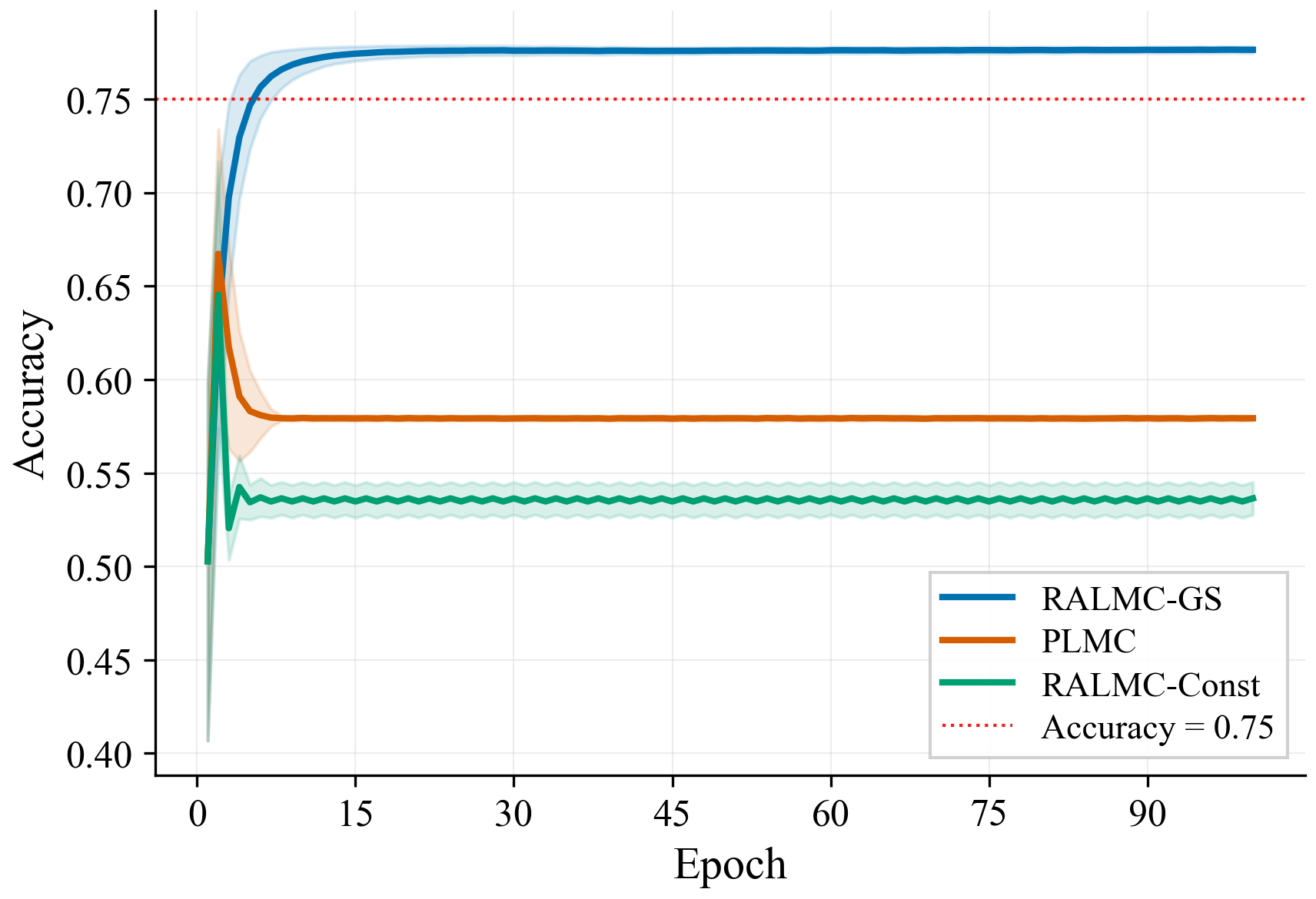} \caption{SCAD prior}
 \end{subfigure}%
\hspace{0.5cm}
 \begin{subfigure}{0.3\textwidth}
\centering \includegraphics[width=\linewidth]{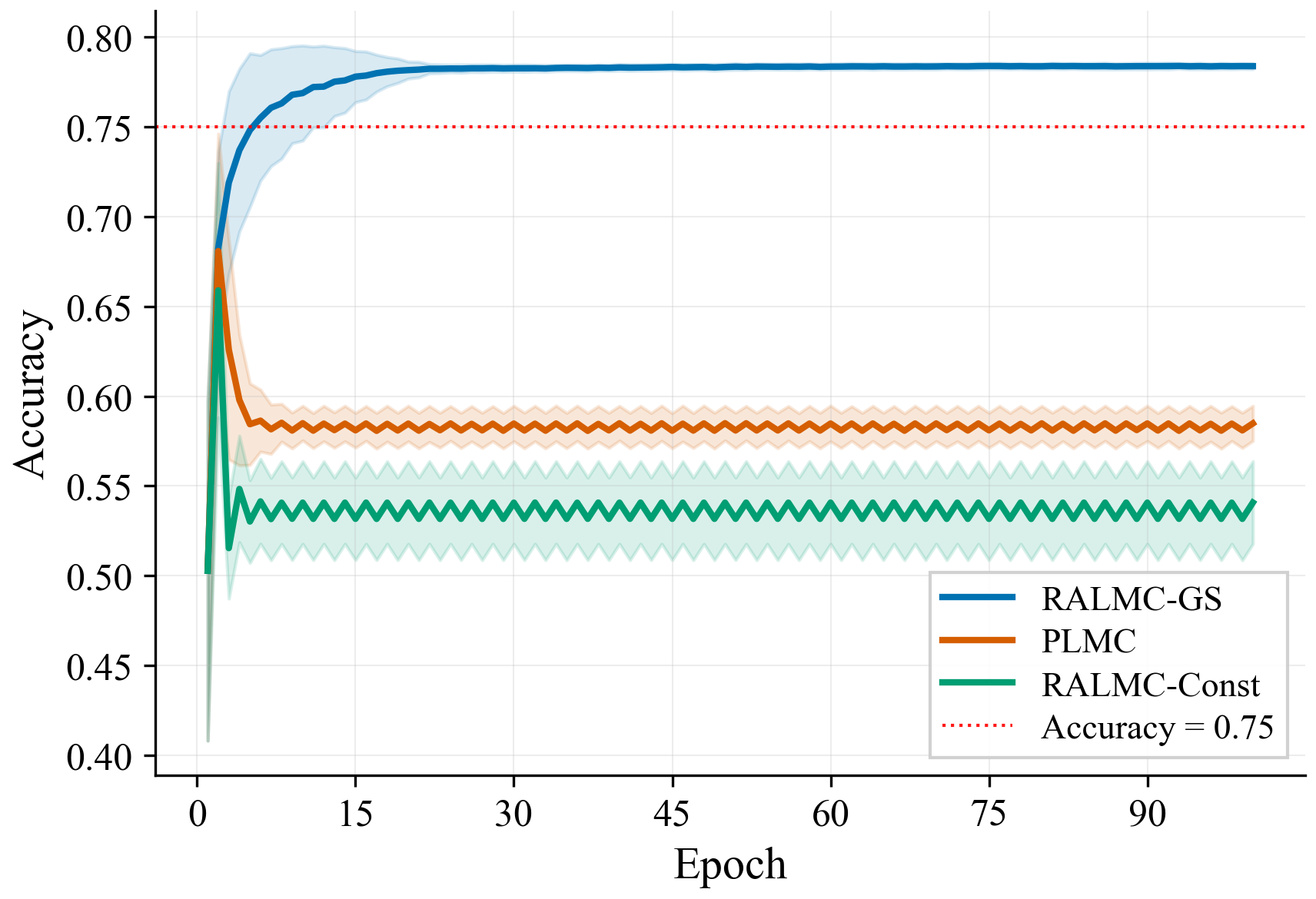} \caption{MCP prior}
 \end{subfigure}
\caption{Performance on accuracy of constrained Bayesian logistic regression under different prior distributions through RALMC with Gaussian smoothing, RALMC with constant potential and PLMC for the test set of the MAGIC Gamma Telescope dataset using the stepsize $\eta = 5\times10^{-4}$ in $100$ epochs}
 \label{fig:large_BLR_GS_tele_accuracy}
\end{figure}

\begin{figure}[!htbp]
\centering
 \begin{subfigure}{0.3\textwidth}
\centering \includegraphics[width=\linewidth]{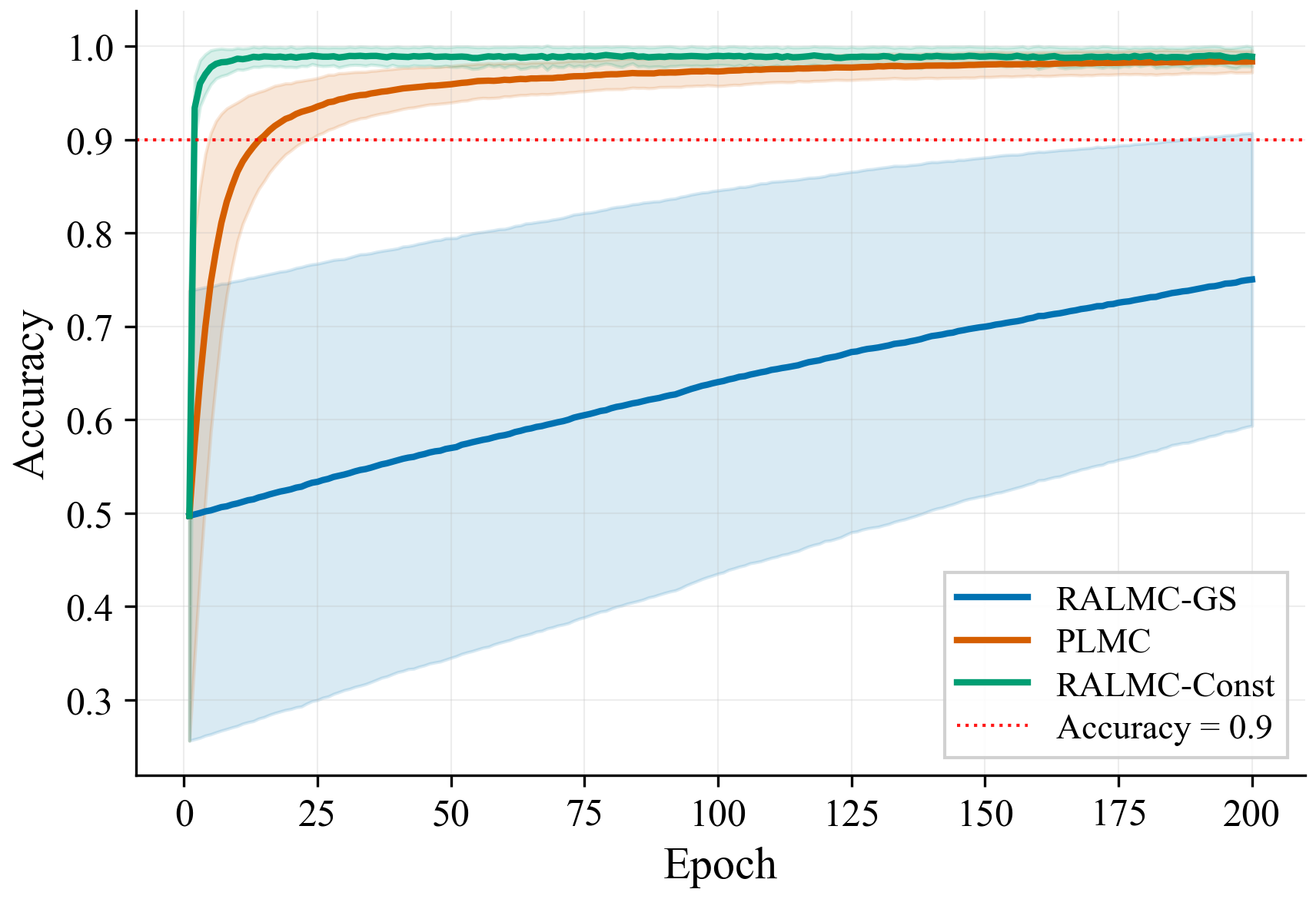} \caption{Laplace prior}
 \end{subfigure}%
\hspace{0.5cm}
 \begin{subfigure}{0.3\textwidth}
\centering \includegraphics[width=\linewidth]{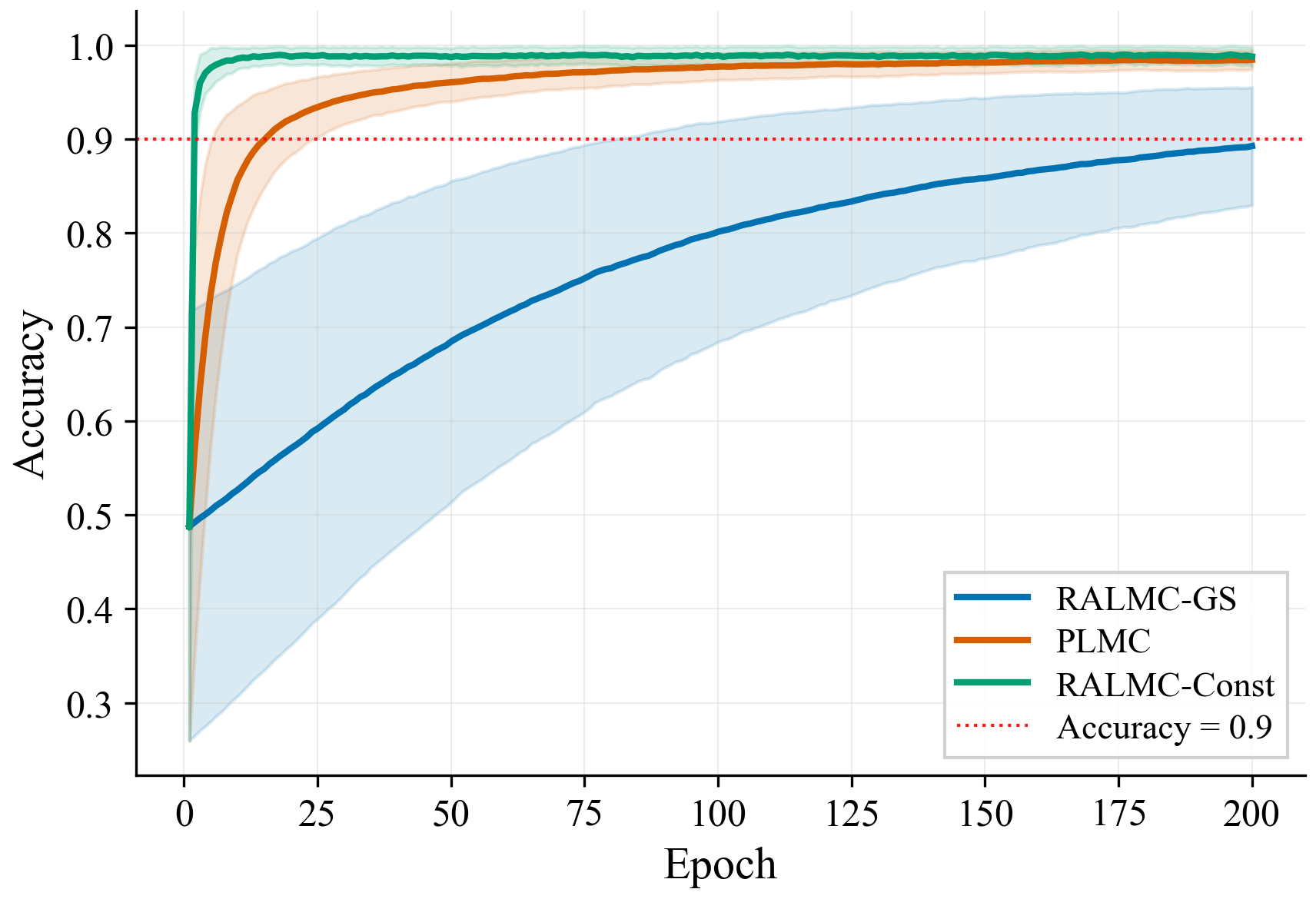} \caption{SCAD prior}
 \end{subfigure}%
\hspace{0.5cm}
 \begin{subfigure}{0.3\textwidth}
\centering \includegraphics[width=\linewidth]{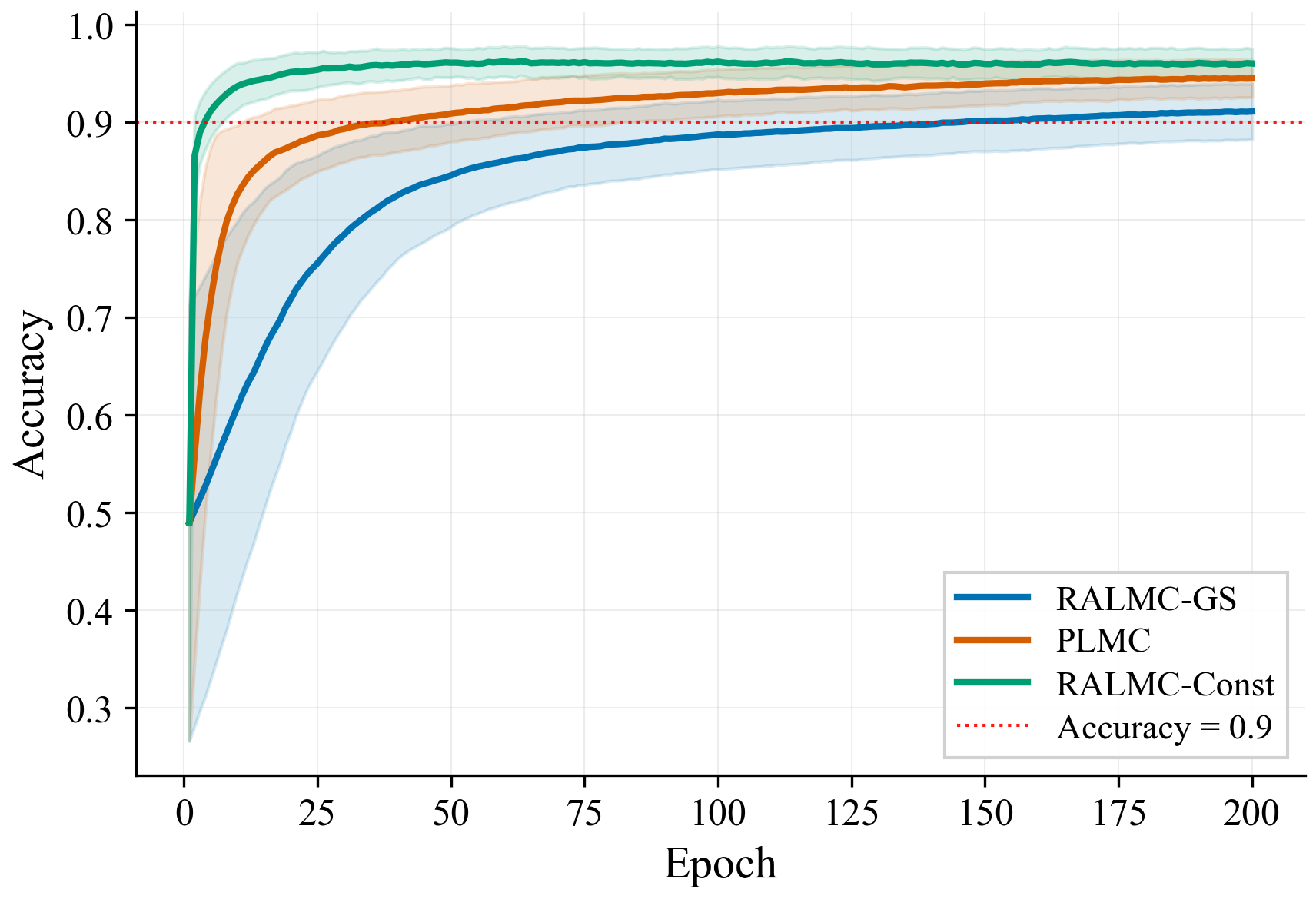} \caption{MCP prior}
 \end{subfigure}
\caption{Performance on accuracy of constrained Bayesian logistic regression under different prior distributions through RALMC with Gaussian smoothing, RALMC with constant potential and PLMC for the test set of the Breast Cancer Wisconsin (Diagnostic) dataset using the stepsize $\eta = 1\times10^{-4}$ in $200$ epochs}
 \label{fig:small_BLR_GS_bc_accuracy}
\end{figure}

\begin{figure}[!htbp]
\centering
 \begin{subfigure}{0.3\textwidth}
\centering \includegraphics[width=\linewidth]{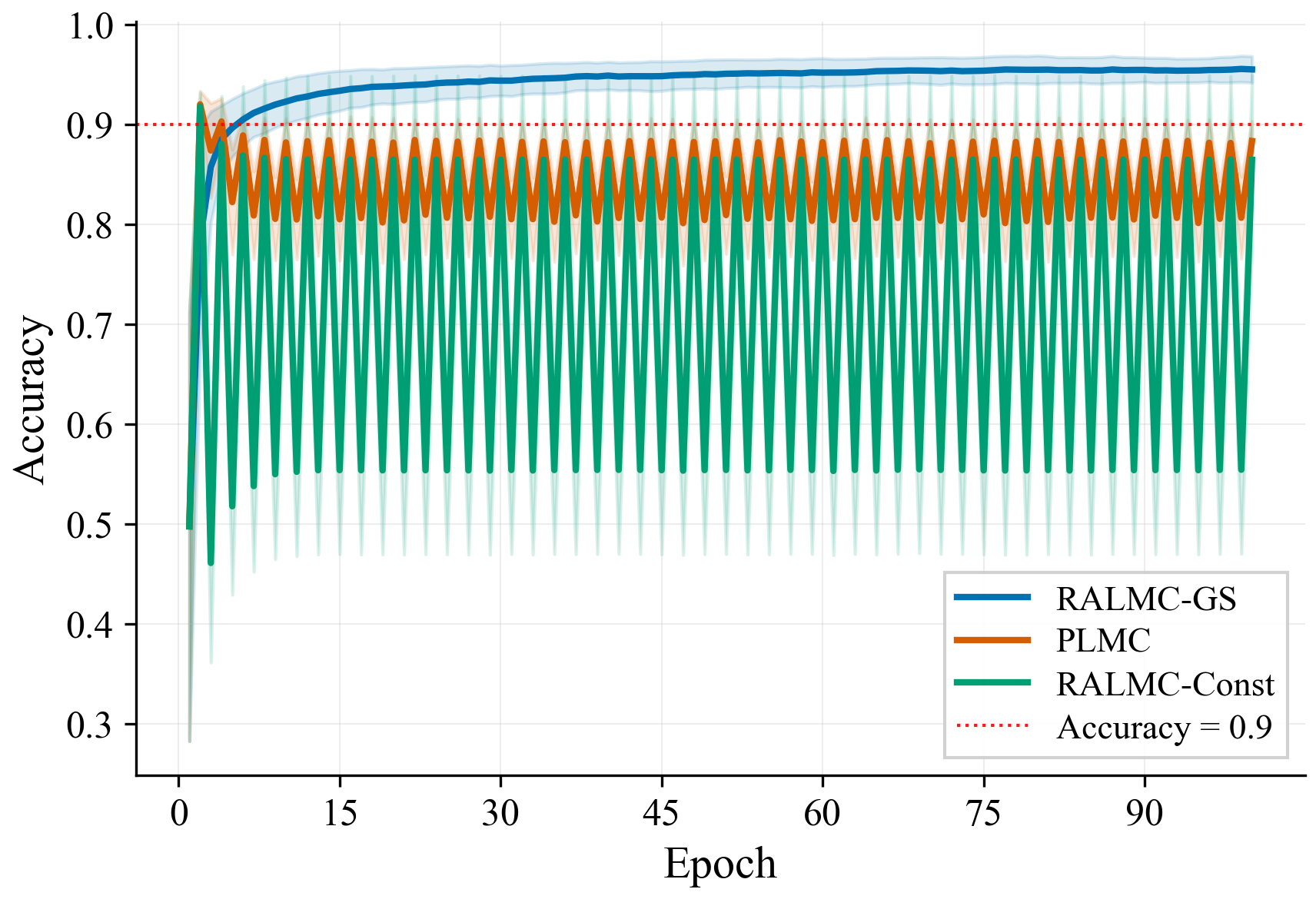} \caption{Laplace prior}
 \end{subfigure}%
\hspace{0.5cm}
 \begin{subfigure}{0.3\textwidth}
\centering \includegraphics[width=\linewidth]{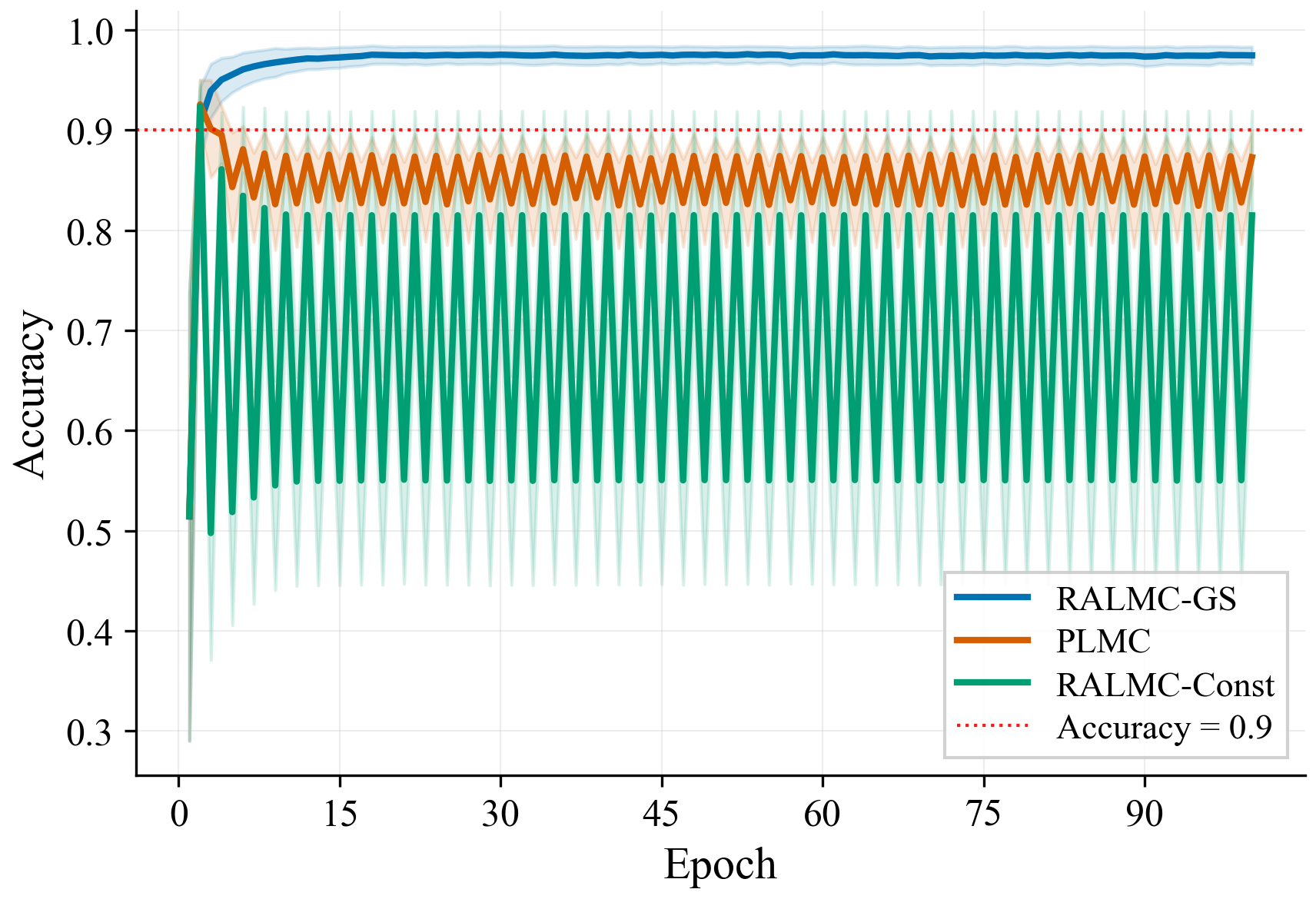} \caption{SCAD prior}
 \end{subfigure}%
\hspace{0.5cm}
 \begin{subfigure}{0.3\textwidth}
\centering \includegraphics[width=\linewidth]{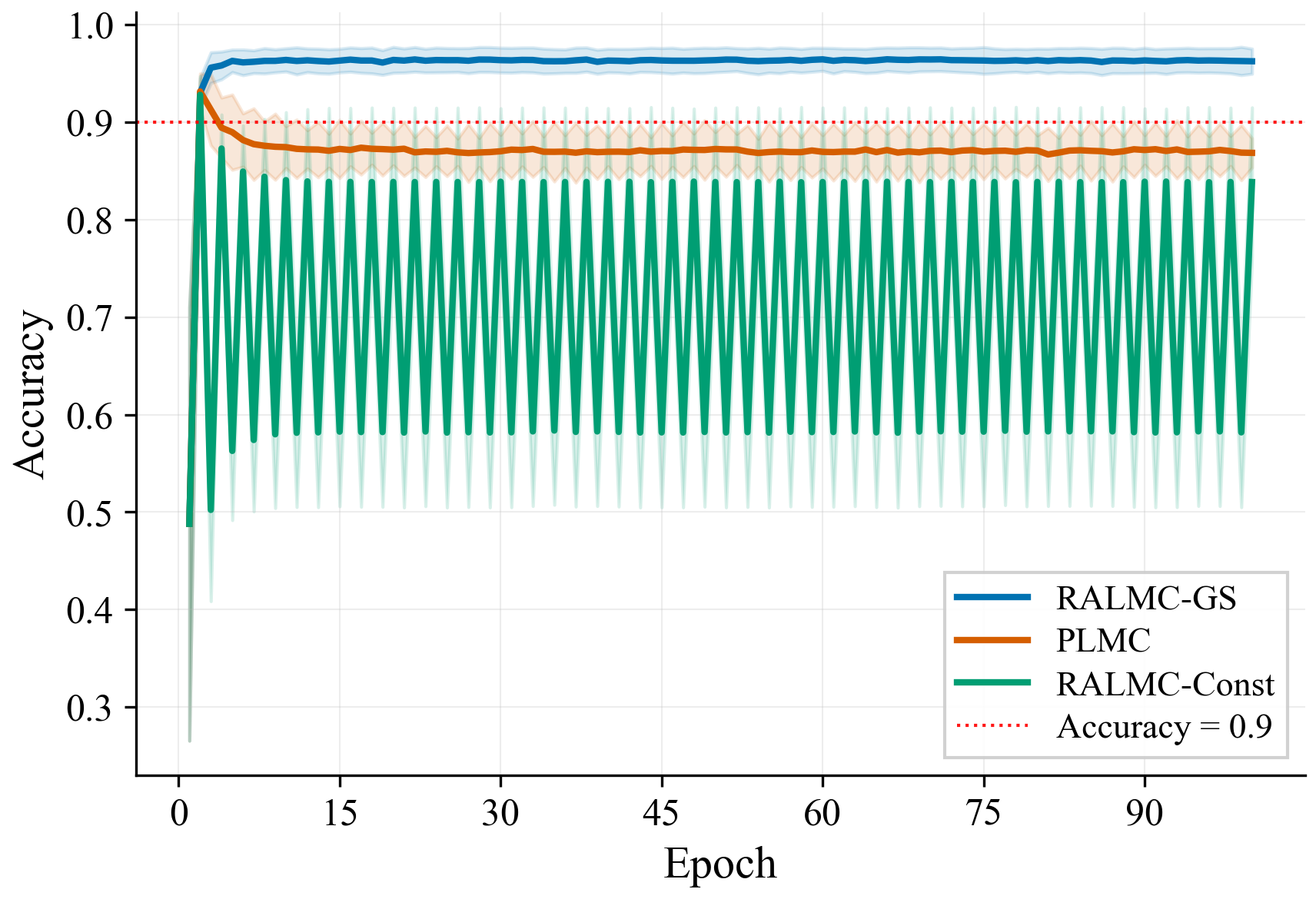} \caption{MCP prior}
 \end{subfigure}
\caption{Performance on accuracy of constrained Bayesian logistic regression under different prior distributions through RALMC with Gaussian smoothing, RALMC with constant potential and PLMC for the test set of the Breast Cancer Wisconsin (Diagnostic) dataset using the stepsize $\eta = 2\times10^{-2}$ in $100$ epochs}
 \label{fig:large_BLR_GS_bc_accuracy}
\end{figure}
\FloatBarrier

\section{Conclusion}
For constrained sampling in machine learning, first order Langevin algorithms such as PLMC have become standard tools for large scale problems, yet their reliance on differentiable log densities limits their applicability. This paper introduced RALD and its discretization RALMC, which use a smooth anchored reference potential and a state dependent scaling factor to handle non-differentiable constrained targets. The analysis gives explicit convergence guarantees and iteration complexity in 2-Wasserstein distance. The numerical results support the use of RALMC for nonsmooth constrained sampling and illustrate the complementary behavior of constant and Gaussian smoothed references.

\section*{Acknowledgments}

The authors would like to thank Qi Feng for helpful discussions. Xiaoyu Wang is supported by the Guangzhou-HKUST(GZ) Joint Funding Program (No.2025A03J3556). Xicheng Zhang is partially supported by NSFC grants of China (No. 12131019).

\bibliographystyle{alpha}
\bibliography{reflected}

\newcommand{\etalchar}[1]{$^{#1}$}
\begin{thebibliography}{CLGL{\etalchar{+}}20}

\bibitem[AC21]{Ahn2021}
Kwangjun Ahn and Sinho Chewi.
\newblock Efficient constrained sampling via the mirror-{L}angevin algorithm.
\newblock In {\em Advances in Neural Information Processing Systems (NeurIPS)},
  volume~34, pages 28405--28418. Curran Associates, Inc., 2021.

\bibitem[ADFDJ03]{andrieu2003introduction}
Christophe Andrieu, Nando De~Freitas, Arnaud Doucet, and Michael~I Jordan.
\newblock An introduction to {MCMC} for machine learning.
\newblock {\em Machine Learning}, 50(1):5--43, 2003.

\bibitem[BB00]{benamou2000computational}
Jean-David Benamou and Yann Brenier.
\newblock A computational fluid mechanics solution to the {Monge-Kantorovich}
  mass transfer problem.
\newblock {\em Numerische Mathematik}, 84:375--393, 2000.

\bibitem[BBCG08]{Bakry2008}
Dominique Bakry, Franck Barthe, Patrick Cattaiux, and Arnaud Guillin.
\newblock A simple proof of the {P}oincar\'e inequality for a large class of
  probability measures.
\newblock {\em Electronic Communications in Probability}, 13:60--66, 2008.

\bibitem[BDMP17]{Brosse}
Nicolas Brosse, Alain Durmus, \'{E}ric Moulines, and Marcelo Pereyra.
\newblock Sampling from a log-concave distribution with compact support with
  proximal {L}angevin {M}onte {C}arlo.
\newblock In {\em Proceedings of the 2017 Conference on Learning Theory},
  volume~65, pages 319--342. PMLR, 2017.

\bibitem[BEL15]{bubeck2015finite}
Sebastien Bubeck, Ronen Eldan, and Joseph Lehec.
\newblock Finite-time analysis of projected {L}angevin {M}onte {C}arlo.
\newblock In {\em Advances in Neural Information Processing Systems},
  volume~28. Curran Associates, Inc., 2015.

\bibitem[BEL18]{bubeck2018sampling}
S{\'e}bastien Bubeck, Ronen Eldan, and Joseph Lehec.
\newblock Sampling from a log-concave distribution with projected {L}angevin
  {M}onte {C}arlo.
\newblock {\em Discrete \& Computational Geometry}, 59(4):757--783, 2018.

\bibitem[BGK{\etalchar{+}}25]{DIGing}
Waheed~U. Bajwa, Mert Gurbuzbalaban, Mustafa~Ali Kutbay, Lingjiong Zhu, and
  Muhammad Zulqarnain.
\newblock {DIGing-SGLD}: Decentralized and scalable {L}angevin sampling over
  time-varying networks.
\newblock {\em arXiv:2511.12836}, 2025.

\bibitem[BGL13]{Bakry2013}
Dominique Bakry, Ivan Gentil, and Michel Ledoux.
\newblock {\em Analysis and Geometry of Markov Diffusion Operators}, volume
  348.
\newblock Springer, Cham, 2013.

\bibitem[Bre11]{brezis2011functional}
Haim Brezis.
\newblock {\em Functional Analysis, Sobolev Spaces and Partial Differential
  Equations}.
\newblock Universitext. Springer, 2011.

\bibitem[CDJB20]{Chatterji20}
Niladri Chatterji, Jelena Diakonikolas, Michael~I. Jordan, and Peter Bartlett.
\newblock Langevin {M}onte {C}arlo without smoothness.
\newblock In {\em Proceedings of the Twenty Third International Conference on
  Artificial Intelligence and Statistics}, volume 108 of {\em Proceedings of
  Machine Learning Research}, pages 1716--1726, 26--28 Aug 2020.

\bibitem[CHS87]{chiang1987diffusion}
Tzuu-Shuh Chiang, Chii-Ruey Hwang, and Shuenn~Jyi Sheu.
\newblock Diffusion for global optimization in $\mathbb{R}^n$.
\newblock {\em SIAM Journal on Control and Optimization}, 25(3):737--753, 1987.

\bibitem[CLGL{\etalchar{+}}20]{Chewi2020}
Sinho Chewi, Thibaut Le~Gouic, Cheng Lu, Tyler Maunu, Philippe Rigollet, and
  Austin Stromme.
\newblock Exponential ergodicity of mirror-{L}angevin diffusions.
\newblock In {\em Advances in Neural Information Processing Systems (NeurIPS)},
  volume~33, pages 19573--19585. Curran Associates, Inc., 2020.

\bibitem[Dal17]{Dalalyan}
Arnak~S Dalalyan.
\newblock Theoretical guarantees for approximate sampling from smooth and
  log-concave densities.
\newblock {\em Journal of the Royal Statistical Society: Series B (Statistical
  Methodology)}, 79(3):651--676, 2017.

\bibitem[DFT{\etalchar{+}}25]{DFTWZ2025}
Hengrong Du, Qi~Feng, Changwei Tu, Xiaoyu Wang, and Lingjiong Zhu.
\newblock Non-reversible {L}angevin algorithms for constrained sampling.
\newblock {\em arXiv:2501.11743}, 2025.

\bibitem[DK19]{DK2017}
Arnak~S. Dalalyan and Avetik~G. Karagulyan.
\newblock User-friendly guarantees for the {L}angevin {M}onte {C}arlo with
  inaccurate gradient.
\newblock {\em Stochastic Processes and their Applications},
  129(12):5278--5311, 2019.

\bibitem[DM17]{DM2017}
Alain Durmus and Eric Moulines.
\newblock Non-asymptotic convergence analysis for the {U}nadjusted {L}angevin
  {A}lgorithm.
\newblock {\em Annals of Applied Probability}, 27(3):1551--1587, 2017.

\bibitem[DM19]{DM2016}
Alain Durmus and Eric Moulines.
\newblock High-dimensional {B}ayesian inference via the {U}nadjusted {L}angevin
  {A}lgorithm.
\newblock {\em Bernoulli}, 25(4A):2854--2882, 2019.

\bibitem[DMM19]{Durmus19}
Alain Durmus, Szymon Majewski, and Blazej Miasojedow.
\newblock Analysis of {L}angevin {M}onte {C}arlo via convex optimization.
\newblock {\em Journal of Machine Learning Research}, 20(1):2666--2711, 2019.

\bibitem[DMP18]{Durmus18}
Alain Durmus, \'{E}ric Moulines, and Marcelo Pereyra.
\newblock Efficient {B}ayesian computation by proximal {M}arkov {C}hain {M}onte
  {C}arlo: When {L}angevin meets {M}oreau.
\newblock {\em SIAM Journal on Imaging Sciences}, 11(1):473--506, 2018.

\bibitem[FL01]{Fan01}
Jianqing Fan and Runze Li.
\newblock Variable selection via nonconcave penalized likelihood and its oracle
  properties.
\newblock {\em Journal of the American Statistical Association},
  96(456):1348--1360, 2001.

\bibitem[FOT11]{FOT2011}
Masatoshi Fukushima, Yoichi Oshima, and Masayoshi Takeda.
\newblock {\em Dirichlet Forms and Symmetric Markov Processes}, volume~19 of
  {\em De Gruyter Studies in Mathematics}.
\newblock Walter de Gruyter, second revised and extended edition, 2011.

\bibitem[GCSR95]{gelman1995bayesian}
Andrew Gelman, John~B Carlin, Hal~S Stern, and Donald~B Rubin.
\newblock {\em Bayesian Data Analysis}.
\newblock Chapman \& Hall/CRC Press, 1995.

\bibitem[GGHZ21]{DistMCMC19}
Mert G\"urb\"uzbalaban, Xuefeng Gao, Yunhan Hu, and Lingjiong Zhu.
\newblock Decentralized stochastic gradient {L}angevin dynamics and
  {H}amiltonian {M}onte {C}arlo.
\newblock {\em Journal of Machine Learning Research}, 22(239):1--69, 2021.

\bibitem[GHZ24]{GHZ2022}
Mert G\"{u}rb\"{u}zbalaban, Yuanhan Hu, and Lingjiong Zhu.
\newblock Penalized overdamped and underdamped {L}angevin {M}onte {C}arlo
  algorithms for constrained sampling.
\newblock {\em Journal of Machine Learning Research}, 25(263):1--67, 2024.

\bibitem[GIWZ24]{GIWZ2024}
Mert G\"{u}rb\"{u}zbalaban, Mohammad~Rafiqul Islam, Xiaoyu Wang, and Lingjiong
  Zhu.
\newblock Generalized {EXTRA} stochastic gradient {L}angevin dynamics.
\newblock {\em arXiv preprint arXiv:2412.01993}, 2024.

\bibitem[GNZZ25]{GNZZ2025}
Mert G\"{u}rb\"{u}zbalaban, Hoang~M. Nguyen, Xicheng Zhang, and Lingjiong Zhu.
\newblock Anchored {L}angevin algorithms.
\newblock {\em arXiv:2509.19455}, 2025.

\bibitem[GSS22]{Goldman21}
Jacob~Vorstrup Goldman, Torben Sell, and Sumeetpal~Sidhu Singh.
\newblock Gradient-based {M}arkov chain {M}onte {C}arlo for {B}ayesian
  inference with non-differentiable priors.
\newblock {\em Journal of the American Statistical Association},
  117(540):2182--2193, 2022.

\bibitem[HKRC18]{hsieh2018mirrored}
Ya-Ping Hsieh, Ali Kavis, Paul Rolland, and Volkan Cevher.
\newblock Mirrored {L}angevin dynamics.
\newblock In {\em Advances in Neural Information Processing Systems},
  volume~31. Curran Associates, Inc., 2018.

\bibitem[HKS89]{stroock-langevin-spectrum}
Richard~A Holley, Shigeo Kusuoka, and Daniel~W Stroock.
\newblock Asymptotics of the spectral gap with applications to the theory of
  simulated annealing.
\newblock {\em Journal of Functional Analysis}, 83(2):333--347, 1989.

\bibitem[IZ26]{2026IslamZhu}
Mohammad~Rafiqul Islam and Lingjiong Zhu.
\newblock Decentralized proximal stochastic gradient {L}angevin dynamics.
\newblock {\em arXiv:2605.00723}, 2026.

\bibitem[Lam21]{Lamperski2021}
Andrew Lamperski.
\newblock Projected stochastic gradient {L}angevin algorithms for constrained
  sampling and non-convex learning.
\newblock In {\em Conference on Learning Theory}, volume 134, pages 2891--2937.
  PMLR, 2021.

\bibitem[LTVW22]{TaoMirror2021}
Ruilin Li, Molei Tao, Santosh~S. Vempala, and Andre Wibisono.
\newblock The mirror {L}angevin algorithm converges with vanishing bias.
\newblock In Sanjoy Dasgupta and Nika Haghtalab, editors, {\em Proceedings of
  The 33rd International Conference on Algorithmic Learning Theory}, volume
  167, pages 718--742. PMLR, 2022.

\bibitem[LZT21]{liSqrtdDimensionDependence2021}
Ruilin Li, Hongyuan Zha, and Molei Tao.
\newblock Sqrt(d) dimension dependence of {{Langevin Monte Carlo}}.
\newblock In {\em International {{Conference}} on {{Learning
  Representations}}}, 2021.

\bibitem[MFWB22]{Mou19}
Wenlong Mou, Nicolas Flammarion, Martin~J. Wainwright, and Peter~L. Bartlett.
\newblock An efficient sampling algorithm for non-smooth composite potentials.
\newblock {\em Journal of Machine Learning Research}, 23(233):1--50, 2022.

\bibitem[NS17]{nesterov2017random}
Yurii Nesterov and Vladimir Spokoiny.
\newblock Random gradient-free minimization of convex functions.
\newblock {\em Foundations of Computational Mathematics}, 17(2):527--566, 2017.

\bibitem[Pey18]{peyre2018comparison}
R{\'e}mi Peyre.
\newblock Comparison between {$W_2$} distance and {$\dot H^{-1}$} norm, and
  localization of {W}asserstein distance.
\newblock {\em ESAIM: Control, Optimisation and Calculus of Variations},
  24(4):1489--1501, 2018.

\bibitem[RRT17]{Raginsky}
Maxim Raginsky, Alexander Rakhlin, and Matus Telgarsky.
\newblock Non-convex learning via stochastic gradient {L}angevin dynamics: a
  nonasymptotic analysis.
\newblock In {\em Conference on Learning Theory}, volume~65, pages 1674--1703.
  PMLR, 2017.

\bibitem[RT96]{RT96}
Gareth~O Roberts and Richard~L Tweedie.
\newblock Exponential convergence of {L}angevin distributions and their
  discrete approximations.
\newblock {\em Bernoulli}, 2(4):341--363, 1996.

\bibitem[S{\l}o01]{slominski2001euler}
Leszek S{\l}omi{\'n}ski.
\newblock Euler's approximations of solutions of {SDE}s with reflecting
  boundary.
\newblock {\em Stochastic Processes and their Applications}, 94(2):317--337,
  2001.

\bibitem[S{\l}o13]{slominski2013penalization}
Leszek S{\l}omi{\'n}ski.
\newblock Weak and strong approximations of reflected diffusions via
  penalization methods.
\newblock {\em Stochastic Processes and their Applications}, 123:752--763,
  2013.

\bibitem[SR20]{SR2020}
Adil Salim and Peter Richt\'{a}rik.
\newblock Primal dual interpretation of the proximal stochastic gradient
  {L}angevin algorithm.
\newblock In {\em Advances in Neural Information Processing Systems (NeurIPS)},
  volume~33, pages 3786--3796. Curran Associates, Inc., 2020.

\bibitem[Stu96]{sturm1996localIII}
Karl-Theodor Sturm.
\newblock Analysis on local {Dirichlet} spaces. {III}. the parabolic {Harnack}
  inequality.
\newblock {\em Journal de Math{\'e}matiques Pures et Appliqu{\'e}es},
  75(3):273--297, 1996.

\bibitem[Stu10]{stuart2010inverse}
Andrew~M Stuart.
\newblock Inverse problems: A {B}ayesian perspective.
\newblock {\em Acta Numerica}, 19:451--559, 2010.

\bibitem[Tan79]{tanaka1979stochastic}
Hiroshi Tanaka.
\newblock Stochastic differential equations with reflecting boundary condition
  in convex regions.
\newblock {\em Hiroshima Mathematical Journal}, 9:163--177, 1979.

\bibitem[TTV16]{teh2016consistency}
Yee~Whye Teh, Alexandre~H Thiery, and Sebastian~J Vollmer.
\newblock Consistency and fluctuations for stochastic gradient {L}angevin
  dynamics.
\newblock {\em The Journal of Machine Learning Research}, 17(1):193--225, 2016.

\bibitem[Vil09]{villani2008optimal}
C{\'e}dric Villani.
\newblock {\em Optimal Transport: Old and New}.
\newblock Springer, Berlin, 2009.

\bibitem[WTWZ26]{WTWZ2025}
Yingli Wang, Changwei Tu, Xiaoyu Wang, and Lingjiong Zhu.
\newblock Accelerated constrained sampling: A large deviations approach.
\newblock {\em Journal of Machine Learning Research}, 27(118):1--61, 2026.

\bibitem[WWZ25]{regime2025}
Xiaoyu Wang, Yingli Wang, and Lingjiong Zhu.
\newblock Regime-switching {L}angevin {M}onte {C}arlo algorithms.
\newblock {\em arXiv preprint arXiv:2509.00941}, 2025.

\bibitem[ZDF{\etalchar{+}}24]{constraint_replica}
Haoyang Zheng, Hengrong Du, Qi~Feng, Wei Deng, and Guang Lin.
\newblock Constrained exploration via reflected replica exchange stochastic
  gradient {L}angevin dynamics.
\newblock In {\em Proceedings of the 41st International Conference on Machine
  Learning}, volume 235, pages 61321--61348. PMLR, 2024.

\bibitem[ZH14]{Zhou14}
Enlu Zhou and Jiaqiao Hu.
\newblock Gradient-based adaptive stochastic search for non-differentiable
  optimization.
\newblock {\em IEEE Transactions on Automatic Control}, 59(7):1818--1832, 2014.

\bibitem[Zha10]{Zhang10}
Cun-Hui Zhang.
\newblock Nearly unbiased variable selection under minimax concave penalty.
\newblock {\em The Annals of Statistics}, 38(2):894--942, 2010.

\bibitem[ZL22]{zheng2022constrained}
Yuping Zheng and Andrew Lamperski.
\newblock Constrained {L}angevin algorithms with {L}-mixing external random
  variables.
\newblock In {\em Advances in Neural Information Processing Systems (NeurIPS)},
  volume~35, pages 20511--20521. Curran Associates, Inc., 2022.

\bibitem[ZPFP20]{Zhang2020}
Kelvin~Shuangjian Zhang, Gabriel Peyr\'{e}, Jalal Fadili, and Marcelo Pereyra.
\newblock Wasserstein control of mirror {L}angevin {M}onte {C}arlo.
\newblock In {\em Conference on Learning Theory}, volume 125, pages 3814--3841.
  PMLR, 2020.

\end{thebibliography}

\appendix

\section{Technical Proofs}

\subsection{Proof of Lemma~\ref{lem:reflected:timechange}}\label{proof:lem:reflected:timechange}
Set $a=e^{U-U_0}$ and $\mathcal L_0f=\Delta f-\langle\nabla U_0,\nabla f\rangle$. Compactness and continuity give $0<a_*\leq a\leq a^*<\infty$. Since $\nabla U_0$ is Lipschitz on the compact convex set $\mathcal C$, \cite[Theorem~4.1]{tanaka1979stochastic} gives a pathwise unique reflected strong solution
$$
dY_s=-\nabla U_0(Y_s)ds+\sqrt{2}\,dB_s-dK_s^0,
$$
where $K^0$ has the outward normal convention of \eqref{simplified eqn:RALD}. Define
$$
A_s:=\int_0^s (a(Y_r))^{-1}dr,\qquad \tau_t:=A^{-1}(t),\qquad X_t:=Y_{\tau_t}.
$$
The bounds on $a$ give $s/a^*\leq A_s\leq s/a_*$ and $a_*t\leq\tau_t\leq a^*t$. Differentiating $A_{\tau_t}=t$ gives $d\tau_t=a(X_t)dt$. Hence,
$$
\int_0^{\tau_t}\nabla U_0(Y_r)\,dr=\int_0^t a(X_u)\nabla U_0(X_u)\,du,\qquad
\langle B_{\tau_\cdot}\rangle_t=\tau_t I_d=\int_0^t a(X_u)I_d\,du.
$$
Define $W_t=\int_0^t (a(X_u))^{-1/2}\,dB_{\tau_u}$. Its quadratic variation is $tI_d$, so it is Brownian motion in the time changed filtration. With $K_t=K^0_{\tau_t}$, the equation becomes
$$
X_t=X_0-\int_0^t a(X_u)\nabla U_0(X_u)\,du+\sqrt2\int_0^t\sqrt{a(X_u)}\,dW_u-K_t.
$$
The increasing time change preserves the normal direction and boundary support of the regulator, proving that $X$ solves~\eqref{simplified eqn:RALD}. Equivalently, for every Neumann test function $f\in C^2(\overline{\mathcal C})$, optional time change and substitution in the drift integral show that
$$
f(X_t)-f(X_0)-\int_0^t a(X_r)\mathcal L_0f(X_r)dr
$$
is a martingale. Conversely, applying the inverse of $S_t=\int_0^t a(X_r)dr$ to any solution of the normally reflected martingale problem for $a\mathcal L_0$ gives a solution of the base reflected problem for $\mathcal L_0$. The two time changes are inverse path transformations, so uniqueness in law of the base problem yields uniqueness in law for $X$. The proof is complete.

\subsection{Proof of Proposition~\ref{prop:continuous:structure}}\label{sec:proof:continuous:structure}

\begin{proof}
\emph{(i)} For a Neumann test function $f$, Lemma~\ref{lem:reflected:timechange} gives
$$
\frac{P_tf(x)-f(x)}{t}=\frac1t\mathbb E^x\left[\int_0^t\mathcal Lf(X_s)\,ds\right]\rightarrow\mathcal Lf(x).
$$
The limit follows from continuity of the paths and bounded continuity of $\mathcal Lf$, proving~\eqref{eqn:generator}.

\emph{(ii)} Set $w=Z^{-1}e^{-U}$ and $q=Z^{-1}e^{-U_0}$. Continuity and compactness give constants $0<c_-\leq c_+<\infty$ bounding both weights. Hence,
$$
\begin{aligned}
c_-\|f\|_{H^1(\mathcal C)}^2 &\leq\|f\|_{L^2(\pi)}^2+\mathcal E(f,f)=\int_{\mathcal C}\left(|f|^2w+\|\nabla f\|^2q\right)\,dx
\leq c_+\|f\|_{H^1(\mathcal C)}^2.
\end{aligned}
$$
Completeness and density of $H^1(\mathcal C)$ therefore make $\mathcal E$ closed and densely defined on $L^2(\pi)$. %
For the truncation $\Phi(r)=\min\{1,\max\{0,r\}\}$, the Sobolev chain rule gives $|\nabla\Phi(f)|\leq|\nabla f|$ almost everywhere, so $\mathcal E(\Phi(f),\Phi(f))\leq\mathcal E(f,f)$, which is the Markov property. Smooth restrictions are dense in $H^1(\mathcal C)$ by the extension property of the domain, and dense in $C(\mathcal C)$ in the uniform norm, giving regularity. If $f$ is constant on a neighborhood of the support of $g$, the gradient integral defining $\mathcal E(f,g)$ vanishes, giving strong locality.

Let $\mathsf A\geq0$ be the associated self adjoint operator and $S_t=e^{-t\mathsf A}$. For a Neumann test function $f$ and $g\in H^1(\mathcal C)$, integration by parts gives
$$
\mathcal E(f,g)=-\frac1Z\int_{\mathcal C}g\,\nabla\cdot\left(e^{-U_0}\nabla f\right)\,dx
=-\langle\mathcal Lf,g\rangle_{L^2(\pi)}.
$$
Thus $-\mathsf Af=\mathcal Lf$ and $\mathsf A1=0$. The associated conservative diffusion solves the Neumann martingale problem for $\pi$ almost every initial point by \cite[Theorems~7.2.1, 7.2.2 and~5.2.2, Eq.~(5.2.26)]{FOT2011}, using boundedness and a countable dense family of Neumann tests in the norm $\|f\|_\infty+\|\mathcal Lf\|_\infty$. Uniqueness in Lemma~\ref{lem:reflected:timechange} identifies $S_t=P_t$ on $L^2(\pi)$. Hence $P_t1=1$, and spectral calculus gives $P_tf\in\mathcal D(\mathsf A)$ for $t>0$ and
$$
\frac{d}{dt}\|P_tf\|_{L^2(\pi)}^2
=-2\langle P_tf,\mathsf AP_tf\rangle_{L^2(\pi)} =-2\mathcal E(P_tf,P_tf).
$$

\emph{(iii)} On the closed convex state space $\mathcal C$, the unweighted Neumann form has the complete Euclidean intrinsic metric. Convexity gives volume doubling and the Poincar\'e inequality on $D_r=\mathcal C\cap B(x,r)$,
$$
\begin{gathered}
|D_{2r}|\leq2^d|D_r|,\quad f_{D_r}:=\frac1{|D_r|}\int_{D_r}f\,dx, \quad \int_{D_r}|f-f_{D_r}|^2dx\leq c_dr^2\int_{D_r}\|\nabla f\|^2dx.
\end{gathered}
$$
The bounds on $w,q$ make the weighted measure and form uniformly comparable to the unweighted ones, and their intrinsic metrics equivalent. Thus \cite[Theorem~3.5 and Corollary~3.6]{sturm1996localIII} gives the parabolic Harnack inequality, and \cite[Proposition~3.1, Theorems~4.1 and~4.8]{sturm1996localIII} gives a jointly continuous, strictly positive kernel $\kappa_t(x,y)$ on $(0,\infty)\times\mathcal C^2$, including the boundary. 

To justify continuity in the initial point, couple the base diffusions $Y^{x_n}$ and $Y^x$ with the same Brownian motion, where $x_n\to x$. Additive noise cancels in their difference, and the convex reflection sign gives
$$
\sup_{0\leq u\leq T}\left\|Y_u^{x_n}-Y_u^x\right\|\leq e^{\operatorname{Lip}(\nabla U_0)T}\|x_n-x\|.
$$
Since $a^{-1}$ is uniformly continuous on $\mathcal C$, the clocks $A_u^{x_n}=\int_0^u a(Y_r^{x_n})^{-1}dr$ converge uniformly to $A_u^x$ on bounded intervals. Their slopes lie between $1/a^*$ and $1/a_*$, so their inverse clocks also converge uniformly on bounded intervals. Continuity of the base paths now gives $X_t^{x_n}=Y^{x_n}_{\tau_t^{x_n}}\to Y^x_{\tau_t^x}=X_t^x$ almost surely. For bounded continuous $f$, dominated convergence yields $P_tf(x_n)\to P_tf(x)$, proving the Feller property.
For $f\in C(\mathcal C)$, the identity $P_tf(x)=\int_{\mathcal C}\kappa_t(x,y)f(y)\pi(dy)$ from part~(ii) therefore extends from $\pi$ almost every $x$ to every $x$ by continuity and full support of $\pi$. %
The Lebesgue density is $p_t(x,y)=\kappa_t(x,y)w(y)$. For fixed $t>0$, both factors are continuous and strictly positive on the compact set $\mathcal C^2$. Consequently,
$$
0<\min_{x,y\in\mathcal C}p_t(x,y)\leq\max_{x,y\in\mathcal C}p_t(x,y)<\infty,
$$
which proves the asserted extrema.
\end{proof}

\subsection{Proof of Theorem~\ref{thm:Gibbs}}\label{sec:thm:Gibbs}

\begin{proof}
By Proposition~\ref{prop:continuous:structure}~(ii), $P_t$ is self adjoint on $L^2(\pi)$ and satisfies $P_t1=1$. Hence, for every bounded measurable $f$,
$$
\begin{aligned}
    \int_{\mathcal C}P_tf\,d\pi
    &=\langle1,P_tf\rangle_{L^2(\pi)} =\langle P_t1,f\rangle_{L^2(\pi)} =\int_{\mathcal C}f\,d\pi.
\end{aligned}
$$
Thus $\pi$ is invariant. %
To prove uniqueness, fix $t>0$. Proposition~\ref{prop:continuous:structure}~(iii) gives a jointly continuous, strictly positive transition density on $\mathcal C^2$. Since $\pi(dy)=w(y)dy$ with $w$ continuous and strictly positive, the density $\kappa_t(x,y)=dP_t(x,\cdot)/d\pi(y)$ has the same properties. Compactness therefore gives $\min_{x,y\in\mathcal C}\kappa_t(x,y)>0$. Set $\varepsilon_t=\frac12\min_{x,y\in\mathcal C}\kappa_t(x,y)$. Since $\int_{\mathcal C}\kappa_t(x,y)\pi(dy)=1$, we have $0<\varepsilon_t\leq1/2$. Thus
$$
P_t(x,dy)=\varepsilon_t\pi(dy)+(1-\varepsilon_t)\frac{\kappa_t(x,y)-\varepsilon_t}{1-\varepsilon_t}\pi(dy).
$$
The second fraction defines a Markov kernel, which contracts total variation. If $\nu$ is another invariant probability measure, then
$$
\mathrm{TV}(\nu,\pi)=\mathrm{TV}(\nu P_t,\pi P_t)\leq(1-\varepsilon_t)\mathrm{TV}(\nu,\pi).
$$
Since $\varepsilon_t>0$, this forces $\mathrm{TV}(\nu,\pi)=0$, proving uniqueness.
\end{proof}

\subsection{Proof of Theorem~\ref{thm:chi:square}}\label{sec:thm:chi:square}

\begin{proof}
First, a finite Poincar\'e constant exists under the stated assumptions. The density $w=Z^{-1}e^{-U}$ has bounds $0<w_{\min}\leq w\leq w_{\max}<\infty$. Let $C_{\mathcal C}$ be a Lebesgue Poincar\'e constant on the bounded convex domain and $f_{\mathcal C}=|\mathcal C|^{-1}\int_{\mathcal C}f\,dx$. For $f\in H^1(\mathcal C)$,
$$
\begin{aligned}
\operatorname{Var}_\pi(f)
&\leq w_{\max}\int_{\mathcal C}|f-f_{\mathcal C}|^2dx\leq w_{\max}C_{\mathcal C}\int_{\mathcal C}\|\nabla f\|^2dx
\leq\frac{w_{\max}}{w_{\min}}C_{\mathcal C}\int_{\mathcal C}\|\nabla f\|^2d\pi.
\end{aligned}
$$
Thus $C_P=(w_{\max}/w_{\min})C_{\mathcal C}$ is admissible. The convergence argument below uses any valid $C_P$, not necessarily this bound. Denote $h_0:=d\mu_0/d\pi$. The assumption $\chi^2(\mu_0\Vert\pi)<\infty$ gives $h_0\in L^2(\pi)$. Since the semigroup is self adjoint by Proposition~\ref{prop:continuous:structure}~(ii), the density evolves according to
$$
    h_t=P_t^*h_0=P_th_0.
$$
For every $t>0$, Proposition~\ref{prop:continuous:structure}~(ii) gives $h_t\in\mathcal D(\mathcal L)$. Since $P_t1=1$ and $\int h_0d\pi=1$, one has $\int h_td\pi=1$. Therefore
$$
\begin{aligned}
    \frac{d}{dt}\chi^2(\mu_t\Vert\pi)
    &=\frac{d}{dt}\|h_t-1\|_{L^2(\pi)}^2
    =2\langle h_t-1,\mathcal Lh_t\rangle_{L^2(\pi)}\\
    &=2\langle h_t,\mathcal Lh_t\rangle_{L^2(\pi)} =-2\mathcal E(h_t,h_t) =-2\int_{\mathcal C}e^{U-U_0}\|\nabla h_t\|^2d\pi.
\end{aligned}
$$
By the definition of $\alpha_0$ and the Poincar\'e inequality,
$$
\begin{aligned}
    \mathcal E(h_t,h_t)
    &\geq\alpha_0\int_{\mathcal C}\|\nabla h_t\|^2d\pi
    \geq\frac{\alpha_0}{C_P}\operatorname{Var}_\pi(h_t)\\
    &=\frac{\alpha_0}{C_P}\|h_t-1\|_{L^2(\pi)}^2 =\frac{\alpha_0}{C_P}\chi^2(\mu_t\Vert\pi).
\end{aligned}
$$
Hence, for every $t>0$,
$$
    \frac{d}{dt}\chi^2(\mu_t\Vert\pi)
    \leq-\frac{2\alpha_0}{C_P}\chi^2(\mu_t\Vert\pi).
$$
Applying Gr\"{o}nwall's inequality on $[s,t]$ with $0<s<t$ gives
$$
    \chi^2(\mu_t\Vert\pi)
    \leq e^{-2\alpha_0(t-s)/C_P}\chi^2(\mu_s\Vert\pi).
$$
Strong continuity of $P_s$ on $L^2(\pi)$ yields $h_s\to h_0$ in $L^2(\pi)$ as $s\to0$. Letting $s\to0$ proves
$$
    \chi^2(\mu_t\Vert\pi)
    \leq e^{-2\alpha_0t/C_P}\chi^2(\mu_0\Vert\pi).
$$
This completes the proof.
\end{proof}

\subsection{Proof of Corollary~\ref{cor:W:2}}\label{sec:cor:W:2}

\begin{proof}
Writing $h_t=d\mu_t/d\pi$ and applying Cauchy-Schwarz inequality gives
$$
    \mathrm{TV}(\mu_t,\pi)
    =\frac12\int_{\mathcal C}|h_t-1|\,d\pi
    \leq\frac12\sqrt{\chi^2(\mu_t\Vert\pi)}.
$$
Choose a maximal coupling $(Y,Z)$ of $\mu_t$ and $\pi$, so that $\mathbb P(Y\neq Z)=\mathrm{TV}(\mu_t,\pi)$. Since $Y,Z\in\mathcal C\subset B(0,R)$,
$$
\mathcal W_2^2(\mu_t,\pi)\leq\mathbb E\|Y-Z\|^2\leq4R^2\mathbb P(Y\neq Z)=4R^2\mathrm{TV}(\mu_t,\pi).
$$
Consequently,
$$
    \mathcal W_2(\mu_t,\pi)
    \leq\sqrt2R\,\chi^2(\mu_t\Vert\pi)^{1/4}
    \leq2^{3/4}R\,\chi^2(\mu_t\Vert\pi)^{1/4}.
$$
Applying Theorem~\ref{thm:chi:square} proves the desired bound.
\end{proof}

\subsection{Proof of Lemma~\ref{lemma:bound:bm}}\label{sec:lemma:bound:bm}
\begin{proof}
We adapt the proof of Lemma 2.3 in~\cite{bubeck2018sampling}. Let $N:=\lceil T/\eta\rceil$. For $i=0,\dots,N-1$, define
$$
Y_i:=\sup_{u\in[i\eta,(i+2)\eta]}
\left\|W_u-W_{(i+1)\eta}\right\|.
$$

We first reduce the modulus of continuity to $\max_{0\leq i\leq N-1}Y_i$. Fix $s,t\in[0,T]$ with $|t-s|\leq\eta$. By symmetry, assume $s\leq t$. If $s=T$, then $s=t$. Otherwise, let $i=\lfloor s/\eta\rfloor\leq N-1$. Then $s,t\in[i\eta,(i+2)\eta]$ and by the triangle inequality,
$$
\|W_t-W_s\|
\leq \left\|W_t-W_{(i+1)\eta}\right\|+\left\|W_s-W_{(i+1)\eta}\right\|
\leq 2Y_i
\leq 2\max_{0\leq j\leq N-1}Y_j.
$$
Taking the supremum over $|t-s|\leq \eta$ yields
\begin{equation}
\label{eq:reduce_to_Y}
    \sup_{s,t\in[0,T],|t-s|\leq \eta}\|W_t-W_s\|^2 \leq 4\max_{0\leq i\leq N-1}Y_i^2.
\end{equation}
Next, we bound $\mathbb E\left[ \max_{0\leq i\leq N-1}Y_i^2 \right]$ via an $L^p$ maximal inequality. For any $p\geq 1$, by Jensen's inequality,
\begin{equation}\label{eq:max_lp}
    \mathbb E\left[ \max_{0\leq i\leq N-1}Y_i^2 \right]
    \leq \mathbb E\left[ \left( \sum_{i=0}^{N-1}Y_i^{2p} \right)^{1/p} \right]
    \leq \left( \mathbb E\left[ \sum_{i=0}^{N-1}Y_i^{2p} \right] \right)^{1/p}
    \leq N^{1/p}\|Y_0^2\|_p
    =N^{1/p}\|Y_0\|_{2p}^2,
\end{equation}
where we used the fact that $(Y_i)_{i\geq 0}$ are identically distributed by stationary increments. We now bound $\|Y_0\|_{2p}$. Note that
$$
Y_0=\sup_{u\in[0,2\eta]}\|W_u-W_{\eta}\|
\leq \sup_{v\in[0,\eta]}\|W_{\eta+v}-W_{\eta}\|
   +\sup_{v\in[0,\eta]}\|W_{\eta}-W_{\eta-v}\|.
$$
By stationary increments and time reversal, both suprema on the right hand side have the same distribution as $\sup_{v\in[0,\eta]}\|W_v\|$. Hence, by Minkowski's inequality,
$$
\|Y_0\|_{2p}\leq 2\left\Vert \sup_{v\in[0,\eta]}\|W_v\| \right\Vert _{2p}.
$$
Since $\|\cdot\|$ is convex and $W_u$ is a martingale, the process $M_u:=\|W_u\|$ is a nonnegative submartingale. Thus, by Doob's maximal inequality, for $p\geq 1$,
\begin{equation}
\label{eq:doob}
\|Y_0\|_{2p}
\leq 2\left\Vert \sup_{u\in[0,\eta]} M_u \right\Vert _{2p}
\leq 4\|M_\eta\|_{2p}
=4\|\,\|W_\eta\|\,\|_{2p}.
\end{equation}
Since $W_\eta\stackrel{d}{=}\sqrt{\eta}\,G$ with $G\sim\mathcal{N}(0,I_d)$, standard Gaussian moment bounds imply that for $p\geq 2$,
\begin{equation}
\label{eq:gauss_moment}
\|\,\|W_\eta\|\,\|_{2p}^2
=\eta\,\|\,\|G\|\,\|_{2p}^2
\leq c_1\,\eta\, d\, p,
\end{equation}
for a universal constant $c_1>0$. Combining \eqref{eq:doob} and \eqref{eq:gauss_moment}, we obtain for $p\geq 2$,
\begin{equation}
\label{eq:Y0_bound}
\|Y_0\|_{2p}^2 \leq c_2\,\eta\, d\, p,
\end{equation}
with a universal constant $c_2>0$. Plugging \eqref{eq:Y0_bound} into \eqref{eq:max_lp} gives
$$
\mathbb{E}\left[\max_{0\leq i\leq N-1} Y_i^2\right]
\leq c_2\, N^{1/p}\,\eta\, d\, p.
$$
Choose $p:=\lceil \log N\rceil$. Since $T/\eta\geq e$, we have $N\geq3$ and hence $p\geq2$. Also $N^{1/p}\leq e$. Moreover, $N\leq T/\eta+1\leq2T/\eta$ and $\log(T/\eta)\geq1$, so
\begin{equation}
p\leq\log N+1
\leq\log(T/\eta)+\log2+1
\leq(2+\log2)\log(T/\eta).
\end{equation}
Thus,
$$
\mathbb{E}\left[\max_{0\leq i\leq N-1} Y_i^2\right]
\leq c_3\, \eta\, d\, \log(T/\eta),
$$
for a universal constant $c_3>0$. Finally, using \eqref{eq:reduce_to_Y},
$$
\mathbb{E}\left[\sup_{s,t\in[0,T],|t-s|\leq \eta}\|W_t-W_s\|^2\right]
\leq 4\,\mathbb{E}\left[\max_{0\leq i\leq N-1} Y_i^2\right]
\leq C\, d\,\eta\log(T/\eta),
$$
where $C:=4c_3$ is a universal constant. This completes the proof.
\end{proof}

\subsection{Proof of Lemma~\ref{lemma:bound:K}}\label{sec:lemma:bound:K}

\begin{proof}
We first estimate the continuous regulator $\hat K$. Since $\operatorname{dist}(X_0,\partial\mathcal C)\geq r_0$ almost surely, $B(X_0,r_0)\subset\mathcal C$ almost surely. For $\hat L(ds)$ almost every $s$, the point $X_0+r_0\hat\nu_s$ belongs to $\mathcal C$. Since $\hat\nu_s\in N_{\mathcal C}^{\rm out}(\hat X_s)$, we have $\left\langle\hat\nu_s,X_0+r_0\hat\nu_s-\hat X_s\right\rangle\leq0$. Hence,
$$
\left\langle\hat X_s-X_0,\hat\nu_s\right\rangle
\geq r_0.
$$
And thus, we can get
\begin{align}
\int_0^t\left\langle\hat X_s-X_0,d\hat K_s\right\rangle
\geq r_0|\hat K|_t.\label{eq:hat:K:normal:lower}
\end{align}
Define
\begin{align}
\widehat D_t
&:=\int_0^t\left\langle\hat X_s-X_0, \, b\left(\bar X_{\lfloor s/\eta\rfloor\eta}\right)\right\rangle ds, \qquad
\widehat M_t:=\int_0^t\left\langle\hat X_s-X_0,
\sqrt2\sigma\left(\bar X_{\lfloor s/\eta\rfloor\eta}\right)dW_s\right\rangle,\nonumber\\
\widehat Q_t
&:=\int_0^t
\left\|\sqrt2\sigma\left(\bar X_{\lfloor s/\eta\rfloor\eta}\right)I_d\right\|_{\rm F}^2ds.
\end{align}
Since $\hat K$ has finite variation, the trace of the quadratic variation of $\hat X$ is $\widehat Q_t$. It\^o's formula, integrated from $0$ to $t$ with $\hat X_0=X_0$, therefore gives
\begin{equation}
\left\|\hat X_t-X_0\right\|^2
=2\widehat D_t+2\widehat M_t
-2\int_0^t\left\langle\hat X_s-X_0,d\hat K_s\right\rangle
+\widehat Q_t.\label{eq:hat:K:Ito:integrated}
\end{equation}
By \eqref{eq:hat:K:normal:lower} and \eqref{eq:hat:K:Ito:integrated},
\begin{align}
2r_0|\hat K|_t
&\leq2\widehat D_t+2\widehat M_t+\widehat Q_t
-\left\|\hat X_t-X_0\right\|^2
\leq2|\widehat D_t|+2|\widehat M_t|+\widehat Q_t.\label{eq:hat:K:pathwise}
\end{align}
Because $\hat X_s,X_0\in\mathcal C\subset B(0,R)$, the Cauchy-Schwarz inequality, It\^o's isometry, and \eqref{eqn:b:sigma:bdd} yield
\begin{align}
|\widehat D_t|^2
&\leq\left(\int_0^t2RB_bds\right)^2
=4R^2B_b^2t^2,\label{eq:hat:K:D}\\
\mathbb E|\widehat M_t|^2
&=2\mathbb E\int_0^t
\left\|\sigma\left(\bar X_{\lfloor s/\eta\rfloor\eta}\right)
\left(\hat X_s-X_0\right)\right\|^2ds\nonumber\\
&\leq2\mathbb E\int_0^t
\left\|\sigma\left(\bar X_{\lfloor s/\eta\rfloor\eta}\right)I_d\right\|_{\rm op}^2
\left\|\hat X_s-X_0\right\|^2ds
\leq8R^2S_\sigma^2t,\label{eq:hat:K:M}\\
\widehat Q_t^2
&\leq\left(\int_0^t2S_\sigma^2ds\right)^2
=4S_\sigma^4t^2.\label{eq:hat:K:Q}
\end{align}
Using $(a+b+c)^2\leq3(a^2+b^2+c^2)$ in \eqref{eq:hat:K:pathwise}, we conclude that
\begin{align}
\mathbb E|\hat K|_t^2
&\leq3r_0^{-2}\left(
\mathbb E|\widehat D_t|^2+\mathbb E|\widehat M_t|^2
+\frac14\mathbb E\widehat Q_t^2\right)\nonumber\\
&\leq3r_0^{-2}\left(4R^2B_b^2+S_\sigma^4\right)t^2
+24r_0^{-2}R^2S_\sigma^2t
:=At^2+Bt.\label{eq:hat:K:proof:final}
\end{align}
Now we compute the bound on the discretization grid. Let $t_j=j\eta$ and write
\begin{equation}\label{eq:discrete:increment:proof}
u_j:=\eta b\left(\bar X_{t_{j-1}}\right)
+\sqrt2\sigma\left(\bar X_{t_{j-1}}\right)\Delta W_j,
\qquad
k_j:=\Delta\bar K_j,
\qquad
\Delta W_j:=W_{t_j}-W_{t_{j-1}}.
\end{equation}
By~\eqref{eq:discrete:regulator:increment},
\begin{equation}\label{eq:discrete:update:proof}
\bar X_{t_j}=\bar X_{t_{j-1}}+u_j-k_j,
\qquad
k_j\in N_{\mathcal C}^{\rm out}(\bar X_{t_j}).
\end{equation}
On $\{k_j\neq0\}$, the point $X_0+r_0k_j/\|k_j\|$ belongs to $\mathcal C$. Hence, the definition of the outward normal cone implies
\begin{align}
\left\langle k_j,
X_0+r_0\frac{k_j}{\|k_j\|}-\bar X_{t_j}\right\rangle \leq 0,
\qquad
\left\langle\bar X_{t_j}-X_0,k_j\right\rangle
\geq r_0\|k_j\|.\label{eq:projection:interior:ball}
\end{align}
The same inequality is immediate when $k_j=0$. Expanding the square in \eqref{eq:discrete:update:proof} gives
\begin{align}
\left\|\bar X_{t_j}-X_0\right\|^2
=\left\|\bar X_{t_{j-1}}-X_0\right\|^2
+2\left\langle\bar X_{t_{j-1}}-X_0,u_j\right\rangle+\|u_j\|^2
-2\left\langle\bar X_{t_j}-X_0,k_j\right\rangle-\|k_j\|^2.\label{eq:discrete:energy:identity}
\end{align}
Summing~\eqref{eq:discrete:energy:identity} from $j=1$ to $k$, using $\bar X_0=X_0$, and applying~\eqref{eq:projection:interior:ball}, we obtain
\begin{align}
2r_0\sum_{j=1}^k\|k_j\|
\leq2\sum_{j=1}^k
\left\langle\bar X_{t_{j-1}}-X_0,u_j\right\rangle
+\sum_{j=1}^k\|u_j\|^2
\leq2|D_T|+2|M_T|+Q_T,\label{eq:bar:K:pathwise}
\end{align}
where
\begin{align}
&D_T:=\sum_{j=1}^k\eta
\left\langle\bar X_{t_{j-1}}-X_0,b\left(\bar X_{t_{j-1}}\right)\right\rangle,\\
&M_T:=\sum_{j=1}^k
\left\langle\bar X_{t_{j-1}}-X_0,
\sqrt2\sigma\left(\bar X_{t_{j-1}}\right)\Delta W_j\right\rangle,\quad
Q_T:=\sum_{j=1}^k\|u_j\|^2.
\end{align}
Since $|\bar K|_T=\sum_{j=1}^k\|k_j\|$, the drift bound and the discrete It\^o isometry give
\begin{align}
|D_T|^2
&\leq\left(\sum_{j=1}^k2RB_b\eta\right)^2
=4R^2B_b^2T^2,\label{eq:bar:K:D}\\
\mathbb E|M_T|^2
&=2\sum_{j=1}^k\eta\,
\mathbb E\left[
\sigma\left(\bar X_{t_{j-1}}\right)^2
\left\|\bar X_{t_{j-1}}-X_0\right\|^2\right]
\leq8R^2S_\sigma^2T.\label{eq:bar:K:M}
\end{align}
For the last term, \eqref{eqn:b:sigma:bdd} implies
\begin{equation}\label{eq:bar:K:Q:pathwise}
Q_T
\leq2B_b^2T\eta
+\frac{4S_\sigma^2}{d}Q_W,
\qquad
Q_W:=\sum_{j=1}^k\|\Delta W_j\|^2.
\end{equation}
Since $Q_W/\eta$ has the chi squared distribution with $kd$ degrees of freedom,
\begin{equation}\label{eq:QW:second:moment}
\mathbb E\left[Q_W^2\right]
=\eta^2\bigl((kd)^2+2kd\bigr)
=d^2T^2+2dT\eta.
\end{equation}
Using $(a+b)^2\leq2a^2+2b^2$, $d\geq1$, and $0<\eta\leq1$ in \eqref{eq:bar:K:Q:pathwise} yields
\begin{align}
\mathbb E\left[Q_T^2\right]
\leq8B_b^4T^2\eta^2
+\frac{32S_\sigma^4}{d^2}\mathbb E Q_W^2
\leq\left(8B_b^4+32S_\sigma^4\right)T^2+64S_\sigma^4T.
\label{eq:bar:K:Q}
\end{align}
Finally,~\eqref{eq:bar:K:pathwise} and $(a+b+c)^2\leq3(a^2+b^2+c^2)$ give
\begin{align}
\mathbb E\left[|\bar K|_T^2\right]
&\leq3r_0^{-2}\left(
\mathbb E|D_T|^2+\mathbb E|M_T|^2+\frac14\mathbb E Q_T^2
\right)\nonumber\\
&\leq3r_0^{-2}\left(4R^2B_b^2+2B_b^4+8S_\sigma^4\right)T^2
+3r_0^{-2}\left(8R^2S_\sigma^2+16S_\sigma^4\right)T
=\bar A T^2+\bar B T.
\end{align}
This completes the proof.
\end{proof}

\subsection{Proof of Lemma~\ref{lemma:alg2bri}}\label{sec:lemma:alg2bri}

\begin{proof}
Denote $r(s):=\lfloor s/\eta\rfloor\eta$. For $0\leq t\leq T$, set
\begin{equation}\label{eq:bridge:DeltaY}
\Delta_Y(t):=\sup_{0\leq s\leq t}\left\|\bar Y_s-\hat Y_s\right\|.
\end{equation}
The convex Skorokhod comparison estimate, with the outward regulator convention $X=Y-K$, gives, for every $t\leq T$,
\begin{align}
\left\|\bar X_t-\hat X_t\right\|^2
\leq\left\|\bar Y_t-\hat Y_t\right\|^2
-2\int_0^t
\left\langle
\left(\bar Y_t-\hat Y_t\right)-\left(\bar Y_s-\hat Y_s\right),
d\left(\bar K_s-\hat K_s\right)
\right\rangle.\label{eq:tanaka:bridge}
\end{align}
This is the convex Skorokhod comparison of \cite{tanaka1979stochastic}, used for penalization in \cite[Eq.~(4.3)]{slominski2013penalization}. Denote $D=\bar Y-\hat Y$, $V=\bar K-\hat K$, and $E=D-V=\bar X-\hat X$. Since $V_0=0$, the bounded variation identity with values after each jump gives
$$
\begin{aligned}
\|E_t\|^2&=\|D_t\|^2-2\langle D_t,V_t\rangle+\|V_t\|^2\\
&=\|D_t\|^2-2\int_0^t\langle D_t-D_s,dV_s\rangle-2\int_0^t\langle E_s,dV_s\rangle-\sum_{s\leq t}\|\Delta V_s\|^2.
\end{aligned}
$$
The normal cone conditions imply $\left\langle\bar X_s-\hat X_s,d\bar K_s\right\rangle\geq0$ and $\left\langle\bar X_s-\hat X_s,d\hat K_s\right\rangle\leq0$, so that $\int_0^t\langle E_s,dV_s\rangle\geq0$. Dropping this nonpositive contribution and the negative jump sum proves~\eqref{eq:tanaka:bridge}. By the definition of $\Delta_Y(t)$,
\begin{align}
2\left|\int_0^t
\left\langle
\left(\bar Y_t-\hat Y_t\right)-\left(\bar Y_s-\hat Y_s\right),
d\left(\bar K_s-\hat K_s\right)
\right\rangle\right|\nonumber
&\leq
4\Delta_Y(t)\operatorname{Var}_{[0,t]}\left(\bar K-\hat K\right)\nonumber\\
&\leq4\Delta_Y(t)\left(|\bar K|_t+|\hat K|_t\right).
\label{eq:tanaka:bridge:variation}
\end{align}
Taking the supremum in~\eqref{eq:tanaka:bridge} and using \eqref{eq:tanaka:bridge:variation} gives
\begin{equation}\label{eq:slominski:bridge:comparison}
\sup_{0\leq s\leq t}\left\|\bar X_s-\hat X_s\right\|^2
\leq\Delta^2_Y(T)
+4\Delta_Y(t)\left(|\bar K|_t+|\hat K|_t\right).
\end{equation}
The inequality applies to the c\`adl\`ag driver $\bar Y$ and the continuous driver $\hat Y$. In particular, it retains the contribution of $\bar K$. Since $\bar Y_{\lfloor s/\eta\rfloor\eta}=\hat Y_{\lfloor s/\eta\rfloor\eta}$ and the coefficients of $\hat Y$ are constant on $\left[\lfloor s/\eta\rfloor\eta,s\right]$,
\begin{align}
\bar Y_s-\hat Y_s
&=-\int_{\lfloor s/\eta\rfloor\eta}^s b\left(\bar X_{\lfloor s/\eta\rfloor\eta}\right)du
-\int_{\lfloor s/\eta\rfloor\eta}^s\sqrt2\sigma\left(\bar X_{\lfloor s/\eta\rfloor\eta}\right)dW_u\nonumber\\
&=-\left(s-\lfloor s/\eta\rfloor\eta\right)b\left(\bar X_{\lfloor s/\eta\rfloor\eta}\right)
-\sqrt2\sigma\left(\bar X_{\lfloor s/\eta\rfloor\eta}\right)\left(W_s-W_{\lfloor s/\eta\rfloor\eta}\right).
\label{eq:bridge:Y:difference}
\end{align}
Therefore, using $(a+b)^2\leq2a^2+2b^2$ and \eqref{eqn:b:sigma:bdd},
\begin{align}
\Delta^2_Y(T)
&\leq2B_b^2\eta^2
+4S_\sigma^2\sup_{0\leq s\leq T}
\left\|W_s-W_{r(s)}\right\|^2.\label{eq:bridge:Y:pathwise}
\end{align}
Lemma~\ref{lemma:bound:bm} gives
\begin{equation}\label{eq:bridge:Y:L2}
\mathbb E\left[\Delta^2_Y(T)\right]
\leq2B_b^2\eta^2+4CdS_\sigma^2\eta\log(T/\eta).
\end{equation}
Taking expectations in~\eqref{eq:slominski:bridge:comparison} at $t=T$ and applying the Cauchy-Schwarz inequality twice yield
\begin{align}
\mathbb E\left[\sup_{0\leq t\leq T}
\left\|\bar X_t-\hat X_t\right\|^2\right]
&\leq\mathbb E\left[\Delta^2_Y(T)\right]
+4\mathbb E\left[\Delta_Y(T)|\bar K|_T\right]
+4\mathbb E\left[\Delta_Y(T)|\hat K|_T\right]\nonumber\\
&\leq2B_b^2\eta^2+4CdS_\sigma^2\eta\log(T/\eta)\nonumber\\
&\quad+4\left(2B_b^2\eta^2
+4CdS_\sigma^2\eta\log(T/\eta)\right)^{1/2}
\cdot\left(\sqrt{\bar A T^2+\bar B T}
+\sqrt{AT^2+BT}\right).
\label{eq:bridge:first:bound}
\end{align}
Since $T\geq e$,
\begin{align}
\sqrt{AT^2+BT}
&\leq T\left(\sqrt A+\sqrt{B/T}\right)
\leq T\left(\sqrt A+\sqrt{B/e}\right),\nonumber\\
\sqrt{\bar A T^2+\bar B T}
&\leq T\left(\sqrt{\bar A}+\sqrt{\bar B/e}\right).
\label{eq:bridge:regulator:simplification}
\end{align}
Moreover,
\begin{equation}
\left(2B_b^2\eta^2
+4CdS_\sigma^2\eta\log(T/\eta)\right)^{1/2}\leq\sqrt2B_b\eta
+2\sqrt{Cd}\,S_\sigma\sqrt{\eta\log(T/\eta)}.
\label{eq:bridge:Y:L2:square:root}
\end{equation}
Substituting \eqref{eq:bridge:regulator:simplification} and \eqref{eq:bridge:Y:L2:square:root} into \eqref{eq:bridge:first:bound} gives
\begin{align}
\mathbb E\left[\sup_{0\leq t\leq T}
\left\|\bar X_t-\hat X_t\right\|^2\right]
&\leq2B_b^2\eta^2+4CdS_\sigma^2\eta\log(T/\eta)\nonumber\\
&\qquad\qquad+4TH_K\left(
\sqrt2B_b\eta
+2\sqrt{Cd}\,S_\sigma\sqrt{\eta\log(T/\eta)}
\right).
\label{eq:bridge:before:uniform}
\end{align}
To collect these terms, note that $0<\eta\leq1$, $T\geq e$, and $\log(T/\eta)\geq1$ imply
$$
\begin{aligned}
\eta^2&\leq\sqrt\eta\,T\log(T/\eta),\qquad \eta\log(T/\eta)\leq\sqrt\eta\,T\log(T/\eta),\\
T\eta&\leq\sqrt\eta\,T\log(T/\eta),\qquad T\sqrt{\eta\log(T/\eta)}\leq\sqrt\eta\,T\log(T/\eta).
\end{aligned}
$$
Applying these inequalities term by term in~\eqref{eq:bridge:before:uniform} gives
$$
\begin{aligned}
\mathbb E\left[\sup_{0\leq t\leq T}\left\|\bar X_t-\hat X_t\right\|^2\right]
&\leq\left[2B_b^2+4CdS_\sigma^2+4H_K\left(\sqrt2B_b+2\sqrt{Cd}\,S_\sigma\right)\right]\sqrt\eta\,T\log(T/\eta)\\
&=C_0\sqrt\eta\,T\log(T/\eta),
\end{aligned}
$$
where the last equality is~\eqref{eq:bridge:C0}. This proves~\eqref{eq:bridge:uniform}.
\end{proof}

\subsection{Proof of Theorem~\ref{global:error:2}}\label{sec:global:error:2}

\begin{proof}
Recall that $K$ and $\hat K$ are outward regulators of bounded variation on every finite time interval in the decompositions $X=Y-K$ and $\hat X=\hat Y-\hat K$. Let $L(dt):=d|K|_t$ and $\hat L(dt):=d|\hat K|_t$ be their total variation measures. The polar decomposition of the vector measures $dK$ and $d\hat K$ gives measurable vector fields $\nu$ and $\hat\nu$ such that
\begin{equation}
dK_t=\nu_tL(dt),
\qquad
d\hat K_t=\hat\nu_t\hat L(dt),
\label{eq:theorem:reflection:polar}
\end{equation}
with $\|\nu_t\|=1$ for $L(dt)$ almost every $t$ and $\|\hat\nu_t\|=1$ for $\hat L(dt)$ almost every $t$. The defining properties of the two Skorokhod problems further give
$$
\nu_t\in N_{\mathcal C}^{\rm out}(X_t)
\quad L(dt)\text{-almost everywhere},
\qquad
\hat\nu_t\in N_{\mathcal C}^{\rm out}(\hat X_t)
\quad \hat L(dt)\text{-almost everywhere},
$$
and $L$ and $\hat L$ are supported on the times at which $X$ and $\hat X$ lie on $\partial\mathcal C$, respectively. This establishes \eqref{eq:theorem:reflection:polar} from the defining properties of the two Skorokhod decompositions.

We now verify the sign of each reflection term. For $L(dt)$ almost every $t$, one has $\nu_t\in N_{\mathcal C}^{\rm out}(X_t)$. Since $\hat X_t\in\mathcal C$, substitute $y=\hat X_t$ into \eqref{eq:outward:normal:cone} to obtain
\begin{align}
\left\langle X_t-\hat X_t,\nu_t\right\rangle\geq0
\qquad L(dt)\text{-almost everywhere}.
\label{eq:theorem:reflection:first:sign}
\end{align}
For $\hat L(dt)$ almost every $t$, one has $\hat\nu_t\in N_{\mathcal C}^{\rm out}(\hat X_t)$. Since $X_t\in\mathcal C$, substitute $y=X_t$ into \eqref{eq:outward:normal:cone} to obtain
\begin{align}
\left\langle X_t-\hat X_t,\hat\nu_t\right\rangle\leq0
\qquad \hat L(dt)\text{-almost everywhere}.
\label{eq:theorem:reflection:second:sign}
\end{align}
Combining \eqref{eq:theorem:reflection:polar}-\eqref{eq:theorem:reflection:second:sign} yields
\begin{align}
\int_0^t\left\langle X_s-\hat X_s,d\left(K_s-\hat K_s\right)\right\rangle
&=\int_0^t\left\langle X_s-\hat X_s,\nu_s\right\rangle L(ds)-\int_0^t\left\langle X_s-\hat X_s,\hat\nu_s\right\rangle\hat L(ds)\geq0.\label{Skorokhod:ineqn}
\end{align}
By~\eqref{simplified eqn:RALD} and~\eqref{eqn:hat:x},
\begin{equation}
d\hat{X}_t = b\left(\bar{X}_{\lfloor t/\eta \rfloor\eta}\right)dt + \sqrt{2}\sigma\left(\bar{X}_{\lfloor t/\eta \rfloor\eta}\right)dW_t - d\hat{K}_t, \qquad \hat{X}_0=X_0,
\end{equation}
and
\begin{equation}
d{X}_t = b({X}_t)dt + \sqrt{2}\sigma({X}_t)dW_t - d{K}_t.
\end{equation}
Both $X$ and $\hat X$ are continuous semimartingales. Set $Z_t:=X_t-\hat X_t$ and $r(t):=\lfloor t/\eta\rfloor\eta$. Then
\begin{align}
dZ_t
&=\left[b(X_t)-b(\hat X_t)+b(\hat X_t)-b\left(\bar X_{r(t)}\right)\right]dt
\nonumber
\\
&\qquad\qquad+\sqrt2\left[\sigma(X_t)-\sigma(\hat X_t)
+\sigma(\hat X_t)-\sigma\left(\bar X_{r(t)}\right)\right]dW_t
-d\left(K_t-\hat K_t\right).
\end{align}
It\^o's formula gives
\begin{align}
d\|Z_t\|^2
&=2\left\langle Z_t,
b(X_t)-b(\hat X_t)+b(\hat X_t)-b\left(\bar X_{r(t)}\right)\right\rangle dt
\nonumber
\\
&\qquad+2\sqrt2\left\langle Z_t,
\left[\sigma(X_t)-\sigma(\hat X_t)
+\sigma(\hat X_t)-\sigma\left(\bar X_{r(t)}\right)\right]dW_t\right\rangle
\nonumber
\\
&\qquad\qquad+2\left\|\left[\sigma(X_t)-\sigma(\hat X_t)
+\sigma(\hat X_t)-\sigma\left(\bar X_{r(t)}\right)\right]I_d\right\|_{\rm F}^2dt
-2\left\langle Z_t,d\left(K_t-\hat K_t\right)\right\rangle
\nonumber
\\
&\leq2\left\langle Z_t,b(X_t)-b(\hat X_t)\right\rangle dt
+\left(\|Z_t\|^2+\left\|b(\hat X_t)-b\left(\bar X_{r(t)}\right)\right\|^2\right)dt
\nonumber
\\
&\qquad+2\sqrt2\left\langle Z_t,
\left[\sigma(X_t)-\sigma(\hat X_t)
+\sigma(\hat X_t)-\sigma\left(\bar X_{r(t)}\right)\right]dW_t\right\rangle
\nonumber
\\
&\qquad\qquad+4\left\|\left[\sigma(X_t)-\sigma(\hat X_t)\right]I_d\right\|_{\rm F}^2dt
+4\left\|\left[\sigma(\hat X_t)-\sigma\left(\bar X_{r(t)}\right)\right]I_d\right\|_{\rm F}^2dt,
\label{eq:ito:X:hat:X}
\end{align}
where~\eqref{Skorokhod:ineqn}, $2\langle u,v\rangle\leq
\|u\|^2+\|v\|^2$, and $2\|A+B\|_{\rm F}^2\leq
4\|A\|_{\rm F}^2+4\|B\|_{\rm F}^2$ were used. Now by assumptions, we can have
\begin{align}
2\left\langle Z_t,b(X_t)-b(\hat X_t)\right\rangle
&\leq-2m\|Z_t\|^2,\qquad
\left\|b(\hat X_t)-b\left(\bar X_{r(t)}\right)\right\|^2 \leq M^2\left\|\hat X_t-\bar X_{r(t)}\right\|^2,\nonumber\\
4\left\|\left[\sigma(X_t)-\sigma(\hat X_t)\right]I_d\right\|_{\rm F}^2
&\leq4\alpha\|Z_t\|^2,\qquad
4\left\|\left[\sigma(\hat X_t)-\sigma\left(\bar X_{r(t)}\right)\right]I_d\right\|_{\rm F}^2\leq4\alpha\left\|\hat X_t-\bar X_{r(t)}\right\|^2.
\label{eq:global:coefficient:bounds}
\end{align}
Let $q:=2m-1-4\alpha>0$ and $c:=M^2+4\alpha$. Apply the product formula to $e^{qt}\|Z_t\|^2$ in \eqref{eq:ito:X:hat:X}, using \eqref{eq:global:coefficient:bounds}. The reflection contribution $-2e^{qt}\langle Z_t,d(K_t-\hat K_t)\rangle$ is a nonpositive measure by the two pointwise normal cone inequalities above. The stochastic integral is square integrable on every fixed interval $[0,T]$, since $Z$, $\sigma$, and $e^{qt}$ are bounded there. Integrating, taking expectations, and using $Z_0=0$ gives
\begin{equation}
\label{eqn:lyapunov}
e^{qT}\mathbb E\|Z_T\|^2
\leq c\int_0^T e^{qs}
\mathbb E\left\|\hat X_s-\bar X_{r(s)}\right\|^2\,ds.
\end{equation}
Since $\bar X$ is constant on $[r(t),r(t)+\eta)$, $\bar X_{r(t)}=\bar X_t$. Lemma~\ref{lemma:alg2bri} gives
\begin{equation}\label{eq:bar:X:hat:X}
\sup_{0\leq t\leq T}\mathbb E\left\|\bar X_t-\hat X_t\right\|^2
\leq\mathbb E\left[\sup_{0\leq t\leq T}
\left\|\bar X_t-\hat X_t\right\|^2\right]
\leq C_0\sqrt\eta\,T\log(T/\eta),
\end{equation}
where $C_0$ is defined in~\eqref{eq:bridge:C0}. Dividing \eqref{eqn:lyapunov} by $e^{qT}$ and using $\int_0^T e^{-q(T-s)}ds\leq1/q$ gives
\begin{align}
\mathbb E\|Z_T\|^2
\leq c\int_0^T e^{-q(T-s)}
\mathbb E\left\|\hat X_s-\bar X_s\right\|^2ds
\leq\frac{c}{q}\sup_{0\leq s\leq T}
\mathbb E\left\|\hat X_s-\bar X_s\right\|^2.
\label{eq:global:gronwall}
\end{align}
Combining~\eqref{eq:global:gronwall} with~\eqref{eq:bar:X:hat:X} yields
\begin{equation}\label{eq:X:hatX:uniform}
\mathbb E\left\|X_T-\hat X_T\right\|^2
\leq \frac{M^2+4\alpha}{2m-1-4\alpha}
C_0\sqrt\eta\,T\log(T/\eta).
\end{equation}
Finally, Minkowski's inequality in $L^2$ and \eqref{eq:bar:X:hat:X}-\eqref{eq:X:hatX:uniform} give
\begin{align}
\mathcal W_2\left(\Law(\bar X_T),\Law(X_T)\right)
&\leq \left(\mathbb E\left\|\bar X_T-X_T\right\|^2\right)^{1/2}\nonumber\\
&\leq \left(\mathbb E\left\|\bar X_T-\hat X_T\right\|^2\right)^{1/2}
+\left(\mathbb E\left\|\hat X_T-X_T\right\|^2\right)^{1/2}\nonumber\\
&\leq \left(1+\sqrt{\frac{M^2+4\alpha}{2m-1-4\alpha}}\right)
\sqrt{C_0\,T\log(T/\eta)}\,\eta^{1/4}.
\end{align}
Squaring this inequality proves the result. The proof is complete.
\end{proof}

\subsection{Proof of Theorem~\ref{thm:final}}\label{sec:thm:final}

\begin{proof}
Fix $T=k\eta\geq e$. By Corollary~\ref{cor:W:2},
\begin{equation}
\mathcal{W}_{2}(\mathrm{Law}(X_{k\eta}),\pi)\leq 2^{3/4}R\left(\chi^{2}(\mu_{0}\Vert\pi)\right)^{1/4}\cdot e^{-\alpha_{0} k\eta/(2C_{P})},
\end{equation}
where $\alpha_{0}$ and $C_{P}$ are defined in Theorem~\ref{thm:chi:square}. Since $2m>1+4\alpha$, Theorem~\ref{global:error:2} gives
\begin{align}
\mathcal{W}_2\left(\mathrm{Law}(\bar{X}_{k\eta}),\mathrm{Law}(X_{k\eta})\right)  & \leq \eta^{1/4}(k\eta)^{1/2}\left(\log(k)\right)^{1/2}
\left(1+\sqrt{\frac{M^2+4\alpha}{2m-1-4\alpha}}\right)\sqrt{C_0},
\end{align}
where $C_{0}$ is given in Theorem~\ref{global:error:2}. Finally, by using the triangle inequality for $\mathcal W_2$ distance and the fact that $\mathrm{Law}(\bar{X}_{k\eta})=\mathrm{Law}(x_{k})$, we conclude that
\begin{align}
\mathcal{W}_{2}(\mathrm{Law}(x_{k}),\pi)
&\leq
\mathcal{W}_{2}(\mathrm{Law}(X_{k\eta}),\pi)
+\mathcal{W}_{2}(\mathrm{Law}(x_{k}),\mathrm{Law}(X_{k\eta}))
\nonumber
\\
&\leq 2^{3/4}R\left(\chi^{2}(\mu_{0}\Vert\pi)\right)^{1/4}\cdot e^{-\alpha_{0} k\eta/(2C_{P})}
+\eta^{1/4}(k\eta)^{1/2}\left(\log(k)\right)^{1/2}\bar{C}_{0},
\end{align}
where $\bar{C}_{0}$ is defined in \eqref{eq:thm:final:Cbar}. This proves \eqref{eq:thm:final:explicit}.

Suppose now that $\sigma\equiv1$ and that $B_b,M,m,R,r_0$ are independent of $d$. Then $\alpha=0$ and $S_\sigma=\|I_d\|_{\mathrm F}=\sqrt d$. The definitions in Lemmas~\ref{lemma:bound:K} and~\ref{lemma:alg2bri} give
$$
\begin{gathered}
A=\mathcal O(d^2),\quad B=\mathcal O(d),\quad
\bar A=\mathcal O(d^2),\quad \bar B=\mathcal O(d^2),\quad
H_K=\mathcal O(d),\quad C_0=\mathcal O(d^2).
\end{gathered}
$$
Since $M$ and $m$ are independent of $d$, \eqref{eq:thm:final:Cbar} yields
$$
\bar C_0
=\left(1+\sqrt{\frac{M^2}{2m-1}}\right)\sqrt{C_0} =\mathcal O(d).
$$
Substituting this estimate into \eqref{eq:thm:final:explicit} proves \eqref{eq:thm:final:dimension}. The proof is complete.
\end{proof}

\subsection{Proof of Corollary~\ref{cor:complexity}}\label{sec:cor:complexity}

\begin{proof}
Let $A_0$ and $T_\ep$ be as in the statement. The definition of $T_\ep$ gives
$$
    \frac{\alpha_0T_\ep}{2C_P}
    \geq\log\left(1 \vee \frac{2A_0}{\ep}\right).
$$
Consequently,
$$
    A_0\exp\left(-\frac{\alpha_0T_\ep}{2C_P}\right)
    \leq\frac{A_0}{\max\left\{1,2A_0/\ep\right\}}
    \leq\frac\ep2.
$$

Choose $k$ as in \eqref{eq:complexity:valid:k} and set $\eta=T_\ep/k$. Then $0<\eta\leq1$ and $k\eta=T_\ep\geq e$, so Theorem~\ref{thm:final} applies. The fourth power of its discretization term is
$$
\begin{aligned}
    \left(\bar C_0\eta^{1/4}(k\eta)^{1/2}(\log k)^{1/2}\right)^4
    &=\bar C_0^4\frac{T_\ep}{k}T_\ep^2(\log k)^2=\frac{\bar C_0^4T_\ep^3(\log k)^2}{k}
    \leq\frac{\ep^4}{16}.
\end{aligned}
$$
Taking fourth roots shows that the discretization term is at most $\ep/2$. Therefore
$$
    \mathcal W_2(\Law(x_k),\pi)
    \leq\frac\ep2+\frac\ep2
    =\ep.
$$

Finally, the function $k/(\log k)^2$ tends to infinity. The smallest integer satisfying \eqref{eq:complexity:valid:k} is bounded, up to logarithmic factors, by
$$
    T_\ep+\frac{\bar C_0^4T_\ep^3}{\ep^4}.
$$
Since $T_\ep$ grows at most logarithmically in $1/\ep$ when the remaining parameters are fixed, this proves the stated complexity. Under the additional scaling assumptions in Theorem~\ref{thm:final}, the second part of that theorem gives $\bar C_0=\mathcal O(d)$, which yields the $d^4$ dependence in \eqref{under:additional}. This completes the proof.
\end{proof}

\subsection{Proof of Proposition~\ref{prop:const:strong:euler}}\label{sec:prop:const:strong:euler}
\begin{proof}
Fix an interior ball $B(z,r)\subset\mathcal C$, and write
$$
D:=\operatorname{diam}(\mathcal C), \quad S:=\sup_{x\in\mathcal C}\|\sigma(x)I_d\|_{\mathrm F}.
$$
Use the bridge $\hat X$ from \eqref{eqn:hat:x} with $b=0$. The regulator estimates in the proof of Lemma~\ref{lemma:bound:K} use only bounded coefficients and an interior ball. To allow every initial point, expand the squared distance to $z$, rather than to $X_0$. The initial squared distance is then at most $D^2$. The same continuous and discrete energy identities give
$$
r|\hat K|_s\leq D^2/2+|\hat M_s|+S^2s,\qquad r|\bar K|_s\leq D^2/2+|\bar M_s|+\tfrac12 Q_s,
$$
Here, with $\Delta W_{j+1}:=W_{(j+1)\eta}-W_{j\eta}$, the martingale terms are
$$
\begin{aligned}
\hat M_s:=\int_0^s\left\langle\hat X_u-z,
\sqrt2\sigma\left(\bar X_{\lfloor u/\eta\rfloor\eta}\right)\,dW_u\right\rangle,\quad \bar M_s:=\sum_{j=0}^{n-1}\left\langle\bar X_{j\eta}-z,
\sqrt2\sigma\left(\bar X_{j\eta}\right)\Delta W_{j+1}\right\rangle,
\end{aligned}
$$
and 
\[
Q_s:=\sum_{j=0}^{n-1}\left\|\sqrt2\sigma\left(\bar X_{j\eta}\right)\Delta W_{j+1}\right\|^2.
\]
Since the states lie in $\mathcal C$ and the coefficients are bounded, It\^o's isometry and orthogonality of the discrete martingale increments give $\mathbb E|\hat M_s|^2,\mathbb E|\bar M_s|^2\leq2D^2S^2s$. Since $\sigma$ is scalar,
$$
Q_s\leq\frac{2S^2}{d}\sum_{j=0}^{n-1}\|\Delta W_{j+1}\|^2,\qquad
\mathbb E Q_s^2\leq4S^4(s^2+2s\eta/d)\leq12S^4s^2.
$$
Consequently,
\begin{align}
\mathbb E\left[|\hat K|_s^2\right]&\leq
R_{\rm c}(s):=3r^{-2}\left(D^4/4+2D^2S^2s+S^4s^2\right),\nonumber\\
\mathbb E\left[|\bar K|_s^2\right]&\leq
R_{\rm d}(s):=3r^{-2}\left(D^4/4+2D^2S^2s+3S^4s^2\right).
\label{eq:const:regulator:bounds}
\end{align}
These bounds are uniform over $X_0\in\mathcal C$, in particular, they do not require the initial law to avoid the boundary. The Brownian maximum argument of Lemma~\ref{lemma:bound:bm}, applied on the $n$ grid intervals, yields
$$
\mathbb E\left[\max_{0\leq j<n}\sup_{j\eta\leq t\leq(j+1)\eta}
\|W_t-W_{j\eta}\|^2\right]\leq C_{\rm B}d\eta(1+\log n)
$$
for a universal $C_{\rm B}$. Indeed, the one dimensional reflection bound, followed by a union bound over intervals and integration of the tail, gives $C_{\rm B}\eta(1 + \log n)$ per coordinate. This version holds for every $s>0$. Thus,
$$
\mathbb E\left[\sup_{t\leq s}\left\|\bar Y_t-\hat Y_t\right\|^2\right]\leq\beta\eta(1+\log n),\qquad \beta:=2C_{\rm B}S^2.
$$
The convex Skorokhod comparison \eqref{eq:slominski:bridge:comparison} and Cauchy-Schwarz inequality give
$$
\mathbb E\left[\sup_{t\leq s}\left\|\bar X_t-\hat X_t\right\|^2\right]
\leq\beta\eta(1+\log n)+4\sqrt{\beta\eta(1+\log n)}\,V_s,\qquad
V_s:=\sqrt{R_{\rm c}(s)}+\sqrt{R_{\rm d}(s)}.
$$
Since $(1 + \log n)\leq n$, we have $\eta(1+\log n)\leq n\eta=s$. Define
\begin{equation}\label{eq:const:Is}
G_s:=\beta\sqrt s+4\sqrt\beta\,V_s,\qquad
I_s:=\left(1+\sqrt{e^{4\ell_\sigma^2s}-1}\right)\sqrt{G_s}.
\end{equation}
The bridge error is at most $G_s\sqrt{\eta(1+\log n)}$. For $f(t)=\mathbb E\|X_t-\hat X_t\|^2$, the reflection sign already proved in \eqref{Skorokhod:ineqn}, It\^o's formula, and coefficient Lipschitz continuity imply
$$
f(t)\leq4\ell_\sigma^2\int_0^t\left(f(u)+G_s\sqrt{\eta(1+\log n)}\right)\,du,\qquad t\leq s.
$$
Therefore, Gr\"{o}nwall's inequlity gives
\[
f(s)\leq\left(e^{4\ell_\sigma^2s}-1\right)G_s\sqrt{\eta(1+\log n)}.
\]
This formula also holds when $\ell_\sigma=0$, since then $f(s)=0$. The synchronous construction and the $L^2$ triangle inequality now yield
$$
\begin{aligned}
\mathcal W_2\left(\Law(\bar X_s),\Law(X_s)\right)
&\leq\left(\mathbb E\left\|\bar X_s-\hat X_s\right\|^2\right)^{1/2}+\left(f(s)\right)^{1/2}\\
&\leq\left(1+\sqrt{e^{4\ell_\sigma^2s}-1}\right)\sqrt{G_s}[\eta(1+\log n)]^{1/4}\\
&:=I_s[\eta(1+\log(s/\eta))]^{1/4}.
\end{aligned}
$$
The estimates are uniform over the initial point in $\mathcal C$. Conditioning on any common initial law and integrating preserves the same bound, proving~\eqref{eq:const:strong:euler:log-rate}.
\end{proof}

\subsection{Proof of Lemma~\ref{lem:const:semigroup:stability}}\label{sec:const:semigroup:stability}
\begin{proof}
Synchronous coupling, the convex reflection sign, and It\^o's formula give $\mathbb E\|X_t^x-X_t^y\|^2\leq e^{2\ell_\sigma^2t}\|x-y\|^2$. Integrating over a coupling of the initial laws proves \eqref{eq:const:short:stability}. For smooth $u$, integration along the segment in $\mathcal C$ and Cauchy-Schwarz inequality give
$$
|P_su(y)-P_su(x)|\leq\bigl(\mathbb E\|X_s^y-X_s^x\|^2\bigr)^{1/2}\left(\mathbb E\int_0^1\left\|\nabla u\left((1-\theta)X_s^x+\theta X_s^y\right)\right\|^2d\theta\right)^{1/2}.
$$
The mean square bound gives $X_s^y\to X_s^x$ in probability as $y\to x$. Dividing by $\|y-x\|$ and using bounded uniform continuity of $\nabla u$ yields
$$
\limsup_{\substack{y\to x\\y\in\mathcal C,\ y\neq x}}
\frac{|P_su(y)-P_su(x)|}{\|y-x\|}
\leq e^{\ell_\sigma^2s}\bigl[P_s(|\nabla u|^2)(x)\bigr]^{1/2}.
$$
The kernel bound $\kappa_s\leq M_s$ and integration of local slopes along segments yield
\begin{equation}\label{eq:const:H1:smoothing}
\operatorname{Lip}(P_su)\leq e^{\ell_\sigma^2s}\sqrt{M_s}\,
\|\nabla u\|_{L^2(\pi)}.
\end{equation}
Smooth approximation in $H^1(\mathcal C)$ and $\|P_sv\|_\infty\leq\sqrt{M_s}\|v\|_2$ extend \eqref{eq:const:H1:smoothing} to $u\in H^1(\mathcal C)$. Spectral calculus for $-\mathcal L$ from Proposition~\ref{prop:continuous:structure}~(ii) gives, for $\pi(f)=0$ and $v\geq0$,
$$
\mathcal E(P_{s+v}f,P_{s+v}f)
\leq\frac{1}{2es}\|P_vf\|_2^2
\leq\frac{e^{-2g_0v}}{2es}\|f\|_2^2,
$$
because $\sup_{z\geq0}ze^{-2sz}=1/(2es)$ and the spectral gap is at least $g_0$. Using $\mathcal E(u,u)\geq\lambda\|\nabla u\|_2^2$ in \eqref{eq:const:H1:smoothing}, we obtain
$$
\operatorname{Lip}(P_tf)
\leq e^{\ell_\sigma^2s}\sqrt{\frac{M_s}{2es\lambda}}\,
e^{-g_0(t-2s)}\|f\|_2,\qquad t\geq2s.
$$
To pass from test functions to densities, first note that the heat kernel bound makes $\mu P_t$ and $\nu P_t$ have bounded densities for $t\geq s$. Their difference has zero integral. For any mean zero $f\in L^2(\pi)$ with $\|f\|_2\leq1$,
$$
\left|\int_{\mathcal C}f\,d(\mu P_t-\nu P_t)\right|=\left|\int_{\mathcal C}P_tf\,d(\mu-\nu)\right|\leq\operatorname{Lip}(P_tf)\mathcal W_1(\mu,\nu).
$$
Taking the supremum over these test functions gives the $L^2(\pi)$ norm of the density difference. Combining the preceding Lipschitz estimate with $\mathcal W_1(\mu,\nu)\leq\mathcal W_2(\mu,\nu)$ gives
\begin{equation}\label{eq:const:density:stability}
\left\|\frac{d(\mu P_t-\nu P_t)}{d\pi}\right\|_2
\leq e^{\ell_\sigma^2s}\sqrt{\frac{M_s}{2es\lambda}}\,
e^{-g_0(t-2s)}\mathcal W_2(\mu,\nu).
\end{equation}
For probability densities $a,b\in L^2(\pi)$ with $a,b\geq m>0$, solve the weak Neumann equation
$$
\int_{\mathcal{C}}\nabla\psi\cdot\nabla v\,d\pi=\int_{\mathcal{C}}(b-a)v\,d\pi,\qquad \pi(\psi)=0.
$$
On the mean zero subspace of $H^1(\mathcal C)$, the Poincar\'e inequality makes the gradient form coercive, so Lax-Milgram~\cite[Corollary~5.8]{brezis2011functional} gives a unique mean zero solution. Taking $v=\psi$ in its weak equation gives
$$
\|\nabla\psi\|_2^2=\int_{\mathcal C}(b-a)\psi\,d\pi\leq\|b-a\|_2\|\psi\|_2\leq\sqrt{C_P}\|b-a\|_2\|\nabla\psi\|_2.
$$
Dividing when $\|\nabla\psi\|_2>0$, with the zero case immediate, yields $\|\nabla\psi\|_2^2\leq C_P\|a-b\|_2^2$. For the density path $a_\theta=(1-\theta)a+\theta b$, the flux $(d\pi/dx)\nabla\psi$ satisfies the weak continuity equation with zero boundary flux. The Benamou-Brenier action bound~\cite{benamou2000computational} gives
$$
\mathcal W_2^2(a\pi,b\pi)
\leq\int_0^1\int_{\mathcal{C}}\frac{|\nabla\psi|^2}{a_\theta}\,d\pi\,d\theta
\leq\frac{C_P}{m}\|a-b\|_2^2.
$$
This weak formulation requires no differentiability of $d\pi/dx$, see \cite[Section~2]{peyre2018comparison}. For $t\geq s$, the semigroup property gives $m_s\leq\kappa_t\leq M_s$, so the densities of $\mu P_t,\nu P_t$ lie in $L^2(\pi)$ and are bounded below by $m_s$ for arbitrary initial laws. Applying the transport estimate to the densities of $\mu P_t$ and $\nu P_t$, then using~\eqref{eq:const:density:stability}, gives
$$
\begin{aligned}
\mathcal W_2(\mu P_t,\nu P_t)&\leq\sqrt{\frac{C_P}{m_s}}\left\|\frac{d(\mu P_t-\nu P_t)}{d\pi}\right\|_2\\
&\leq e^{\ell_\sigma^2s}\sqrt{\frac{C_PM_s}{2es\lambda m_s}}e^{-g_0(t-2s)}\mathcal W_2(\mu,\nu)\\
&=H_se^{-g_0(t-2s)}\mathcal W_2(\mu,\nu),
\end{aligned}
$$
since $g_0=\lambda/C_P$ and $H_s$ is defined in~\eqref{eq:const:heatkernel:constants}. This proves~\eqref{eq:const:long:stability}.
\end{proof}

\subsection{Proof of Theorem~\ref{thm:U:0:const:strong}}\label{sec:thm:U:0:const:strong}
\begin{proof}
Let $Q_\eta$ denote one projected Euler-Maruyama step and $Q=Q_\eta^n$. Proposition~\ref{prop:const:strong:euler}, conditioned on the initial state, gives for every probability measure $\nu$ on $\mathcal C$
$$
\mathcal W_2(\nu Q,\nu P_s)\leq
\delta_s:=I_s[\eta(1+\log(s/\eta))]^{1/4}.
$$
Set $\rho_i:=\mu_0Q^iP_{(m-i)s}$ for $i=0,\ldots,m$, so that $\rho_0=\mu_0P_{ms}$ and $\rho_m=\mu_0Q^m$. For $i=1,\ldots,m$, write $\nu_i:=\mu_0Q^{i-1}$ and $j=m-i$. The semigroup property gives
$$
\rho_i=(\nu_iQ)P_{js},\qquad
\rho_{i-1}=(\nu_iP_s)P_{js}.
$$
Since $\mathcal W_2(\nu_iQ,\nu_iP_s)\leq\delta_s$, Lemma~\ref{lem:const:semigroup:stability} yields
$$
\mathcal W_2(\rho_i,\rho_{i-1})
\leq\delta_s
\begin{cases}
1,&j=0,\\
e^{\ell_\sigma^2s},&j=1,\\
H_se^{-g_0(j-2)s},&j\geq2.
\end{cases}
$$
For $m\geq3$, the triangle inequality and the change of index $j=m-i$ give
$$
\begin{aligned}
\mathcal W_2\left(\mu_0Q^m,\mu_0P_{ms}\right)&\leq\sum_{i=1}^m\mathcal W_2(\rho_i,\rho_{i-1})\\
&\leq\delta_s\left(1+e^{\ell_\sigma^2s}+H_s\sum_{j=2}^{m-1}e^{-g_0(j-2)s}\right)\\
&\leq\delta_s\left(1+e^{\ell_\sigma^2s}+H_s\sum_{j=2}^{\infty}e^{-g_0(j-2)s}\right)\\
&=\delta_s\left(1+e^{\ell_\sigma^2s}+\frac{H_s}{1-e^{-g_0s}}\right)=B_s\delta_s.
\end{aligned}
$$
For $m=1$, only $j=0$ occurs, while for $m=2$ only $j=0,1$ occur. The same triangle inequality therefore gives
$$
\mathcal W_2\left(\mu_0Q,\mu_0P_s\right)\leq\delta_s\leq B_s\delta_s,\quad
\mathcal W_2\left(\mu_0Q^2,\mu_0P_{2s}\right)\leq\delta_s\left(1+e^{\ell_\sigma^2s}\right)\leq B_s\delta_s.
$$
Thus the estimate holds for every integer $m\geq1$. To identify the numerical law, use the one step transition kernel $Q_\eta$ and the initial law $\Law(\bar X_0)=\mu_0$. For every Borel set $A\subseteq\mathcal C$ and integer $k\geq0$, the Markov property gives
$$
\mathbb P(\bar X_{(k+1)\eta}\in A)=\int_{\mathcal C}Q_\eta(x,A)\,\Law(\bar X_{k\eta})(dx).
$$
Induction yields $\Law(\bar X_{k\eta})=\mu_0Q_\eta^k$. Since $Q=Q_\eta^n$ and $n\eta=s$,
$$
\mu_0Q^m=\mu_0(Q_\eta^n)^m=\mu_0Q_\eta^{nm}=\Law(\bar X_{nm\eta})=\Law(\bar X_{ms}).
$$
Substituting this identity, $T=ms$, the definition of $\delta_s$, and $F_s=B_sI_s$ into the preceding estimate gives
$$
\begin{aligned}
\mathcal W_2\left(\Law(\bar X_T),\mu_0P_T\right)&=\mathcal W_2\left(\mu_0Q^m,\mu_0P_{ms}\right)\\
&\leq B_s\delta_s=B_sI_s\left[\eta\left(1+\log\frac{s}{\eta}\right)\right]^{1/4}
=F_s\left[\eta\left(1+\log\frac{s}{\eta}\right)\right]^{1/4},
\end{aligned}
$$
which proves~\eqref{eq:const:uniform:comparison}.
Under the additional initial $\chi^2$ condition, Corollary~\ref{cor:W:2}, $\alpha_0\geq\lambda$, and $g_0=\lambda/C_P$ give
$$
\begin{aligned}
\mathcal W_2(\Law(\bar X_{ms}),\pi)
&\leq\mathcal W_2(\mu_0Q^m,\mu_0P_{ms}) +\mathcal W_2(\mu_0P_{ms},\pi)\leq B_s\delta_s+A_0e^{-g_0ms/2}.
\end{aligned}
$$
Substituting $F_s=B_sI_s$ and $T=ms$ proves \eqref{eq:const-ref-unconditional-W2}.
\end{proof}

\subsection{Proof of Corollary~\ref{cor:complexity:const}}\label{sec:cor:complexity:const}
\begin{proof}
The chosen lower bound on $m$ gives
$$
A_0e^{-g_0ms/2}\leq\frac{A_0}{1+2A_0/\ep}=\frac{\ep A_0}{\ep+2A_0}\leq\frac\ep2.
$$ 
Since $\eta=s/n$, the fourth power of the discretization term is
$$
F_s^4\eta(1+\log(s/\eta))=F_s^4s\,\frac{(1 + \log n)}n\leq\frac{\ep^4}{16}.
$$
Thus the two terms in \eqref{eq:const-ref-unconditional-W2} sum to at most $\ep$. For fixed $s$, choose $m$ to be the larger of $1$ and the ceiling of its stated lower bound, so $m=\mathcal O(1+\log(1/\ep))$. For sufficiently small $\ep$, the explicit choice
$
n=\left\lceil16F_s^4s\,\ep^{-4}[\log(1/\ep)]^2\right\rceil
$
satisfies the required condition, because $\log n=4\log(1/\ep)+2\log\log(1/\ep)+\mathcal O(1)$ and hence
$
\frac{n}{(1+\log n)(16F_s^4s\,\ep^{-4})}\rightarrow\infty.
$
Thus $n=\widetilde{\mathcal O}(\ep^{-4})$ and $k=mn=\widetilde{\mathcal O}(\ep^{-4})$.
\end{proof}

\subsection{Proof of Corollary~\ref{cor:const:conditioning}}\label{sec:const:conditioning}
\begin{proof}
Set $\tau=1/g_0$ and $u_x=\kappa_\tau(x,\cdot)-1$. Symmetry and the semigroup identity give, for $v\geq0$,
$$
\pi(u_x)=0,\qquad
\|u_x\|_2^2=\kappa_{2\tau}(x,x)-1,\qquad
\kappa_{2\tau+v}(x,y)-1=\langle u_x,P_vu_y\rangle_2.
$$
In particular, the kernel identities follow by expanding $u_x=\kappa_\tau(x,\cdot)-1$ and applying symmetry and the semigroup property:
$$
\begin{aligned}
\|u_x\|_2^2&=\int_{\mathcal C}\kappa_\tau(x,z)\kappa_\tau(z,x)\pi(dz)-1=\kappa_{2\tau}(x,x)-1,\\
\langle u_x,P_vu_y\rangle_2&=\int_{\mathcal C}\int_{\mathcal C}\kappa_\tau(x,z)\kappa_v(z,w)\kappa_\tau(w,y)\pi(dw)\pi(dz)-1=\kappa_{2\tau+v}(x,y)-1, \quad v>0.
\end{aligned}
$$
For $v=0$, the second identity follows directly from $P_0=I$. Since $\pi(u_y)=0$, the spectral gap gives $\|P_vu_y\|_2\leq e^{-g_0v}\|u_y\|_2$. Consequently,
$$
\left|\langle u_x,P_vu_y\rangle_2\right|\leq e^{-g_0v}\sqrt{\kappa_{2\tau}(x,x)-1}\sqrt{\kappa_{2\tau}(y,y)-1}\leq e^{-g_0v}A_{\rm hk}.
$$
Taking the supremum gives
$$
\sup_{x,y}|\kappa_{2\tau+v}(x,y)-1|
\leq \exp(-g_0v)A_{\rm hk}.
$$
At $v=s_*-2/g_0=g_0^{-1}\log\max\{1,2A_{\rm hk}\}$, the right side is at most $1/2$. Thus $m_{s_*}\geq1/2$ and $M_{s_*}\leq3/2$. Using $s_*=a_*/g_0$ in \eqref{eq:const:heatkernel:constants}, we obtain
$$
\begin{aligned}
H_{s_*} &=\exp(\ell_\sigma^2s_*)\sqrt{\frac{M_{s_*}}{2e\,g_0s_*m_{s_*}}}
\leq\exp\!\left(\frac{\ell_\sigma^2}{g_0}a_*\right)\sqrt{\frac{3/2}{2e\,a_*(1/2)}}\\
&=\exp\!\left(\frac{\ell_\sigma^2}{g_0}a_*\right)\sqrt{\frac{3}{2e\,a_*}}
\leq\exp\!\left(\frac{\ell_\sigma^2}{g_0}a_*\right)\sqrt{\frac{3}{2a_*}}.
\end{aligned}
$$
Since $g_0s_*=a_*\geq2$ and $e^{\ell_\sigma^2s_*}\geq1$, the definitions in~\eqref{eq:const:Fs} and~\eqref{eq:const:Is} give
$$
\begin{aligned}
B_{s_*}&\leq e^{\ell_\sigma^2s_*}\left(2+\frac{\sqrt{3/(2ea_*)}}{1-e^{-a_*}}\right)\leq e^{\ell_\sigma^2s_*}\left(2+\frac{\sqrt{3/(4e)}}{1-e^{-2}}\right)\leq3e^{\ell_\sigma^2s_*},\\
I_{s_*}&=\left(1+\sqrt{e^{4\ell_\sigma^2s_*}-1}\right)\sqrt{G_{s_*}}\leq2e^{2\ell_\sigma^2s_*}\sqrt{G_{s_*}}.
\end{aligned}
$$
Under $\ell_\sigma^2/g_0\leq\ell_*$, their product satisfies
$$
F_{s_*}\leq6\exp\left(3\ell_\sigma^2s_*\right)\sqrt{G_{s_*}}
\leq6\exp(6\ell_*)\max\{1,2A_{\rm hk}\}^{3\ell_*}\sqrt{G_{s_*}},
$$
which proves \eqref{eq:const:conditioning:Fs}. Finally, $\operatorname{Lip}(\sigma)\leq\sqrt\Lambda\,\operatorname{Lip}(U)/2$ gives $\ell_\sigma^2\leq d\Lambda\,\operatorname{Lip}(U)^2/4$, proving the stated sufficient condition.
\end{proof}

\subsection{Proof of Lemma~\ref{lemma:propositions of Gaussian smoothing}}
\label{sec:proof:Gaussian:preliminary:bounds}

\begin{proof}
For every $\mu>0$, $g_\mu$ is smooth, $g$ has at most linear growth, while every derivative of the Gaussian kernel is integrable against that growth, locally uniformly in $x$. Differentiation under the convolution integral is therefore valid. Since $g_\mu(x)=\mathbb E[g(x+\mu\xi)]$ and $\xi\sim\mathcal N(0,I_d)$,
$$
\begin{aligned}
    |h_\mu(x)|
    &=\left|g(x)-\mathbb E[g(x+\mu\xi)]\right|\\
    &\leq\mathbb E|g(x)-g(x+\mu\xi)|
    \leq K\mu\mathbb E\|\xi\|
    \leq K\mu\left(\mathbb E\|\xi\|^2\right)^{1/2}
    =K\mu\sqrt d.
\end{aligned}
$$
This proves \eqref{eq:Gaussian:hmu:sup}. For any $x,y\in\mathcal C$, the Lipschitz continuity of $g$ with constant $K$ gives the same Lipschitz bound for $g_\mu$,
\begin{align}
|g_{\mu}(x)-g_{\mu}(y)|
&=\left|\mathbb E\!\left[g(x+\mu\xi)-g(y+\mu\xi)\right]\right|\leq \mathbb E\!\left[
|g(x+\mu\xi)-g(y+\mu\xi)|
\right]\leq K\|x-y\|.
\end{align}
Hence,
\begin{align}
|h_{\mu}(x)-h_{\mu}(y)|
\leq |g(x)-g(y)|+|g_{\mu}(x)-g_{\mu}(y)|\leq 2K\|x-y\|,
\end{align}
which proves~\eqref{eq:Gaussian:hmu:Lipschitz}. Since $g_{\mu}$ is differentiable and Lipschitz with constant $K$, for every unit vector $v\in\mathbb R^d$,
\begin{equation}
|\langle\nabla g_{\mu}(x),v\rangle|
=\lim_{t\to0}
\frac{|g_{\mu}(x+tv)-g_{\mu}(x)|}{|t|}
\leq K.
\end{equation}
Taking the supremum over $\|v\|=1$ gives $\|\nabla g_{\mu}(x)\|\leq K$. Since $0\in\mathcal C\subset B(0,R)$,
\begin{align}
\|\nabla U_{\mu}(x)\|
&\leq \|\nabla f(x)\|+\|\nabla g_{\mu}(x)\|\nonumber\\
&\leq \|\nabla f(0)\|+L_f\|x\|+K
\leq \|\nabla f(0)\|+L_fR+K.
\end{align}
Taking the supremum over $x\in\mathcal C$ proves \eqref{eq:Gaussian:Gmu:explicit}.

It remains to verify the weak convexity and smoothness of $g_\mu$. Fix $x,y\in\mathbb R^d$ and $\lambda\in[0,1]$. Applying the weak convexity inequality with parameter $\rho$ in \eqref{assumption:g:convex:Gaussian:smoothing} to the two points $x+\mu\xi$ and $y+\mu\xi$, we obtain, for every $\xi\in\mathbb R^d$,
\begin{align}
g\!\left(\lambda x+(1-\lambda)y+\mu\xi\right)
&=g\!\left(\lambda(x+\mu\xi)
+(1-\lambda)(y+\mu\xi)\right)\nonumber\\
&\leq \lambda g(x+\mu\xi)+(1-\lambda)g(y+\mu\xi)
+\frac{\rho}{2}\lambda(1-\lambda)
\|(x+\mu\xi)-(y+\mu\xi)\|^2\nonumber\\
&=\lambda g(x+\mu\xi)+(1-\lambda)g(y+\mu\xi)
+\frac{\rho}{2}\lambda(1-\lambda)\|x-y\|^2.
\end{align}
Taking expectations with respect to $\xi$ on both sides yields
\begin{align}
g_{\mu}\!\left(\lambda x+(1-\lambda)y\right)
&\leq \lambda g_{\mu}(x)+(1-\lambda)g_{\mu}(y)
+\frac{\rho}{2}\lambda(1-\lambda)\|x-y\|^2.
\end{align}
This is precisely the weak convexity inequality with parameter $\rho$ in \eqref{assumption:g:convex:Gaussian:smoothing}. Hence $g_{\mu}$ is weakly convex with parameter $\rho$. Differentiating the Gaussian convolution kernel and using $\mathbb E\xi=0$ gives
$$
    \nabla g_\mu(x)
    =\frac1\mu\mathbb E\left[\xi g(x+\mu\xi)\right] =\frac1\mu\mathbb E\left[\xi\left(g(x+\mu\xi)-g(x)\right)\right].
$$
Consequently,
\begin{align}
\|\nabla g_{\mu}(x)-\nabla g_{\mu}(y)\|
&\leq \frac1\mu\mathbb E\!\left[
\|\xi\|\,|g(x+\mu\xi)-g(y+\mu\xi)|
\right]\leq \frac{K}{\mu}\mathbb E\|\xi\|\,\|x-y\|
\leq \frac{K\sqrt d}{\mu}\|x-y\|.
\end{align}
Since $g_\mu$ is differentiable and weakly convex with parameter $\rho$, the function
$$
    x\mapsto g_\mu(x)+\frac{\rho}{2}\|x\|^2
$$
is differentiable and convex. The monotonicity of its gradient gives
$$
\begin{aligned}
0 &\leq
\left\langle
\nabla g_\mu(x)+\rho x-\nabla g_\mu(y)-\rho y,
x-y
\right\rangle=
\langle\nabla g_\mu(x)-\nabla g_\mu(y),x-y\rangle
+\rho\|x-y\|^2.
\end{aligned}
$$
Therefore,
$$
\begin{aligned}
\langle\nabla U_\mu(x)-\nabla U_\mu(y),x-y\rangle
&=
\langle\nabla f(x)-\nabla f(y),x-y\rangle
+\langle\nabla g_\mu(x)-\nabla g_\mu(y),x-y\rangle\\
&\geq(m_f-\rho)\|x-y\|^2.
\end{aligned}
$$
The two gradient Lipschitz bounds also give
$$
\begin{aligned}
\|\nabla U_\mu(x)-\nabla U_\mu(y)\|
&\leq
\|\nabla f(x)-\nabla f(y)\|
+\|\nabla g_\mu(x)-\nabla g_\mu(y)\|\leq
\left(L_f+\frac{K\sqrt d}{\mu}\right)\|x-y\|.
\end{aligned}
$$
The proof is complete.
\end{proof}
\subsection{Proof of Lemma~\ref{lemma:Gaussian:sufficient:conditions}}
\label{sec:proof:Gaussian:sufficient:conditions}

\begin{proof}
Fix $x,y\in\mathcal C$ with $x\neq y$ and define
$$
    r_\mu(x) :=\nabla U_\mu(x)(e^{h_\mu(x)}-1).
$$
Then
\begin{equation}\label{eq:Gaussian:b:perturbation}
    b(x)=-\nabla U_\mu(x)-r_\mu(x).
\end{equation}
By the definition of $\varepsilon_{b,\mu}$,
$$
    \|r_\mu(x)-r_\mu(y)\|
    \leq
    \varepsilon_{b,\mu}\|x-y\|.
$$
Since $U_\mu$ is strongly convex with constant $m_f-\rho$,
\begin{align}
    \langle b(x)-b(y),x-y\rangle
    &=
    -\langle\nabla U_\mu(x)-\nabla U_\mu(y),x-y\rangle
    -\langle r_\mu(x)-r_\mu(y),x-y\rangle
    \nonumber\\
    &\leq
    -(m_f-\rho)\|x-y\|^2
    +\varepsilon_{b,\mu}\|x-y\|^2
    =
    -m_\mu\|x-y\|^2.
    \label{eq:Gaussian:b:dissipativity}
\end{align}
By Lemma~\ref{lemma:propositions of Gaussian smoothing},
$$
    \left\|\nabla U_\mu(x)-\nabla U_\mu(y)\right\|
    \leq
    \left(L_f+\frac{K\sqrt d}{\mu}\right)\|x-y\|.
$$
Combining this estimate with \eqref{eq:Gaussian:b:perturbation} gives
\begin{align}
    \|b(x)-b(y)\|
    &\leq
    \|\nabla U_\mu(x)-\nabla U_\mu(y)\|
    +\|r_\mu(x)-r_\mu(y)\|
    \nonumber\\
    &\leq
    \left(
    L_f+\frac{K\sqrt d}{\mu}
    +\varepsilon_{b,\mu}
    \right)\|x-y\|
   =M_\mu\|x-y\|.
    \label{eq:Gaussian:b:Lipschitz}
\end{align}
Moreover,
\begin{align}
    \left\|\sigma(x)I_d-\sigma(y)I_d\right\|_{\mathrm F}
    =
    \sqrt d\,
    \left|e^{h_\mu(x)/2}-e^{h_\mu(y)/2}\right|
    \leq
    \sqrt d\,\varepsilon_{\sigma,\mu}\|x-y\|
    =
    \sqrt{\alpha_\mu}\|x-y\|.
    \label{eq:Gaussian:sigma:Lipschitz}
\end{align}
It remains to verify the required relations among the constants. From \eqref{eq:Gaussian:perturbation:small},
$$
    \varepsilon_{b,\mu}
    +2\alpha_\mu < m_f-\rho-\frac12.
$$
Therefore
$$
\begin{aligned}
    m_\mu &=m_f-\rho-\varepsilon_{b,\mu}>\frac12+2\alpha_\mu >\alpha_\mu,
\end{aligned}
$$
so $m_\mu>0$ and $0\leq\alpha_\mu<m_\mu$. Hence \eqref{eq:Gaussian:b:dissipativity}, \eqref{eq:Gaussian:b:Lipschitz}, and \eqref{eq:Gaussian:sigma:Lipschitz} verify Assumptions~\ref{asmpt:strongly:convex} and \ref{asmpt:smooth}. Finally,
$$
    2m_\mu-1-4\alpha_\mu = 2\left( m_f-\rho-\frac12 -\varepsilon_{b,\mu} -2\alpha_\mu
    \right)
    = 2\left( m_f-\rho-\frac12 -\varepsilon_{b,\mu} -2d\varepsilon_{\sigma,\mu}^2
    \right)
    >0.
$$
Thus, $2m_\mu>1+4\alpha_\mu$, and all coefficient assumptions of Theorem~\ref{global:error:2} hold.
\end{proof}
\subsection{Verification of the explicit Gaussian smoothing criterion}
\label{sec:Gaussian:explicit:criterion}
Recall that
$$
    H_\mu=K\mu\sqrt d,
    \qquad
    L_\mu=L_f+\frac{K\sqrt d}{\mu},
    \qquad
    G=\|\nabla f(0)\|+L_fR+K.
$$
Lemma~\ref{lemma:propositions of Gaussian smoothing} gives
$$
\begin{aligned}
    &|h_\mu(x)|\leq H_\mu,
    \qquad
    |h_\mu(x)-h_\mu(y)|
    \leq2K\|x-y\|,\\
    &\left\|\nabla U_\mu(x)-\nabla U_\mu(y)\right\|
    \leq L_\mu\|x-y\|,
    \qquad
    \|\nabla U_\mu(x)\|\leq G.
\end{aligned}
$$
The mean value theorem gives
$$
\begin{aligned}
    \left|e^{h_\mu(x)/2}-e^{h_\mu(y)/2}\right|
    \leq
    \frac12e^{H_\mu/2}
    |h_\mu(x)-h_\mu(y)|
    \leq
    Ke^{H_\mu/2}\|x-y\|.
\end{aligned}
$$
Hence,
\begin{equation}
    \varepsilon_{\sigma,\mu}
    \leq Ke^{H_\mu/2}.
\end{equation}

For $r_\mu(x)=\nabla U_\mu(x)\left(e^{h_\mu(x)}-1\right)$,
$$
\begin{aligned}
    \|r_\mu(x)-r_\mu(y)\|
    &\leq
    \left\|\nabla U_\mu(x)-\nabla U_\mu(y)\right\|
    \left|e^{h_\mu(x)}-1\right|+
    \|\nabla U_\mu(y)\|
    \left|e^{h_\mu(x)}-e^{h_\mu(y)}\right|\\
    &\leq L_\mu\left(e^{H_\mu}-1\right)\|x-y\|
    +Ge^{H_\mu}|h_\mu(x)-h_\mu(y)|\\
    &\leq
    \left[
    L_\mu\left(e^{H_\mu}-1\right) +2KGe^{H_\mu}
    \right]\|x-y\|.
\end{aligned}
$$
Therefore,
\begin{equation}
    \varepsilon_{b,\mu}
    \leq
    L_\mu\left(e^{H_\mu}-1\right)+2KGe^{H_\mu}.
\end{equation}
Combining the two bounds yields
$$
\begin{aligned}
    \varepsilon_{b,\mu}
    +2d\varepsilon_{\sigma,\mu}^2 &\leq L_\mu\left(e^{H_\mu}-1\right) +2KGe^{H_\mu} +2dK^2e^{H_\mu}.
\end{aligned}
$$
Thus \eqref{eq:Gaussian:perturbation:explicit} implies \eqref{eq:Gaussian:perturbation:small}. As an alternative sufficient condition, if
\begin{equation}\label{eq:Gaussian:perturbation:asymptotic}
    \limsup_{\mu\rightarrow0}
    \left(
    \varepsilon_{b,\mu}
    +2d\varepsilon_{\sigma,\mu}^2
    \right)
    <m_f-\rho-\frac12,
\end{equation}
then \eqref{eq:Gaussian:perturbation:small} holds for every sufficiently small $\mu>0$. 

\subsection{Proof of Lemma~\ref{lemma:Gaussian:MC:error}}
\label{sec:proof:Gaussian:MC:error}

\begin{proof}
Throughout this proof, $\mathbb E$ denotes expectation with respect to all random variables in the synchronous coupling. For fixed $x\in\mathcal C$, let $\mathbb E_{p_N,k}$ denote expectation only over the $N$ samples $\{\xi_{i,k}\}_{i=1}^N$ used to approximate the Gaussian smoothing function at iteration $k$, and let $\mathbb E_{g_N,k}$ denote expectation only over the $N$ samples $\{\hat\xi_{i,k}\}_{i=1}^N$ used to approximate its gradient. All other quantities are regarded as constants under these two expectations.

Before projection, the difference between the two synchronously coupled algorithms is
\begin{align}
&\tilde x_k-x_k
+\eta\bigl(\tilde b_k(\tilde x_k)-b(x_k)\bigr)
+\sqrt{2\eta}\bigl(\tilde\sigma_k(\tilde x_k)-\sigma(x_k)\bigr)
\xi_{k+1}.
\end{align}
By the nonexpansiveness of the orthogonal projection onto the convex set $\mathcal C$, we obtain
\begin{align}
\mathbb E\|\tilde x_{k+1}-x_{k+1}\|^2
&\leq\mathbb E\Bigl\|\tilde x_k-x_k
+\eta\bigl(\tilde b_k(\tilde x_k)-b(x_k)\bigr)
+\sqrt{2\eta}\bigl(\tilde\sigma_k(\tilde x_k)-\sigma(x_k)\bigr)
\xi_{k+1}\Bigr\|^2
\nonumber\\
&=\mathbb E\Bigl\|\tilde x_k-x_k
+\eta\bigl(b(\tilde x_k)-b(x_k)\bigr)
+\eta\bigl(\tilde b_k(\tilde x_k)-b(\tilde x_k)\bigr)\Bigr\|^2
\nonumber\\
&\quad+2\eta d\,\mathbb E\left[\Bigl|
\sigma(\tilde x_k)-\sigma(x_k)
+\tilde\sigma_k(\tilde x_k)-\sigma(\tilde x_k)\Bigr|^2\right]\nonumber\\
&\leq(1+\eta)\mathbb E
\left[\|\tilde x_k-x_k+\eta(b(\tilde x_k)-b(x_k))\|^2\right]
+\eta(1+\eta)
\mathbb E\left[\left\|\tilde b_k(\tilde x_k)-b(\tilde x_k)\right\|^2\right]\nonumber\\
&\qquad\qquad+4\eta d\,\mathbb E\left[\left|\sigma(\tilde x_k)-\sigma(x_k)\right|^2\right]
+4\eta d\,\mathbb E\left[|\tilde\sigma_k(\tilde x_k)
-\sigma(\tilde x_k)|^2\right],
\label{eq:Gaussian:coupling:first}
\end{align}
where the equality uses that $\xi_{k+1}$ is independent of all quantities at iteration $k$, $\mathbb E[\xi_{k+1}]=0$, and $\mathbb E\|\xi_{k+1}\|^2=d$. The last inequality is obtained by applying $2\langle a,b\rangle\leq\eta\|a\|^2+\eta^{-1}\|b\|^2$ to the first term and applying $|a+b|^2\leq2|a|^2+2|b|^2$ to the second term.

The first term in \eqref{eq:Gaussian:coupling:first} is controlled by the dissipativity and Lipschitz bounds for the ideal drift $b$ in \eqref{eq:Gaussian:b:dissipativity} and \eqref{eq:Gaussian:b:Lipschitz}:
\begin{align}
\|\tilde x_k-x_k+\eta(b(\tilde x_k)-b(x_k))\|^2&=\|\tilde x_k-x_k\|^2
+2\eta\langle\tilde x_k-x_k,b(\tilde x_k)-b(x_k)\rangle
+\eta^2\|b(\tilde x_k)-b(x_k)\|^2\nonumber\\
&\leq(1-2m_\mu\eta+M_\mu^2\eta^2)
\|\tilde x_k-x_k\|^2.
\label{eq:Gaussian:ideal:Euler:contraction}
\end{align}
The third term is controlled directly by the coefficient condition for $\sigma$:
\begin{equation}
d|\sigma(\tilde x_k)-\sigma(x_k)|^2
=\|\sigma(\tilde x_k)I_d-\sigma(x_k)I_d\|_{\mathrm F}^2
\leq\alpha_\mu\|\tilde x_k-x_k\|^2.
\label{eq:Gaussian:ideal:sigma:difference}
\end{equation}
The remaining terms are required to bound the approximation errors from the Monte Carlo simulation. For any fixed $x\in\mathcal C$, decompose the drift and diffusion errors as
\begin{align}
\tilde b_k(x)-b(x)
&=-\bigl(\widetilde{\nabla} U_{\mu,k}(x)-\nabla U_\mu(x)\bigr)
e^{h_\mu(x)-\left(\tilde U_{\mu,k}(x)-U_\mu(x)\right)}\nonumber\\
&\qquad\qquad\qquad\qquad-\nabla U_\mu(x)e^{h_\mu(x)}
\left(e^{-\left(\tilde U_{\mu,k}(x)-U_\mu(x)\right)}-1\right),
\label{eq:Gaussian:b:MC:difference}\\
\tilde\sigma_k(x)-\sigma(x)
&=e^{h_\mu(x)/2}
\left(e^{-(\tilde U_{\mu,k}(x)-U_\mu(x))/2}-1\right).
\label{eq:Gaussian:sigma:MC:difference}
\end{align}
Thus, the first summand in~\eqref{eq:Gaussian:b:MC:difference} requires a second moment bound for the gradient estimator and an exponential moment bound for the function estimator. The second summand in~\eqref{eq:Gaussian:b:MC:difference} and the error in~\eqref{eq:Gaussian:sigma:MC:difference} require only exponential moment bounds for the function estimator. We derive these bounds in that order. For the gradient estimator, the Gaussian integration by parts identity and $\mathbb E_{g_N,k}[\hat\xi_{i,k}]=0$ give
\begin{align}
\mathbb E_{g_N,k}\left[\widetilde{\nabla} U_{\mu,k}(x)\right]
&=\nabla f(x)+\frac{1}{\mu N}\sum_{i=1}^N
\mathbb E_{g_N,k}\left[
\hat\xi_{i,k}\left(g\left(x+\mu\hat\xi_{i,k}\right)-g(x)\right)\right]\nonumber\\
&=\nabla f(x)+\nabla g_\mu(x)=\nabla U_\mu(x).
\end{align}
Independence within the simulation therefore gives
\begin{align}
\mathbb E_{g_N,k}\left\|\widetilde{\nabla} U_{\mu,k}(x)-\nabla U_\mu(x)\right\|^2
&\leq\frac1N\mathbb E_{g_N,k}\left\|
\frac{\hat\xi_{1,k}}{\mu}
\left(g\left(x+\mu\hat\xi_{1,k}\right)-g(x)\right)\right\|^2\nonumber\\
&\leq\frac{K^2}{N}\mathbb E_{g_N,k}\|\hat\xi_{1,k}\|^4
=\frac{K^2d(d+2)}{N}.
\label{eq:Gaussian:gradient:MC:bound}
\end{align}
Here the variance of the average of $N$ independent centered vectors is $1/N$ times the variance of one vector. The Lipschitz bound $|g(x+\mu\hat\xi)-g(x)|\leq K\mu\|\hat\xi\|$ and the Gaussian moment $\mathbb E\|\hat\xi\|^4=d(d+2)$ give the remaining steps. For the function estimator, define
$$
    F_x(z_1,\ldots,z_N):=\frac1N\sum_{i=1}^N g(x+\mu z_i).
$$
For $z=(z_1,\ldots,z_N)$ and $z'=(z_1',\ldots,z_N')$ in $\mathbb R^{dN}$,
$$
    |F_x(z)-F_x(z')|
    \leq\frac{K\mu}{N}\sum_{i=1}^N\|z_i-z_i'\|
    \leq\frac{K\mu}{\sqrt N}
    \left[\sum_{i=1}^N\|z_i-z_i'\|^2\right]^{1/2}.
$$
Thus, $F_x$ is $K\mu/\sqrt N$ Lipschitz on $\mathbb R^{dN}$. Since $(\xi_{1,k},\ldots,\xi_{N,k})$ is standard Gaussian in $\mathbb R^{dN}$ and
$$
    \tilde U_{\mu,k}(x)-U_\mu(x)
    =F_x(\xi_{1,k},\ldots,\xi_{N,k}) -\mathbb E_{p_N,k}\left[F_x(\xi_{1,k},\ldots,\xi_{N,k})\right],
$$
the Gaussian concentration inequality gives
\begin{equation}
\mathbb E_{p_N,k}\left[\exp\!\left(\lambda
\bigl(\tilde U_{\mu,k}(x)-U_\mu(x)\bigr)\right)\right]
\leq\exp\!\left(\frac{\lambda^2K^2\mu^2}{2N}\right),
\qquad \forall\lambda\in\mathbb R.
\label{eq:Gaussian:potential:mgf}
\end{equation}
Unbiasedness of the Monte Carlo simulation and Jensen's inequality give, for every $a>0$,
\begin{equation}
\mathbb E_{p_N,k}\left[
e^{-a(\tilde U_{\mu,k}(x)-U_\mu(x))}\right]\geq1.
\label{eq:Gaussian:potential:Jensen}
\end{equation}
Combining \eqref{eq:Gaussian:potential:mgf} with $\lambda=-2$ and \eqref{eq:Gaussian:potential:Jensen} gives
\begin{align}
\mathbb E_{p_N,k}\left |e^{-(\tilde U_{\mu,k}(x)-U_\mu(x))}-1\right |^2
&=\mathbb E_{p_N,k}\left[e^{-2(\tilde U_{\mu,k}(x)-U_\mu(x))}\right]
-2\mathbb E_{p_N,k}\left[e^{-(\tilde U_{\mu,k}(x)-U_\mu(x))}\right]+1\nonumber\\
&\leq e^{2K^2\mu^2/N}-1,
\label{eq:Gaussian:potential:exp:error}
\end{align}
Returning to the expectation of the drift error \eqref{eq:Gaussian:b:MC:difference}, the unbiasedness of the Monte Carlo gradient estimator makes the cross term vanish after taking $\mathbb E_{g_N,k}$. Taking $\mathbb E_{p_N,k}$ then gives
\begin{align}
& \mathbb E_{p_N,k}\mathbb E_{g_N,k}
\left\|\tilde b_k(x)-b(x)\right\|^2
\nonumber\\
&\qquad=e^{2h_\mu(x)}
\mathbb E_{g_N,k}
\left\|\widetilde{\nabla} U_{\mu,k}(x)-\nabla U_\mu(x)\right\|^2
\mathbb E_{p_N,k}\left[e^{-2(\tilde U_{\mu,k}(x)-U_\mu(x))}\right]\nonumber\\
&\qquad\qquad\qquad+\left\|\nabla U_\mu(x)\right\|^2e^{2h_\mu(x)}
\mathbb E_{p_N,k}
\left[\left(e^{-(\tilde U_{\mu,k}(x)-U_\mu(x))}-1\right)^2\right]
\nonumber\\
&\qquad\leq \frac{K^2d(d+2)}{N}
e^{2K\mu\sqrt d+2K^2\mu^2/N}\nonumber\\
&\qquad\qquad\qquad+\left(\|\nabla f(0)\|+L_fR+K\right)^2
e^{2K\mu\sqrt d}\left(e^{2K^2\mu^2/N}-1\right)
=:A_{\mu,N}.
\end{align}
Here in the inequality, the first term uses \eqref{eq:Gaussian:hmu:sup}, \eqref{eq:Gaussian:gradient:MC:bound} and \eqref{eq:Gaussian:potential:mgf} with $\lambda=-2$. The second term uses \eqref{eq:Gaussian:Gmu:explicit}, \eqref{eq:Gaussian:potential:exp:error}, and \eqref{eq:Gaussian:hmu:sup}. For the diffusion error in \eqref{eq:Gaussian:sigma:MC:difference}, \eqref{eq:Gaussian:potential:mgf} with $\lambda=-1$ and \eqref{eq:Gaussian:potential:Jensen} with $a=1/2$ give
$$
\begin{aligned}
\mathbb E_{p_N,k}
\left[\left(e^{-(\tilde U_{\mu,k}(x)-U_\mu(x))/2}-1\right)^2\right]
&=\mathbb E_{p_N,k}\left[ e^{-(\tilde U_{\mu,k}(x)-U_\mu(x))}\right] -2\mathbb E_{p_N,k}
\left[e^{-(\tilde U_{\mu,k}(x)-U_\mu(x))/2}\right]+1\\
&\leq e^{K^2\mu^2/(2N)}-1.
\end{aligned}
$$
Combining this estimate with \eqref{eq:Gaussian:hmu:sup} yields
\begin{align}
\mathbb E_{p_N,k}\left[|\tilde\sigma_k(x)-\sigma(x)|^2\right]
&=e^{h_\mu(x)}
\mathbb E_{p_N,k}
\left[\left(e^{-(\tilde U_{\mu,k}(x)-U_\mu(x))/2}-1\right)^2\right]
\nonumber\\
&\leq e^{K\mu\sqrt d}
\left(e^{K^2\mu^2/(2N)}-1\right)
=B_{\mu,N}.
\end{align}
Since $\tilde x_k$ is independent of all the Monte Carlo samples at iteration $k$, iterated expectation gives
\begin{align}
\mathbb E\left\|\tilde b_k(\tilde x_k)-b(\tilde x_k)\right\|^2
&=\mathbb E\left[
\left.\mathbb E_{p_N,k}\mathbb E_{g_N,k}
\|\tilde b_k(x)-b(x)\|^2\right |{x=\tilde x_k}\right]
\leq A_{\mu,N},\nonumber\\
\mathbb E\left|\tilde\sigma_k(\tilde x_k)-\sigma(\tilde x_k)\right|^2
&=\mathbb E\left[
\left.\mathbb E_{p_N,k}|\tilde\sigma_k(x)-\sigma(x)|^2
\right |{x=\tilde x_k}\right]
\leq B_{\mu,N}.
\label{eq:Gaussian:MC:random:bounds}
\end{align}
Substituting~\eqref{eq:Gaussian:ideal:Euler:contraction}, \eqref{eq:Gaussian:ideal:sigma:difference}, and \eqref{eq:Gaussian:MC:random:bounds} into \eqref{eq:Gaussian:coupling:first} closes the one step update:
\begin{align}
\mathbb E\|\tilde x_{k+1}-x_{k+1}\|^2
&\leq\left[(1+\eta)\left(1-2m_\mu\eta+M_\mu^2\eta^2\right)
+4\alpha_\mu\eta\right]\mathbb E\|\tilde x_k-x_k\|^2\nonumber\\
&\qquad\qquad+\eta\left((1+\eta)A_{\mu,N}+4dB_{\mu,N}\right).
\label{eq:Gaussian:coupling:recursion}
\end{align}

We finally solve the resulting recursion. Under \eqref{eq:Gaussian:eta:conservative}, the multiplicative coefficient above satisfies
\begin{align}
&(1+\eta)(1-2m_\mu\eta+M_\mu^2\eta^2)
+4\alpha_\mu\eta
\nonumber
\\
&\qquad=1-\left(2m_\mu-1-4\alpha_\mu\right)\eta
+(M_\mu^2-2m_\mu)\eta^2+M_\mu^2\eta^3\nonumber\\
&\qquad\leq1-(2m_\mu-1-4\alpha_\mu)\eta
+\left(\max\left\{0,M_\mu^2-2m_\mu\right\}+M_\mu^2\right)\eta^2\nonumber\\
&\qquad\leq1-\frac{2m_\mu-1-4\alpha_\mu}{2}\eta.
\end{align}
Since $m_\mu\leq M_\mu$,
\begin{equation}
1-2m_\mu\eta+M_\mu^2\eta^2
=(1-m_\mu\eta)^2+(M_\mu^2-m_\mu^2)\eta^2\geq0.
\end{equation}
We have $1-(2m_\mu-1-4\alpha_\mu)\eta/2$ lies in $[0,1)$. Starting from $\tilde x_0=x_0$, iterate~\eqref{eq:Gaussian:coupling:recursion} to obtain
$$
\begin{aligned}
\mathbb E\|\tilde x_k-x_k\|^2
&\leq\eta\bigl((1+\eta)A_{\mu,N}+4dB_{\mu,N}\bigr)\sum_{j=0}^{k-1}\left(1-\frac{2m_\mu-1-4\alpha_\mu}{2}\eta\right)^j\\
&=\frac{2\bigl((1+\eta)A_{\mu,N}+4dB_{\mu,N}\bigr)}{2m_\mu-1-4\alpha_\mu}\left[1-\left(1-\frac{2m_\mu-1-4\alpha_\mu}{2}\eta\right)^k\right]\\
&\leq\frac{2\bigl((1+\eta)A_{\mu,N}+4dB_{\mu,N}\bigr)}{2m_\mu-1-4\alpha_\mu}.
\end{aligned}
$$
For $k=0$, the same final bound holds because the error is zero. Finally, the synchronous construction above is itself a coupling of $\nu_k=\Law(x_k)$ and $\tilde\nu_k=\Law(\tilde x_k)$. Hence,
$$
\mathcal W_2^2(\nu_k,\tilde\nu_k)
\leq
\mathbb E\|x_k-\tilde x_k\|^2
\leq
\frac{2}{2m_\mu-1-4\alpha_\mu}
\left((1+\eta)A_{\mu,N}+4dB_{\mu,N}\right),
$$
which is exactly \eqref{eq:Gaussian:MC:W2}. The proof is complete.
\end{proof}

\subsection{Proof of Theorem~\ref{theorem:Gaussian:MC:convergence}}
\label{sec:proof:Gaussian:MC:convergence}

\begin{proof}
Fix $\mu$ satisfying the hypotheses. By~\eqref{eq:Gaussian:m:M:alpha} and~\eqref{eq:Gaussian:perturbation:small},
$$
2m_\mu-1-4\alpha_\mu=2\left(m_f-\rho-\frac12-\varepsilon_{b,\mu}-2d\varepsilon_{\sigma,\mu}^2\right)>0.
$$
Lemma~\ref{lemma:Gaussian:sufficient:conditions} therefore verifies Assumptions~\ref{asmpt:strongly:convex} and~\ref{asmpt:smooth} for the exact coefficients in~\eqref{eq:Gaussian:exact:coefficients}, with $(m,M,\alpha)=(m_\mu,M_\mu,\alpha_\mu)$. The present theorem supplies the domain regularity, potential regularity, Poincar\'e inequality, and initial law assumptions required for continuous time convergence. In particular, Proposition~\ref{prop:continuous:structure} supplies the transition semigroup and heat kernel properties. The common initial law also satisfies the interior support condition in Theorem~\ref{thm:final}. Finally, \eqref{eq:Gaussian:eta:conservative} implies $0<\eta\leq1$, and $k\eta\geq e$ is assumed. Thus Theorem~\ref{thm:final} applies to the ideal iterates $x_k$.

Now substituting $U_0=U_\mu$ and the constants in~\eqref{eq:Gaussian:Cbar} into~\eqref{eq:thm:final:explicit} gives
$$
\mathcal W_2(\nu_k,\pi)\leq 2^{3/4}R\left(\chi^2(\nu_0\Vert\pi)\right)^{1/4}e^{-\alpha_{0,\mu}k\eta/(2C_P)}+\bar C_{0,\mu}\eta^{1/4}(k\eta)^{1/2}(\log k)^{1/2}.
$$
Here the first term is the continuous time mixing error and the second is the discretization error. For the same initial law and stepsize, Lemma~\ref{lemma:Gaussian:MC:error} applies to the ideal and Monte Carlo iterates. Taking square roots in~\eqref{eq:Gaussian:MC:W2} yields
$$
\mathcal W_2(\tilde\nu_k,\nu_k)\leq\left(\frac{2\bigl((1+\eta)A_{\mu,N}+4dB_{\mu,N}\bigr)}{2m_\mu-1-4\alpha_\mu}\right)^{1/2}.
$$
Consequently, the triangle inequality gives
$$
\begin{aligned}
\mathcal W_2(\tilde\nu_k,\pi)&\leq\mathcal W_2(\tilde\nu_k,\nu_k)+\mathcal W_2(\nu_k,\pi)\\
&\leq 2^{3/4}R\left(\chi^2(\nu_0\Vert\pi)\right)^{1/4}e^{-\alpha_{0,\mu}k\eta/(2C_P)}+\bar C_{0,\mu}\eta^{1/4}(k\eta)^{1/2}(\log k)^{1/2}\\
&\qquad+\left(\frac{2\bigl((1+\eta)A_{\mu,N}+4dB_{\mu,N}\bigr)}{2m_\mu-1-4\alpha_\mu}\right)^{1/2},
\end{aligned}
$$
which proves~\eqref{eq:Gaussian:final:W2}.

Fix $\ep>0$ and keep $\mu$ fixed. Since $\alpha_{0,\mu}\geq e^{-K\mu\sqrt d}>0$ by~\eqref{eq:Gaussian:Cbar} and the initial $\chi^2$ divergence is finite, choose a finite $T\geq e$ such that
$$
2^{3/4}R\left(\chi^2(\nu_0\Vert\pi)\right)^{1/4}e^{-\alpha_{0,\mu}T/(2C_P)}\leq\frac\ep3.
$$
For each integer $k$, set $\eta=T/k$. With $T$ fixed, the discretization term becomes
\begin{align}\label{discretization:term}
\bar C_{0,\mu}\eta^{1/4}(k\eta)^{1/2}(\log k)^{1/2}=\bar C_{0,\mu}T^{3/4}\left(\frac{(\log k)^2}{k}\right)^{1/4}\rightarrow0.
\end{align}
The upper bound on $\eta$ in~\eqref{eq:Gaussian:eta:conservative} is strictly positive and independent of $k$ and $N$. Since $T/k\to0$, one can choose an integer $k$ large enough that this stepsize condition holds and the discretization term in \eqref{discretization:term} is at most $\ep/3$. This choice preserves $k\eta=T\geq e$.

It remains to choose $N$ with $T,k,\eta$ now fixed. To verify the required decay directly from~\eqref{eq:Gaussian:A} and~\eqref{eq:Gaussian:B}, use $e^t-1=\int_0^t e^s\,ds\leq te^t$ for $t\geq0$. Since $N\geq1$, this gives
$$
\begin{aligned}
&A_{\mu,N}\leq\frac{e^{2K\mu\sqrt d+2K^2\mu^2}}{N}\left[K^2d(d+2)+2K^2\mu^2\left(\|\nabla f(0)\|+L_fR+K\right)^2\right],
\\
&B_{\mu,N}\leq\frac{K^2\mu^2}{2N}e^{K\mu\sqrt d+K^2\mu^2/2}.
\end{aligned}
$$
Thus both quantities are $\mathcal{O}(N^{-1})$ for the fixed model and smoothing level, and the Monte Carlo error is $\mathcal{O}(N^{-1/2})$. Choose an integer $N$ large enough that
$$
(1+\eta)A_{\mu,N}+4dB_{\mu,N}\leq\frac{\ep^2}{18}\left(2m_\mu-1-4\alpha_\mu\right).
$$
The square root term in~\eqref{eq:Gaussian:final:W2} is then at most $\ep/3$. Each of the three errors is at most $\ep/3$, so their sum is at most $\ep$, the proof is complete.
\end{proof}

\subsection{Proof of Lemma~\ref{lem:GS:sharp:coefficients}}
\label{sec:proof:GS:sharp:coefficients}
\begin{proof}
We first bound the normalized scaling factor, then the derivatives of the drift and diffusion coefficients in~\eqref{eq:GS:sharp:coefficients}. Since $g_\delta$ is convex and $K$-Lipschitz, Jensen's inequality and $\mathbb E|G|\leq1$ give
$$
-K\mu\leq U_\delta-U_{\delta,\mu}\leq0,
\qquad
e^{-K\mu}\leq a_\delta^{\rm n}(x) =\frac{e^{U_\delta(x)-U_{\delta,\mu}(x)}}{\int_{\mathcal C}e^{U_\delta-U_{\delta,\mu}}\,d\pi_\delta}
\leq e^{K\mu}.
$$
The quadratic parts cancel in $U_\delta'-U_{\delta,\mu}'$, so
$$
|U_{\delta,\mu}'|\leq\kappa R+K,\qquad |U_\delta'-U_{\delta,\mu}'|\leq2K,\qquad
\kappa\leq U_{\delta,\mu}''\leq\kappa+K/\mu.
$$
The last bound follows from convexity and Gaussian integration by parts:
$U_{\delta,\mu}''-\kappa=\mu^{-1}\mathbb E[Gg_\delta'(x+\mu G)]$, with $|g_\delta'|\leq K$.

For the drift, differentiating $b_\delta^{\rm n}=-a_\delta^{\rm n}U_{\delta,\mu}'$ gives
$$
(b_\delta^{\rm n})' =-a_\delta^{\rm n}\bigl[U_{\delta,\mu}'' +(U_\delta'-U_{\delta,\mu}')U_{\delta,\mu}'\bigr].
$$
Since~\eqref{eq:GS:sharp:admissibility} implies $\kappa-2K(\kappa R+K)>0$, the preceding bounds yield
\begin{align*}
(b_\delta^{\rm n})'\leq-e^{-K\mu}\bigl[\kappa-2K(\kappa R+K)\bigr],\qquad
|(b_\delta^{\rm n})'|\leq e^{K\mu}\bigl[\kappa+K/\mu+2K(\kappa R+K)\bigr].
\end{align*}
For the diffusion, differentiating $\sigma_\delta^{\rm n}=\sqrt{a_\delta^{\rm n}}$ gives directly
$$
|(\sigma_\delta^{\rm n})'|^2 =\tfrac14a_\delta^{\rm n}|U_\delta'-U_{\delta,\mu}'|^2
\leq e^{K\mu}K^2.
$$
The fundamental theorem of calculus gives
$$
\begin{aligned}
&(b_\delta^{\rm n}(x)-b_\delta^{\rm n}(y))(x-y)=(x-y)^2\int_0^1(b_\delta^{\rm n})'(y+t(x-y))\,dt\leq-e^{-K\mu}[\kappa-2K(\kappa R+K)]|x-y|^2,\\
&|b_\delta^{\rm n}(x)-b_\delta^{\rm n}(y)|\leq|x-y|\int_0^1|(b_\delta^{\rm n})'(y+t(x-y))|\,dt\leq e^{K\mu}[\kappa+K/\mu+2K(\kappa R+K)]|x-y|.
\end{aligned}
$$
Similarly, Cauchy-Schwarz inequality on $[0,1]$ gives
$$
|\sigma_\delta^{\rm n}(x)-\sigma_\delta^{\rm n}(y)|^2\leq|x-y|^2\int_0^1|(\sigma_\delta^{\rm n})'(y+t(x-y))|^2\,dt\leq e^{K\mu}K^2|x-y|^2.
$$
These are exactly the three estimates in~\eqref{eq:GS:sharp:bounds}.
\end{proof}

\subsection{Proof of Proposition~\ref{prop:GS:sharp:contraction}}
\label{sec:proof:GS:sharp:contraction}
\begin{proof}
Fix $\delta>0$ and $0<\eta<\eta_c$. For deterministic $x,y\in\mathcal C$, use the same $G\sim\mathcal N(0,1)$ to define
$$
\begin{aligned}
X_1^x&=\mathcal P_{\mathcal C}\!\left(x+\eta b_\delta^{\rm n}(x)
+\sqrt{2\eta}\sigma_\delta^{\rm n}(x)G\right),\quad
X_1^y=\mathcal P_{\mathcal C}\!\left(y+\eta b_\delta^{\rm n}(y) +\sqrt{2\eta}\sigma_\delta^{\rm n}(y)G\right).
\end{aligned}
$$
These variables have marginal laws $Q^{\rm R}_{\delta,\eta}(x,\cdot)$ and $Q^{\rm R}_{\delta,\eta}(y,\cdot)$. Since projection onto $\mathcal C$ is nonexpansive, $\mathbb EG=0$, and $\mathbb EG^2=1$,
\begin{align*}
\mathbb E|X_1^x-X_1^y|^2
&\leq\mathbb E\left|x-y+\eta\bigl(b_\delta^{\rm n}(x)-b_\delta^{\rm n}(y)\bigr)
+\sqrt{2\eta}\bigl(\sigma_\delta^{\rm n}(x)-\sigma_\delta^{\rm n}(y)\bigr)G\right|^2\\
&=|x-y|^2+2\eta(x-y)\bigl(b_\delta^{\rm n}(x)-b_\delta^{\rm n}(y)\bigr)\\
&\qquad+\eta^2|b_\delta^{\rm n}(x)-b_\delta^{\rm n}(y)|^2
+2\eta|\sigma_\delta^{\rm n}(x)-\sigma_\delta^{\rm n}(y)|^2.
\end{align*}
The mixed term involving $G$ vanishes because the other factors are deterministic. Substituting the three bounds in~\eqref{eq:GS:sharp:bounds}, in the same order, gives
\begin{align*}
\mathbb E|X_1^x-X_1^y|^2
&\leq\Bigl\{1-2\eta e^{-K\mu}\bigl[\kappa-2K(\kappa R+K)\bigr]
+\eta^2M_*^2+2\eta e^{K\mu}K^2\Bigr\}|x-y|^2\\
&=\bigl(1-2\gamma_*\eta+M_*^2\eta^2\bigr)|x-y|^2
=q_\eta^2|x-y|^2,
\end{align*}
where the equality uses the definition of $\gamma_*$ in~\eqref{eq:GS:sharp:constants}. That definition and~\eqref{eq:GS:sharp:admissibility} imply $0<\gamma_*<e^{-K\mu}\kappa<M_*$. Therefore
$$
q_\eta^2 =(1-\gamma_*\eta)^2+(M_*^2-\gamma_*^2)\eta^2>0,
\qquad
q_\eta^2=1-\eta(2\gamma_*-M_*^2\eta)<1,
$$
the last inequality following from $0<\eta<2\gamma_*/M_*^2$.

Now let $(X_0,Y_0)$ be an optimal coupling of $\nu$ and $\nu'$, so that $\mathbb E|X_0-Y_0|^2=\mathcal W_2^2(\nu,\nu')$. Take $G$ independent of this pair and apply the two updates above with $x=X_0$ and $y=Y_0$. The resulting $(X_1,Y_1)$ is a coupling of $\nu Q^{\rm R}_{\delta,\eta}$ and $\nu'Q^{\rm R}_{\delta,\eta}$. Conditioning on $(X_0,Y_0)$ and then taking expectation yields
$$
\begin{aligned}
\mathcal W_2^2\left(\nu Q^{\rm R}_{\delta,\eta},\nu'Q^{\rm R}_{\delta,\eta}\right)
&\leq\mathbb E|X_1-Y_1|^2\leq q_\eta^2\mathbb E|X_0-Y_0|^2 =q_\eta^2\mathcal W_2^2(\nu,\nu').
\end{aligned}
$$
Taking square roots proves~\eqref{eq:GS:sharp:contraction}.

Finally, $\mathcal P_2(\mathcal C)$ is complete in $\mathcal W_2$, and compactness of $\mathcal C$ ensures that every probability measure on $\mathcal C$ has finite second moment. The map $\nu\mapsto\nu Q^{\rm R}_{\delta,\eta}$ sends this space into itself and is a strict contraction with factor $q_\eta<1$. The Banach fixed point theorem therefore gives a unique fixed point $\pi^{\rm R}_{\delta,\eta}$, that is,
$\pi^{\rm R}_{\delta,\eta}Q^{\rm R}_{\delta,\eta}=\pi^{\rm R}_{\delta,\eta}$.
Using this invariance and writing $\nu_{k+1}=\nu_kQ^{\rm R}_{\delta,\eta}$, the one step estimate gives
$$
\mathcal W_2\left(\nu_{k+1},\pi^{\rm R}_{\delta,\eta}\right)
\leq q_\eta\mathcal W_2\left(\nu_k,\pi^{\rm R}_{\delta,\eta}\right).
$$
Induction proves the stated $k$-step bound. Since $\gamma_*$ and $M_*$ depend only on $\kappa,K,R,\mu$, both $\eta_c$ and $q_\eta$ are independent of $\delta$.
\end{proof}

\subsection{Proof of Theorem~\ref{thm:GS:sharp:separation}}
\label{sec:proof:GS:sharp:separation}
\begin{proof}
The RALMC assertion follows from~\eqref{eq:GS:sharp:contraction}, since the first condition in~\eqref{eq:GS:sharp:separation:condition} places $\eta$ in the interval~\eqref{eq:GS:sharp:eta}. We prove the PLMC assertions by computing the response of the projected transition mean to a change in the initial point.

Fix $\delta,\eta>0$ and define
\begin{equation}\label{eq:GS:sharp:projected:mean}
F_\delta(x):=x-\eta U_\delta'(x),\qquad
m_\delta(x):=\int_{\mathcal C}z\,Q^{\rm P}_{\delta,\eta}(x,dz)
=\mathbb E\left[\mathcal P_{\mathcal C}\left(F_\delta(x)+\sqrt{2\eta}G\right)\right].
\end{equation}
The last equality follows from~\eqref{eq:GS:sharp:kernels}. For the potential in~\eqref{eq:GS:sharp:family},
$$
U_\delta'(x)=\kappa x+\frac{Kx}{\sqrt{x^2+\delta^2}},\qquad U_\delta''(x)=\kappa+\frac{K\delta^2}{(x^2+\delta^2)^{3/2}}.
$$
Hence $F_\delta(0)=0$ and $F_\delta'(0)=1-\eta(\kappa+K/\delta)$. Projection onto $[-R,R]$ is the clipping map $\mathcal P_{\mathcal C}(z)=\min\{R,\max\{-R,z\}\}$, whose derivative is $\mathbf1_{(-R,R)}(z)$ away from $z=\pm R$. Since $\sqrt{2\eta}G$ has a continuous density, these two exceptional points have probability zero. Thus, almost surely,
$$
\lim_{x\to0}\frac{\mathcal P_{\mathcal C}(F_\delta(x)+\sqrt{2\eta}G)
-\mathcal P_{\mathcal C}(\sqrt{2\eta}G)}{x} =F_\delta'(0)\mathbf1_{\{|\sqrt{2\eta}G|<R\}}.
$$
To pass this limit through the expectation in~\eqref{eq:GS:sharp:projected:mean}, choose $0<\varepsilon<R$. Nonexpansiveness of projection and the mean value theorem bound the absolute difference quotient by
$$
\frac{|F_\delta(x)-F_\delta(0)|}{|x|}
\leq\sup_{|u|\leq\varepsilon}\left|F_\delta'(u)\right|<\infty,
\qquad 0<|x|\leq\varepsilon.
$$
This deterministic bound is integrable for the fixed $\delta,\eta$. Dominated convergence therefore yields
\begin{equation}\label{eq:GS:sharp:mean:derivative}
m_\delta'(0)=F_\delta'(0)\mathbb P\left(\left|\sqrt{2\eta}G\right|<R\right)
=\left[1-\eta\left(\kappa+\frac K\delta\right)\right]p_\eta.
\end{equation}

Then for any coupling $(Y,Z)$ of two probability measures $\rho,\rho'$ with finite second moments, Jensen's inequality gives $|\mathbb EY-\mathbb EZ|^2\leq\mathbb E|Y-Z|^2$. Taking the infimum over couplings yields
$$
\mathcal W_2(\rho,\rho')\geq\left|\int z\,\rho(dz)-\int z\,\rho'(dz)\right|.
$$
Apply this inequality to $Q^{\rm P}_{\delta,\eta}(x,\cdot)$ and $Q^{\rm P}_{\delta,\eta}(0,\cdot)$, whose means are defined in~\eqref{eq:GS:sharp:projected:mean}. Dividing by $|x|$ and using~\eqref{eq:GS:sharp:mean:derivative} gives
\begin{equation}\label{eq:GS:sharp:mean:W2:lower}
\liminf_{\substack{x\to0\\x\neq0}}
\frac{\mathcal W_2(Q^{\rm P}_{\delta,\eta}(x,\cdot),Q^{\rm P}_{\delta,\eta}(0,\cdot))}{|x|}
\geq\lim_{x\to0}\frac{|m_\delta(x)-m_\delta(0)|}{|x|}=|1-\eta(\kappa+K/\delta)|p_\eta.
\end{equation}
Under~\eqref{eq:GS:sharp:separation:condition}, $\eta(\kappa+K/\delta)>1+p_\eta^{-1}>1$. The absolute value in~\eqref{eq:GS:sharp:mean:W2:lower} is consequently $\eta(\kappa+K/\delta)-1$, and its product with $p_\eta$ is strictly greater than one. This proves~\eqref{eq:GS:sharp:PL:expansion}.

Fix any $\eta\in(0,\eta_c)$, with $\eta_c$ given by~\eqref{eq:GS:sharp:eta}. Since $R,\eta>0$, one has $p_\eta>0$, so the upper bound in~\eqref{eq:GS:sharp:delta:range} is positive. For every $\delta$ in that range,
$$
\eta(\kappa+K/\delta)>\eta K/\delta>1+p_\eta^{-1}.
$$
Thus both conditions in~\eqref{eq:GS:sharp:separation:condition} hold. Finally, if PLMC is globally nonexpansive, then the numerator in~\eqref{eq:GS:sharp:mean:W2:lower} is at most $|x|$ for every $x\in\mathcal C$. Therefore, 
$
|1-\eta(\kappa+K/\delta)|p_\eta\leq1$ gives $
\eta(\kappa+K/\delta)\leq1+p_\eta^{-1}$. For $0<\eta\leq\eta_{\max}$,
$$
p_\eta=\mathbb P\!\left(|G|<\frac R{\sqrt{2\eta}}\right)
\geq\mathbb P\!\left(|G|<\frac R{\sqrt{2\eta_{\max}}}\right)
=p_{\eta_{\max}}>0.
$$
Combining these inequalities proves~\eqref{eq:GS:sharp:PL:necessary}. In particular,
$$
\eta\leq\frac{1+p_{\eta_{\max}}^{-1}}{\kappa+K/\delta}
\leq\frac{1+p_{\eta_{\max}}^{-1}}{K}\,\delta,
$$
which establishes the claimed $O(\delta)$ restriction with a constant independent of $\delta$.
\end{proof}

\subsection{Proof of Proposition~\ref{prop:dirichlet:comparison}}
\label{proof:prop:dirichlet:comparison}
\begin{proof}
Both Rayleigh quotients are taken over the same functions $f\in H^1(\mathcal C)$ with $\operatorname{Var}_\pi(f)>0$. Since $\|\nabla f\|^2\geq0$, integrating the bounds $a_{{\rm n},\min}\leq a_{\rm n}(x)\leq a_{{\rm n},\max}$ gives
\begin{equation}\label{eq:gap:energy:comparison}
a_{{\rm n},\min}\mathcal E_{\rm PLD}(f)
\leq\mathcal E_{\rm RALD}^{\rm n}(f)
\leq a_{{\rm n},\max}\mathcal E_{\rm PLD}(f).
\end{equation}
For the lower bound, divide the left inequality in~\eqref{eq:gap:energy:comparison} by $\operatorname{Var}_\pi(f)$ to obtain
$$
\frac{\mathcal E_{\rm RALD}^{\rm n}(f)}{\operatorname{Var}_\pi(f)}
\geq a_{{\rm n},\min}
\frac{\mathcal E_{\rm PLD}(f)}{\operatorname{Var}_\pi(f)}
\geq a_{{\rm n},\min}\lambda_{\rm PLD}.
$$
Taking the infimum over all admissible $f$ yields
$\lambda_{\rm RALD}^{\rm n}\geq a_{{\rm n},\min}\lambda_{\rm PLD}$.

For the upper bound, fix $\varepsilon>0$. By the definition of $\lambda_{\rm PLD}$, there exists $f_\varepsilon\in H^1(\mathcal C)$ with $\operatorname{Var}_\pi(f_\varepsilon)>0$ such that
$$
\frac{\mathcal E_{\rm PLD}(f_\varepsilon)}{\operatorname{Var}_\pi(f_\varepsilon)}
\leq\lambda_{\rm PLD}+\varepsilon.
$$
Using $f_\varepsilon$ as a test function for the RALD Rayleigh quotient and applying the right inequality in~\eqref{eq:gap:energy:comparison}, we obtain
$$
\lambda_{\rm RALD}^{\rm n}
\leq\frac{\mathcal E_{\rm RALD}^{\rm n}(f_\varepsilon)}{\operatorname{Var}_\pi(f_\varepsilon)}\leq a_{{\rm n},\max}
\frac{\mathcal E_{\rm PLD}(f_\varepsilon)}{\operatorname{Var}_\pi(f_\varepsilon)}
\leq a_{{\rm n},\max}(\lambda_{\rm PLD}+\varepsilon).
$$
Letting $\varepsilon\rightarrow0$ proves the upper bound in~\eqref{eq:gap:comparison}. This argument does not require either infimum to be attained.
\end{proof}

\subsection{Proof of Proposition~\ref{prop:bottleneck:minorization}}
\label{sec:prop:bottleneck:minorization}

\begin{proof}
If $u\notin\mathcal C$, then $\mathcal P_{\mathcal C}(u)\in\partial\mathcal C$. Since $D_+\subset\operatorname{int}(\mathcal C)$, it follows that
$$
\begin{aligned}
    \left\{\mathcal P_{\mathcal C}(Z_{\rm R}(x))\in D_+\right\}
    =\left\{Z_{\rm R}(x)\in D_+\right\},\quad
    \left\{\mathcal P_{\mathcal C}(Z_{\rm P}(x))\in D_+\right\}
    =\left\{Z_{\rm P}(x)\in D_+\right\}.
\end{aligned}
$$
Thus the projected crossing probabilities are integrals of the Gaussian proposal densities over $D_+$. For fixed $x\in B$, these densities are
$$
\begin{aligned}
p_{\rm R}(x,y)&=(4\pi\eta a_{\rm n}(x))^{-d/2}
\exp\left[-\frac{\|y-x+\eta a_{\rm n}(x)\nabla U_0(x)\|^2}{4\eta a_{\rm n}(x)}\right],\\
p_{\rm P}(x,y)&=(4\pi\eta)^{-d/2}
\exp\left[-\frac{\|y-x+\eta\nabla U(x)\|^2}{4\eta}\right].
\end{aligned}
$$
Expanding the squared norms, the quadratic, cross, and drift terms give, respectively,
\begin{equation}\label{eq:proof:bottleneck:density:ratio}
\begin{aligned}
\log\frac{p_{\rm R}(x,y)}{p_{\rm P}(x,y)}
&=-\frac d2\log a_{\rm n}(x)
+\frac{1-a_{\rm n}(x)^{-1}}{4\eta}\|y-x\|^2\\
&\qquad+\frac12\langle y-x,\nabla U(x)-\nabla U_0(x)\rangle
+\frac\eta4\left(\|\nabla U(x)\|^2-a_{\rm n}(x)\|\nabla U_0(x)\|^2\right).
\end{aligned}
\end{equation}
For $y\in D_+$, $r_x\leq\|y-x\|\leq\operatorname{diam}(\mathcal C)$. Since $a_{\rm n}(x)>1$, the quadratic term is bounded below using $r_x$. Cauchy-Schwarz inequality gives
$$
\langle y-x,\nabla U(x)-\nabla U_0(x)\rangle
\geq-\operatorname{diam}(\mathcal C)\|\nabla U(x)-\nabla U_0(x)\|.
$$
Dropping the nonnegative term $\eta\|\nabla U(x)\|^2/4$ in~\eqref{eq:proof:bottleneck:density:ratio} and exponentiating gives, for every $y\in D_+$,
$$
\begin{aligned}
p_{\rm R}(x,y)&\geq a_{\rm n}(x)^{-d/2}
\exp\Bigg\{\frac{r_x^2}{4\eta}\left(1-\frac1{a_{\rm n}(x)}\right)-\frac{\operatorname{diam}(\mathcal C)}2\|\nabla U(x)-\nabla U_0(x)\| -\frac\eta4a_{\rm n}(x)\|\nabla U_0(x)\|^2\Bigg\}
\\
&\qquad\qquad\qquad\qquad\qquad\cdot
p_{\rm P}(x,y).
\end{aligned}
$$
The factor multiplying $p_{\rm P}(x,y)$ is independent of $y$. Integrating this inequality over $D_+$ therefore yields
$$
\begin{aligned}
&\mathbb P_\eta^{\rm RALMC}(x,D_+)
=\int_{D_+}p_{\rm R}(x,y)\,dy\\
&\,\geq a_{\rm n}(x)^{-d/2}
\exp\Bigg\{\frac{r_x^2}{4\eta}\left(1-\frac1{a_{\rm n}(x)}\right)-\frac{\operatorname{diam}(\mathcal C)}2\|\nabla U(x)-\nabla U_0(x)\| -\frac\eta4a_{\rm n}(x)\|\nabla U_0(x)\|^2\Bigg\}
\\
&\qquad\qquad\qquad\qquad\qquad\cdot
\int_{D_+}p_{\rm P}(x,y)\,dy.
\end{aligned}
$$
Since the Gaussian density is strictly positive and $|D_+|>0$,
$$
\int_{D_+}p_{\rm P}(x,y)\,dy
=\mathbb P_\eta^{\rm PLMC}(x,D_+)>0.
$$
Dividing the preceding inequality by this positive quantity cancels the integral on its right hand side and gives
$$
\begin{aligned}
&\frac{\mathbb P_\eta^{\rm RALMC}(x,D_+)}{\mathbb P_\eta^{\rm PLMC}(x,D_+)}
\\
&\geq a_{\rm n}(x)^{-d/2}
\exp\Bigg\{\frac{r_x^2}{4\eta}\left(1-\frac1{a_{\rm n}(x)}\right)-\frac{\operatorname{diam}(\mathcal C)}2\|\nabla U(x)-\nabla U_0(x)\| -\frac\eta4a_{\rm n}(x)\|\nabla U_0(x)\|^2\Bigg\},
\end{aligned}
$$
which is~\eqref{eq:bottleneck:pointwise:ratio}. For the uniform estimate, $a_{\rm n}(x)\leq A$ gives $a_{\rm n}(x)^{-d/2}\geq A^{-d/2}$. Also, $a_{\rm n}(x)\geq\rho>1$ and $r_x\geq r$ imply
$$
r_x^2\left(1-\frac1{a_{\rm n}(x)}\right)
\geq r^2\left(1-\frac1\rho\right).
$$
Finally, bound each of the two nonnegative drift penalties by its supremum over $B$. Substitution into~\eqref{eq:bottleneck:pointwise:ratio} gives~\eqref{eq:bottleneck:finite:ratio}.
\end{proof}

\subsection{Proof of Corollary~\ref{cor:bottleneck:minorization}}
\label{sec:cor:bottleneck:minorization}

\begin{proof}
We first compute the logarithmic crossing rate for a Gaussian proposal
$$
Y_\eta=x+\eta b+\sqrt{2\eta a}\,G,
$$
where $x\in B$, $a>0$, and $b\in\mathbb R^d$ are fixed as $\eta\to0$. Its density is
$$
p_\eta(y)=(4\pi\eta a)^{-d/2}
\exp\left[-\frac{\|y-x-\eta b\|^2}{4\eta a}\right].
$$
For every $y\in D_+$, the triangle inequality gives
$\|y-x-\eta b\|\geq r_x-\eta\|b\|$. For sufficiently small $\eta$, this lower bound is positive. Since $D_+\subset\mathcal C$,
$$
\mathbb P(Y_\eta\in D_+)
\leq |\mathcal C|(4\pi\eta a)^{-d/2}
\exp\left[-\frac{(r_x-\eta\|b\|)^2}{4\eta a}\right].
$$
Taking logarithms and multiplying by $\eta$, the volume and Gaussian normalization terms tend to zero because $\eta\log\eta\to0$. Hence
$$
\limsup_{\eta\to0}\eta\log\mathbb P(Y_\eta\in D_+)
\leq-\frac{r_x^2}{4a}.
$$
Fix $\varepsilon>0$. By the definition of distance, choose $y_\varepsilon\in D_+$ with $\|y_\varepsilon-x\|<r_x+\varepsilon$. Openness of $D_+$ provides a radius $0<s_\varepsilon<\varepsilon$ such that the ball $\mathcal{B}(y_\varepsilon,s_\varepsilon)$ lies in $D_+$. For every point $y$ of this ball,
$$
\|y-x-\eta b\|
\leq\|y-y_\varepsilon\|+\|y_\varepsilon-x\|+\eta\|b\|
<r_x+2\varepsilon+\eta\|b\|.
$$
Integrating the density over this ball therefore yields
$$
\mathbb P(Y_\eta\in D_+)
\geq |\mathcal{B}(y_\varepsilon,s_\varepsilon)|(4\pi\eta a)^{-d/2}
\exp\left[-\frac{(r_x+2\varepsilon+\eta\|b\|)^2}{4\eta a}\right].
$$
For fixed $\varepsilon$, the ball has positive volume independent of $\eta$. Taking $\liminf$ after multiplying the logarithm by $\eta$ gives
$$
\liminf_{\eta\to0}\eta\log\mathbb P(Y_\eta\in D_+)
\geq-\frac{(r_x+2\varepsilon)^2}{4a}.
$$
Now let $\varepsilon\rightarrow0$ and combine the two bounds to obtain
\begin{equation}\label{eq:proof:bottleneck:gaussian:ldp}
\lim_{\eta\to0}\eta\log\mathbb P(Y_\eta\in D_+)
=-\frac{r_x^2}{4a}.
\end{equation}
For normalized RALMC, take $a=a_{\rm n}(x)$ and $b=-a_{\rm n}(x)\nabla U_0(x)$ in Equation~\eqref{eq:proof:bottleneck:gaussian:ldp}. For PLMC, take $a=1$ and $b=-\nabla U(x)$. The projection identities at the start of the proof of Proposition~\ref{prop:bottleneck:minorization} then give Equation~\eqref{eq:bottleneck:exact:rates}. Subtracting the two limits gives Equation~\eqref{eq:bottleneck:probability:ratio}. Since $r_x>0$, its right hand side is strictly positive if and only if $a_{\rm n}(x)>1$. The final uniform bound follows from $r_x\geq r$ and $a_{\rm n}(x)\geq\rho$.
\end{proof}

\end{document}